\pdfoutput=1
\documentclass[11pt]{article}

\PassOptionsToPackage{table}{xcolor}
\PassOptionsToPackage{hyperfootnotes=false}{hyperref}
\usepackage[preprint]{acl}

\usepackage{times}
\usepackage{latexsym}
\usepackage{fontawesome5}
\usepackage{etoolbox}
\usepackage[T1]{fontenc}
\usepackage[utf8]{inputenc}
\usepackage{microtype}
\usepackage{inconsolata}
\usepackage{graphicx}
\usepackage[skip=2pt,belowskip=0pt]{caption}
\usepackage[skip=1pt]{subcaption}
\usepackage{booktabs}
\usepackage{amsmath}
\usepackage{multirow}
\usepackage{xcolor}
\usepackage{amsfonts}
\usepackage{amsthm}
\usepackage{tikz}
\usetikzlibrary{arrows.meta}
\usepackage{stfloats}

\theoremstyle{plain}
\newtheorem{theorem}{Theorem}
\newtheorem{proposition}[theorem]{Proposition}
\newtheorem{lemma}[theorem]{Lemma}
\newtheorem{corollary}[theorem]{Corollary}
\theoremstyle{definition}
\newtheorem{definition}{Definition}

\theoremstyle{remark}
\newtheorem*{remark}{Remark}

\newcommand{\AblBandN}{28}
\newcommand{\AblBandPLo}{0.219}
\newcommand{\AblPrefPHi}{0.83}
\newcommand{\AblPrefPLo}{0.50}

\newcommand{\AchronyHumanTau}{0.363}
\newcommand{\AchronyModelTau}{0.171}
\newcommand{\AnchorsAfter}{0.7}
\newcommand{\AnchorsBefore}{6.3}

\newcommand{\BandMcP}{< 0.001}

\newcommand{\BestBandModel}{Gemini 3.1 Pro}

\newcommand{\BestClassModel}{Gemini 3.6 Flash}

\newcommand{\ChronAgree}{0.878}

\newcommand{\ChronExactId}{88.9}

\newcommand{\ChronTauHi}{0.97}
\newcommand{\ChronTauLo}{0.62}

\newcommand{\CorpusTau}{0.456}
\newcommand{\DedDelta}{-0.005}

\newcommand{\DedExactShare}{42.1}

\newcommand{\DedTauFloor}{0.849}

\newcommand{\DedTauShare}{96.1}
\newcommand{\DedTauTold}{0.884}

\newcommand{\EarlyGZeroN}{56}

\newcommand{\EarlyHumanN}{14}

\newcommand{\EffortMaxBand}{0.143}
\newcommand{\EffortMaxTauRec}{0.025}
\newcommand{\EffortNModels}{4}

\newcommand{\FloorCorr}{0.837}
\newcommand{\FreeSignP}{3.8 \times 10^{-6}}

\newcommand{\GOnePDS}{4.096}
\newcommand{\GOneVsGTwoL}{22}

\newcommand{\GOneVsGTwoW}{23}
\newcommand{\GOneVsGZeroL}{14}
\newcommand{\GOneVsGZeroP}{0.016}

\newcommand{\GOneVsGZeroT}{11}
\newcommand{\GOneVsGZeroW}{31}
\newcommand{\GTwoAgree}{0.998}

\newcommand{\GZeroPDS}{3.850}
\newcommand{\GraphBestDelta}{+0.250}
\newcommand{\GraphBestModel}{Claude Sonnet 5}
\newcommand{\GraphCommonModels}{2}
\newcommand{\GraphInertModels}{2}
\newcommand{\GraphLeakMax}{0.111}
\newcommand{\GraphNJudged}{252}
\newcommand{\GraphNModels}{4}
\newcommand{\GraphNStories}{476}
\newcommand{\GraphWorstDelta}{-0.035}
\newcommand{\GraphWorstModel}{Gemini 3.6 Flash}
\newcommand{\GridCost}{33}

\newcommand{\HumanHopRate}{0.502}
\newcommand{\InterJudgeDiff}{0.044}
\newcommand{\InterJudgeN}{93}
\newcommand{\InterJudgeR}{0.895}
\newcommand{\JudgedBandHi}{83}
\newcommand{\JudgedBandLo}{68}
\newcommand{\LeastChronoModel}{Llama 4 Maverick}
\newcommand{\LeastChronoTau}{0.480}

\newcommand{\MaxChars}{23}

\newcommand{\NAboveHuman}{19}
\newcommand{\NBeatFloor}{7}

\newcommand{\NFewerHops}{18}
\newcommand{\NModels}{19}

\newcommand{\NModelsOrdering}{19}

\newcommand{\NNonRefusers}{16}

\newcommand{\NRefusedShared}{17}

\newcommand{\NWorks}{83}

\newcommand{\PDSCompsN}{5}
\newcommand{\PDSCompsUp}{5}
\newcommand{\PinnedPct}{5.5}

\newcommand{\PoolBandHitPct}{47.9}

\newcommand{\PoolClassifyPct}{50.2}

\newcommand{\PoolContigPct}{97.6}

\newcommand{\PoolHopRateFree}{0.298}
\newcommand{\PoolLisFree}{0.796}

\newcommand{\PoolTauFree}{0.638}

\newcommand{\PoolTauRec}{0.743}

\newcommand{\ProseCuedPDS}{3.60}
\newcommand{\ProseCuedWins}{19}

\newcommand{\ProseFreePDSs}{3.74}
\newcommand{\ProseFreeTau}{0.989}
\newcommand{\ProseFreeWins}{68}
\newcommand{\ProseFullChronoModels}{3}
\newcommand{\ProseLoseAll}{8}
\newcommand{\ProseNModels}{8}
\newcommand{\ProseNStories}{14}
\newcommand{\ProsePairP}{1.2 \times 10^{-7}}
\newcommand{\ProsePairs}{111}
\newcommand{\ProseTies}{24}

\newcommand{\RcgCeiling}{0.12}

\newcommand{\RefusalHi}{13.0}
\newcommand{\RefusalLo}{12.7}
\newcommand{\RetroAgree}{0.764}

\newcommand{\RetroExactRev}{93.6}

\newcommand{\RetroRcg}{0.003}
\newcommand{\RetroRevealPos}{0.84}
\newcommand{\RetroRunLo}{2}
\newcommand{\RetroSceneExact}{5}
\newcommand{\RetroTauHi}{-0.43}
\newcommand{\RetroTauLo}{-0.92}
\newcommand{\RhoBandRec}{0.557}
\newcommand{\RhoClassBand}{0.717}
\newcommand{\RhoClassRec}{0.747}
\newcommand{\RhoCuedTauRec}{0.088}

\newcommand{\SinkHi}{46}
\newcommand{\SinkLo}{29}
\newcommand{\SinkN}{5}

\newcommand{\TauFreeHi}{0.808}
\newcommand{\TauFreeLo}{0.480}

\newcommand{\TauRecHi}{0.892}
\newcommand{\TauRecLo}{0.409}
\newcommand{\TellingWorth}{+0.035}
\newcommand{\TellingWorthMax}{0.151}

\newcommand{\ValidRate}{99.5}

\newcommand{\BandHitHiPct}{60.2}
\newcommand{\BandHitLoPct}{28.9}

\newcommand{\BestClassAccPct}{79.5}
\newcommand{\ChanceClassifyPct}{14.3}
\newcommand{\ChronBandPct}{5.3}
\newcommand{\ChronRecExactPct}{0}

\newcommand{\ConfAchronySelfPct}{7.1}

\newcommand{\DedExactFloorPct}{8.5}
\newcommand{\DedExactToldPct}{20.2}

\newcommand{\EarlyGZeroPct}{21.4}

\newcommand{\GOneBandPct}{51.8}
\newcommand{\GTwoBandPct}{98.2}
\newcommand{\GZeroBandPct}{41.1}
\newcommand{\GuessGOnePct}{67.4}
\newcommand{\GuessGTwoPct}{86.5}
\newcommand{\GuessGZeroPct}{77.3}

\newcommand{\HAblClaudeSonnetN}{147}

\newcommand{\HAblClaudeSonnetPct}{18}
\newcommand{\HAblClaudeSonnetPrel}{<0.0001}

\newcommand{\HAblClaudeSonnetW}{26}

\newcommand{\HAblEpistemicR}{0.271}

\newcommand{\HAblGeminiFlashPct}{47}

\newcommand{\HAblHetChi}{25.0}

\newcommand{\HAblHetPrel}{<0.0001}
\newcommand{\HAblLayerN}{51}
\newcommand{\HAblMDE}{0.19}

\newcommand{\HAblQuestionsR}{0.333}

\newcommand{\HAdjBetaBand}{-0.58}

\newcommand{\HAdjBetaTau}{+0.70}

\newcommand{\HAdjPBandrel}{<0.0001}

\newcommand{\HAlphaBetter}{0.415}
\newcommand{\HAlphaBetterCI}{[0.359, 0.471]}
\newcommand{\HAlphaBetterDec}{0.595}
\newcommand{\HAlphaBetterN}{311}
\newcommand{\HAlphaConf}{0.112}

\newcommand{\HAlphaEnd}{0.073}

\newcommand{\HAlphaKindHi}{0.55}
\newcommand{\HAlphaKindLo}{0.21}

\newcommand{\HBandDepthBeta}{+0.12}

\newcommand{\HBandDepthPrel}{=0.13}
\newcommand{\HBandSplitN}{381}

\newcommand{\HBandSplitPct}{34}
\newcommand{\HBandSplitPrel}{<0.0001}

\newcommand{\HCatRho}{0.18}

\newcommand{\HCatRhoPrel}{=0.70}
\newcommand{\HChronCI}{[0.497, 0.562]}

\newcommand{\HChronEasiest}{Chronicle}
\newcommand{\HChronEasiestTau}{0.707}
\newcommand{\HChronExact}{2.7}

\newcommand{\HChronHardest}{Zigzag and Achrony}
\newcommand{\HChronHardestTau}{0.424}
\newcommand{\HChronN}{294}
\newcommand{\HChronSCI}{[0.488, 0.582]}

\newcommand{\HChronTau}{0.537}
\newcommand{\HConfConsensusN}{227}

\newcommand{\HConfConsensusPrel}{<0.0001}
\newcommand{\HConfConsensusR}{+0.26}
\newcommand{\HConfMean}{4.21}
\newcommand{\HConfSd}{0.85}
\newcommand{\HConfTauDelta}{-0.171}
\newcommand{\HConfTauN}{918}

\newcommand{\HConfTauPrel}{<0.0001}
\newcommand{\HConfTauR}{-0.34}

\newcommand{\HCuedBetterCI}{[0.185, 0.304]}
\newcommand{\HCuedBetterInvPct}{76}
\newcommand{\HCuedBetterL}{200}

\newcommand{\HCuedBetterT}{30}
\newcommand{\HCuedBetterW}{64}

\newcommand{\HCuedConfPct}{81}

\newcommand{\HCuedLMPct}{30}

\newcommand{\HCuedLongerPct}{57}

\newcommand{\HCuedModelsN}{8}
\newcommand{\HCuedReadersAgainst}{49}

\newcommand{\HCuedReadersN}{68}

\newcommand{\HCuedWordsHyp}{747}
\newcommand{\HCuedWordsOther}{711}

\newcommand{\HDepthTauR}{-0.06}
\newcommand{\HExpoBetterBeta}{+0.040}

\newcommand{\HExpoBetterPrel}{=0.36}
\newcommand{\HExpoConfBeta}{-0.025}

\newcommand{\HExpoConfPrel}{=0.70}
\newcommand{\HExpoMax}{6}
\newcommand{\HFloorFourteen}{0.833}

\newcommand{\HGAblBetterInvPct}{69}

\newcommand{\HGAblBetterN}{259}

\newcommand{\HGAblBetterPrel}{<0.0001}

\newcommand{\HGAblConfPct}{62}
\newcommand{\HGAblConfPrel}{=0.0004}

\newcommand{\HGAblLMPct}{31}

\newcommand{\HGNoneAdjPct}{52}
\newcommand{\HGNoneAdjPrel}{=0.61}

\newcommand{\HGNoneBetterPct}{59}
\newcommand{\HGNoneBetterPrel}{=0.0028}

\newcommand{\HGNoneLMPct}{54}

\newcommand{\HGNoneLenBeta}{+0.67}

\newcommand{\HGNoneLenBetaPrel}{=0.0003}

\newcommand{\HGNoneLongerPct}{61}

\newcommand{\HGNoneWordsHyp}{649}
\newcommand{\HGNoneWordsOther}{603}

\newcommand{\HGOrderAdjPct}{71}

\newcommand{\HGOrderBetterCI}{[0.671, 0.771]}
\newcommand{\HGOrderBetterInvPct}{28}
\newcommand{\HGOrderBetterL}{70}

\newcommand{\HGOrderBetterPct}{72}
\newcommand{\HGOrderBetterPrel}{<0.0001}
\newcommand{\HGOrderBetterR}{0.721}

\newcommand{\HGOrderBetterW}{181}

\newcommand{\HGOrderConfPrel}{=0.69}

\newcommand{\HGOrderEndR}{0.571}

\newcommand{\HGOrderModelsFor}{3}
\newcommand{\HGOrderModelsN}{4}

\newcommand{\HGOrderReadersFor}{53}
\newcommand{\HGOrderReadersN}{68}
\newcommand{\HGOrderStoriesFor}{12}
\newcommand{\HGOrderStoriesN}{14}

\newcommand{\HGapCuedPts}{2}

\newcommand{\HJHCuedKappa}{+0.278}

\newcommand{\HJHGAblCI}{[-0.23, +0.18]}
\newcommand{\HJHGAblKappa}{-0.030}

\newcommand{\HJHGNoneCI}{[+0.28, +0.91]}
\newcommand{\HJHGNoneKappa}{+0.625}

\newcommand{\HJHGOrderKappa}{+0.242}

\newcommand{\HJHIndAgree}{0.679}
\newcommand{\HJHIndN}{779}
\newcommand{\HJHPoolCI}{[+0.17, +0.43]}
\newcommand{\HJHPoolKappa}{+0.306}

\newcommand{\HJHPoolRaw}{0.662}

\newcommand{\HJudgeGOrderN}{45}

\newcommand{\HJudgeGOrderR}{0.511}

\newcommand{\HJudgeGOrderTail}{0.0022}

\newcommand{\HLitChance}{42}
\newcommand{\HLitCorrect}{1{,}430}
\newcommand{\HLitEnd}{88}

\newcommand{\HLitJumps}{69}

\newcommand{\HLitN}{1{,}764}
\newcommand{\HLitOpen}{86}

\newcommand{\HLitOverall}{81.1}

\newcommand{\HMachineCuedPct}{55}

\newcommand{\HMachineGOnePct}{36}

\newcommand{\HMachineGZeroPct}{38}

\newcommand{\HMachinePrel}{=0.0033}

\newcommand{\HModelTauFourteen}{0.699}
\newcommand{\HModelsAboveHuman}{16}
\newcommand{\HModelsN}{19}

\newcommand{\HNCommentsCued}{186}
\newcommand{\HNCommentsCuedOrder}{50}

\newcommand{\HNJudgments}{3{,}528}

\newcommand{\HNReaders}{71}

\newcommand{\HNShapeStories}{83}
\newcommand{\HNSittings}{294}
\newcommand{\HNStories}{14}

\newcommand{\HPdsCorr}{0.43}

\newcommand{\HPdsCorrN}{193}

\newcommand{\HPdsCued}{3.31}

\newcommand{\HPdsGOne}{3.90}

\newcommand{\HPdsGOneVsCuedD}{0.87}

\newcommand{\HPdsGOneVsCuedPrel}{<0.0001}
\newcommand{\HPdsGOneVsGZeroD}{0.51}

\newcommand{\HPdsGZero}{3.57}

\newcommand{\HPdsGap}{+0.28}

\newcommand{\HPdsGapComplex}{+0.81}
\newcommand{\HPdsGapEngage}{+0.52}

\newcommand{\HReadsMonthly}{43}

\newcommand{\HRepeatBetterChancePct}{39}

\newcommand{\HRepeatBetterDecPct}{93}
\newcommand{\HRepeatBetterKappa}{0.601}

\newcommand{\HRepeatBetterPct}{76}

\newcommand{\HRepeatConfChancePct}{38}

\newcommand{\HRepeatConfPct}{62}

\newcommand{\HRepeatEndChancePct}{36}

\newcommand{\HRepeatEndPct}{57}
\newcommand{\HRepeatPairs}{442}
\newcommand{\HRepeatPairsN}{828}
\newcommand{\HRtCued}{217}
\newcommand{\HRtGOne}{188}
\newcommand{\HRtGZero}{179}
\newcommand{\HRtIqr}{154--246}
\newcommand{\HRtMedian}{199}
\newcommand{\HRtMinReader}{99}
\newcommand{\HSideA}{1{,}400}
\newcommand{\HSideB}{1{,}511}
\newcommand{\HSideN}{3{,}528}

\newcommand{\HSidePrel}{=0.0415}
\newcommand{\HSideTie}{617}

\newcommand{\HSpoilerCuedPct}{20}

\newcommand{\HSpoilerGZeroPct}{17}

\newcommand{\HTopTwentyNonChrono}{19}

\newcommand{\HViolCued}{0.99}

\newcommand{\HViolCuedVsGZeroPrel}{=0.0418}
\newcommand{\HViolGOne}{0.73}
\newcommand{\HViolGZero}{0.70}

\newcommand{\HWriteDegree}{10}

\newcommand{\HWriteSome}{30}
\newcommand{\MajorityClassifyPct}{26.5}

\newcommand{\ProseFreeBandPct}{1.8}
\newcommand{\RandomBandHitPct}{8.4}
\newcommand{\RetroBandPct}{1.1}

\newcommand{\ShapeTauChronicle}{+0.96}

\newcommand{\ShapeTauAnalepsis}{+0.68}

\newcommand{\ShapeTauProlepsis}{+0.75}

\newcommand{\ShapeTauZigzag}{+0.43}

\newcommand{\ShapeTauSyllepsis}{+0.62}

\newcommand{\ShapeTauRetrograde}{-0.82}

\newcommand{\ShapeTauAchrony}{+0.26}

\newcommand{\NNonlinearWorks}{73}
\newcommand{\CuedNValid}{1{,}569}

\newcommand{\RetroRevN}{176}   %
\newcommand{\RetroValidN}{188}
\newcommand{\ChronIdN}{169}    %
\newcommand{\ChronValidN}{190}
\newcommand{\OtherValidN}{1{,}191} %
\newcommand{\OtherRevN}{5}
\newcommand{\OtherIdN}{54}
\newcommand{\BandHitNoRulePct}{62.1}
\newcommand{\BandHitCIHalf}{0.11}
\newcommand{\BandHitCIOverlap}{17}
\newcommand{\LowoBandLo}{0.423}
\newcommand{\LowoBandHi}{0.479}
\newcommand{\LowoMovers}{7}      %
\newcommand{\HeldOutInBand}{64}  %
\newcommand{\HeldOutInBandPct}{77}
\newcommand{\GZeroAgree}{0.554}
\newcommand{\GOneAgree}{0.554}
\newcommand{\ZigzagAgree}{0.396}
\newcommand{\HJudgeGOrderTailLo}{0.019}
\newcommand{\EarlyPairedP}{0.0078}
\newcommand{\EarlyPairedDisc}{8}
\newcommand{\EarlyCleanGZeroPct}{8.3}
\newcommand{\EarlyCleanGOnePct}{8.5}
\newcommand{\EarlyCleanP}{0.125}
\newcommand{\EarlyGOneN}{55}         %
\newcommand{\EarlyGTwoN}{56}         %
\newcommand{\EarlyNoRevealN}{13}     %
\newcommand{\EarlyGZeroGOneDisc}{4}  %
\newcommand{\EarlyCleanGZeroK}{4}
\newcommand{\EarlyCleanGZeroN}{48}
\newcommand{\EarlyCleanGOneK}{4}
\newcommand{\EarlyCleanGOneN}{47}
\newcommand{\EarlyCleanGTwoN}{48}
\newcommand{\EarlyCleanHumanN}{12}
\newcommand{\EarlyCurveGZeroLo}{3.8}
\newcommand{\EarlyCurveGZeroHi}{34.6}
\newcommand{\EarlyCurveGTwoHi}{7.7}
\newcommand{\CuedRevN}{181}
\newcommand{\CuedIdN}{223}
\newcommand{\HByCatCuedChron}{0.32}
\newcommand{\HByCatCuedChronN}{37}
\newcommand{\HByCatCuedChronP}{0.047}
\newcommand{\HByCatCuedRetro}{0.48}
\newcommand{\HByCatOrderChron}{0.65}
\newcommand{\HByCatOrderChronN}{31}
\newcommand{\HSubsetMaxMovePts}{7}

\newcommand{\ChronBandTauHi}{0.974}
\newcommand{\RetroBandTauLo}{-0.923}
\newcommand{\BaseIdChronHit}{0}      %
\newcommand{\BaseRevRetroHit}{0}     %
\newcommand{\BaseSceneRetroHit}{6}   %
\newcommand{\BaseScenePooledPct}{7.2}  %
\newcommand{\BaseBlockPct}{9.6}        %
\newcommand{\RandomMCPct}{5.6}
\newcommand{\RandomDraws}{5{,}000}
\newcommand{\RandomMCHi}{0.269}
\newcommand{\RandomMCLo}{0.000}
\newcommand{\CollideN}{9}
\newcommand{\CollideCells}{136}
\newcommand{\CollideHops}{112}
\newcommand{\CollideLabels}{69}
\newcommand{\FloorPairAcc}{92.5}
\newcommand{\ToldPairAcc}{94.2}
\newcommand{\HNPresentations}{1{,}176}
\newcommand{\HNFirstPres}{882}

\newcommand{\SceneNWorks}{10}
\newcommand{\SceneNCued}{196}
\newcommand{\SceneHumanWithin}{0.0}
\newcommand{\SceneHumanBetween}{94.0}
\newcommand{\SceneCuedWithin}{98.0}
\newcommand{\SceneCuedBetween}{98.8}
\newcommand{\SceneCuedAllRev}{93.9}

\definecolor{Indigo}{HTML}{211cbf}
\definecolor{Cyan}{HTML}{74d2fa}
\definecolor{Violet}{HTML}{7A4DFF}
\definecolor{Magenta}{HTML}{ec8dff}
\definecolor{LavenderOne}{HTML}{A7A8F0}
\definecolor{LavenderTwo}{HTML}{9194E4}
\definecolor{LavenderThree}{HTML}{8180DE}
\definecolor{LavenderFour}{HTML}{6E65CE}
\definecolor{LavenderFive}{HTML}{5445B5}
\colorlet{bpink}{Magenta!38}      %
\colorlet{bcyan}{Cyan!40}         %
\colorlet{blav}{Violet!26}        %
\colorlet{blemon}{LavenderOne!40} %
\colorlet{inkmag}{Violet}
\colorlet{inkblue}{Indigo}
\colorlet{inkteal}{LavenderFive}
\colorlet{inkviol}{LavenderFour}
\newcommand{\keycell}[1]{\cellcolor{blav}\textbf{#1}}
\newcommand{\poscell}[1]{\cellcolor{bcyan}#1}
\newcommand{\negcell}[1]{\cellcolor{bpink}#1}
\newcommand{\warncell}[1]{\cellcolor{blemon}#1}
\newcommand{\keyrow}{\rowcolor{blav!50}}
\newcommand{\cswatch}[2]{%
	\colorbox{#1}{\scriptsize\,#2\,}}
\newcommand{\cink}[2]{\textcolor{#1}{\textbf{#2}}}
\newcommand{\cKey}{\cswatch{blav}{violet}}
\newcommand{\cPos}{\cswatch{bcyan}{cyan}}
\newcommand{\cNeg}{\cswatch{bpink}{magenta}}
\newcommand{\cWarn}{\cswatch{blemon}{lavender}}
\newcommand{\hd}[1]{\textbf{#1}}
\newcommand{\fitto}[2]{%
	\resizebox{\ifdim\width>#1 #1\else\width\fi}{!}{#2}}
\newcommand{\fitcol}[1]{\fitto{\columnwidth}{#1}}
\newcommand{\fitfull}[1]{\fitto{\textwidth}{#1}}

\newcommand{\fabula}{\textit{fabula}}
\newcommand{\syuzhet}{\textit{syuzhet}}
\DeclareRobustCommand{\Task}[1]{\ensuremath{\mathrm{T}_{#1}}}
\DeclareRobustCommand{\ProseTask}[1]{\ensuremath{\mathrm{T}_{5}^{\mathrm{#1}}}}
\DeclareRobustCommand{\GraphCondition}[1]{\ensuremath{\mathrm{G}_{#1}}}
\DeclareRobustCommand{\Ablation}[1]{\ensuremath{\mathrm{A}_{#1}}}
\newcommand{\micon}[1]{%
	\makebox[2.5ex][l]{\raisebox{-0.15\height}{\includegraphics[height=2.1ex]{logos/#1.pdf}}}%
	\hspace{0.4em}%
}
\newcommand{\noicon}{\makebox[2.2ex]{}\hspace{0.4em}}  %

\newcommand{\FigPlane}{%
	\begin{figure*}[t]
		\includegraphics[width=\textwidth]{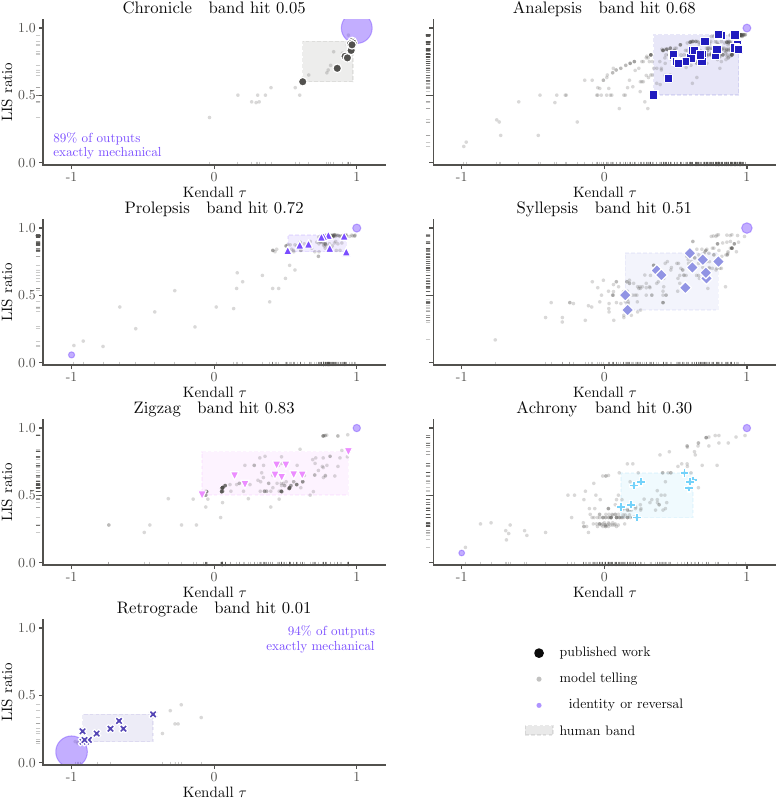}
        \caption{The structural plane, one panel per requested structure: the
        $(\tau, \mathrm{LIS})$ region the gold annotations occupy (dashed box), those
        works (large marks), and every valid cued telling (small dots). Outputs that
        are the exact identity or reversal stack on one coordinate and are drawn once,
        in \cink{Violet}{violet}, with area proportional to their share. The two
        failing categories are dense stacks at the ends of the scale, not near-misses.
        Ticks along the two axes give the marginal distribution of the valid
        tellings, which the overlapping cloud inside the panel obscures.}
        \label{fig:plane}
\end{figure*}}

\newcommand{\FigFloor}{%
	\begin{figure}[t]
		\centering
		\includegraphics[width=\columnwidth]{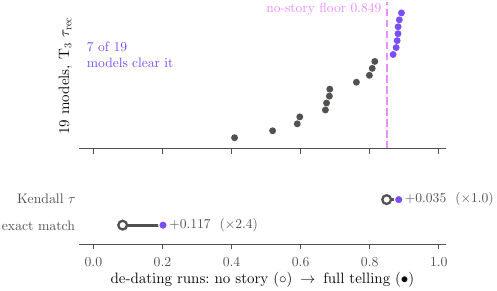}
		\caption{The no-story floor. Top: \Task{3} $\tau_{\mathrm{rec}}$ for each model, against
		a solver given the events with the telling deleted (dashed). The two come
		from separate runs, so the panel calibrates the scores and does not rank
		the models. Bottom: the same solvers with and without the telling, from
		the de-dating runs.}
		\label{fig:floor}
		\vspace{-4mm}
\end{figure}}

\newcommand{\FigGraphArm}{%
	\begin{figure}[t]
		\centering
		\includegraphics[width=\columnwidth]{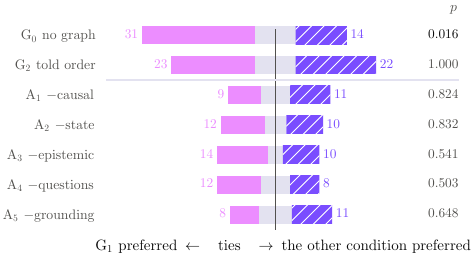}
		\caption{Judge verdicts for every scaffolding comparison: blind
			judgment counts with ties centered, and two-sided exact binomial $p$ over
			decisive pairs. \GraphCondition{1} beats \GraphCondition{0} and no other difference is
			detected. Readers revise two of these rows (\S\ref{sec:human}).}
		\label{fig:grapharm}
\end{figure}}

\newcommand{\FigHumanJudge}{%
	\begin{figure}[t]
		\centering
		\includegraphics[width=\columnwidth]{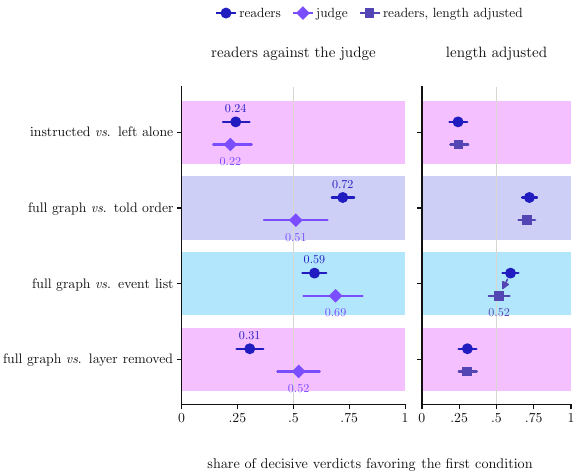}
		\caption{Share of decisive verdicts favoring the first-named condition, from
		\HNReaders{} readers and from the judge, with 95\% intervals
		(reader-clustered bootstrap for readers, exact binomial for the judge).
		The right panel repeats the reader rates after length adjustment. The
		judge's rate matches the readers' on instructed versus free and on graph
		versus event list. On the other two comparisons, graph versus told order
		and graph versus one layer removed, the judge returns ties, and these are
		where readers found the clearest differences.}
		\label{fig:humanjudge}
\end{figure}}

\newcommand{\FigReveal}{%
	\begin{figure}[t]
		\centering
		\includegraphics[width=0.92\columnwidth]{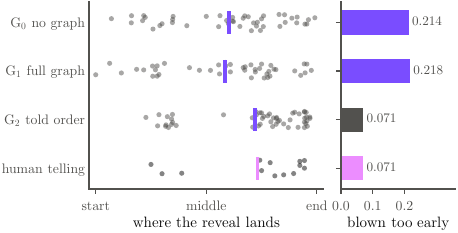}
		\caption{Premature disclosure (\S\ref{sec:reveal}) as the scorer counts it. Left:
		where each story narrates the reveal, with a tick at the condition mean.
		Right: share of stories with fewer than three events before the reveal.
		The bars include a work whose graph designates no reveal, which is scored
		early in every condition, and a chronicle whose explanation any
		chronological telling narrates third. \S\ref{sec:rcg} gives the
		rates without these two works.}
		\label{fig:reveal}
\end{figure}}

\title{Can LLMs Tell Stories Out of Order?\\
	A Corpus and Benchmark for Nonlinear Narrative Generation}

\makeatletter
\AtBeginDocument{%
    \patchcmd{\@maketitle}{\vskip 0.2in plus 1fil minus 0.1in}{\vskip 0.07in}{}%
        {\PackageError{preprint}{Could not adjust the title-author gap}{}}%
    \patchcmd{\@maketitle}{\vskip 0.3in plus 2fil minus 0.1in}{\vskip 0.1in plus 1fil}{}%
        {\PackageError{preprint}{Could not adjust the space below the authors}{}}%
}
\makeatother
\author{%
    Syed Rifat Raiyan$^{1*}$,
    Md.\ Tariquzzaman$^{1*}$,
    Md Farhan Ishmam$^{2\dagger}$,
    Tenghao Huang$^{3}$\\%
    {\normalsize\renewcommand{\arraystretch}{0.96}%
    \begin{tabular}[t]{@{}c@{}}
        $^1$Islamic University of Technology, Bangladesh\\
        $^2$University of Utah, USA\qquad $^3$University of Southern California, USA\\[1.5pt]
        \small\texttt{\{rifatraiyan, tariquzzaman\}@iut-dhaka.edu, farhan.ishmam@utah.edu, tenghaoh@usc.edu}\\[1pt]
        \small\faGithub\;\href{https://github.com/Starscream-11813/LLM-nonlinear-storytelling}{\texttt{Starscream-11813/LLM-nonlinear-storytelling}}\qquad
        \faGlobeAmericas\;\href{https://nonlinear-narrative.github.io/}{\texttt{nonlinear-narrative.github.io}}
    \end{tabular}}%
}
\hypersetup{
    pdftitle={Can LLMs Tell Stories Out of Order? A Corpus and Benchmark for Nonlinear Narrative Generation},
    pdfauthor={Syed Rifat Raiyan, Md. Tariquzzaman, Md Farhan Ishmam, Tenghao Huang}
}

\begin{document}
	\maketitle
    \def\thefootnote{*}\footnotetext{Equal contribution}
    \def\thefootnote{\textdagger}\footnotetext{Corresponding author}
    \def\thefootnote{\arabic{footnote}}
	\begin{abstract}
			Every story has two orders: the order in which its events happen and the order in which they are told. A story is nonlinear when the two differ. We introduce a benchmark of \NWorks{} published works that measures whether language models can tell nonlinear stories. For each work, we annotate one list of events in both orders, together with a story graph of causal dependencies, character knowledge, and the questions the telling raises and resolves.
            Given events in chronological order, all \NModelsOrdering{} models tell them, on average, more chronologically than the published works do; asked for a specific structure, they land within the human range about half the time. We then compare a model's stories of the same events under different conditions, using both an automated judge and \HNReaders{} human readers. Both prefer the chronological story a model writes by default to the nonlinear one it writes on request, which readers find harder to follow in \HCuedConfPct\% of decisive judgments. Readers also prefer stories written from the story graph to stories constrained to the published telling order (\HGOrderBetterPct\% of decisive judgments), a difference the judge does not detect. Our results suggest that following an instruction to reorder a story is not yet the same as telling it well out of order.

	\end{abstract}

    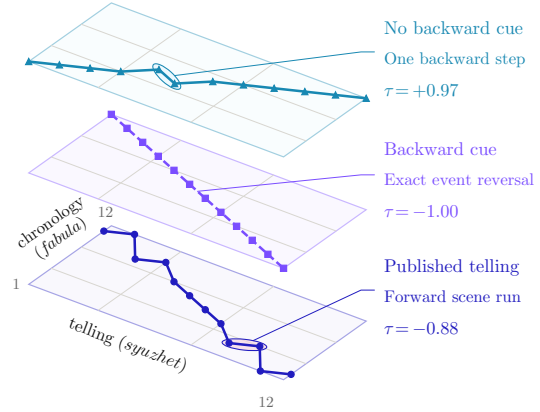
\begin{figure}
        \centering
        \begingroup

\definecolor{teaserPublished}{RGB}{33,28,191}
\definecolor{teaserBackward}{RGB}{122,77,255}
\definecolor{teaserUncued}{RGB}{24,136,178} %
\definecolor{teaserPublishedPlane}{RGB}{221,219,248}
\definecolor{teaserBackwardPlane}{RGB}{230,218,255}
\definecolor{teaserUncuedPlane}{RGB}{221,244,254}
\definecolor{teaserHighlight}{RGB}{247,214,228}
\definecolor{teaserGrid}{RGB}{216,214,208}
\definecolor{teaserAxis}{RGB}{17,17,17}
\definecolor{teaserTick}{RGB}{102,102,102}

\def\teaserPublishedOrder{11,12,9,10,8,7,6,5,3,4,1,2}
\def\teaserBackwardOrder{12,11,10,9,8,7,6,5,4,3,2,1}
\def\teaserUncuedOrder{1,2,3,4,6,5,7,8,9,10,11,12}

\font\teaserFormulaRoman=cmr10 at 6.8bp
\font\teaserFormulaItalic=cmmi10 at 6.8bp
\font\teaserFormulaSymbol=cmsy10 at 6.8bp
\font\teaserTellingItalic=cmmi10 at 6.5bp
\font\teaserChronologyItalic=cmmi10 at 6.2bp

\ExplSyntaxOn
\cs_set_protected:Npn \teaserUnkerned #1
  {
    \str_map_inline:nn {#1}
      {
        \str_if_eq:nnTF {##1} { ~ }
          { \hskip \fontdimen2\font }
          { \hbox{##1} }
      }
  }
\cs_set_protected:Npn \teaserValue #1
  {
    \str_map_inline:nn {#1}
      {
        \str_case:nnF {##1}
          {
            {.}{\hbox~to~1.881744bp{\teaserFormulaItalic\char"3A\hss}}
            {-}{\hbox~to~5.280182bp{\teaserFormulaSymbol\char"00\hss}}
            {+}{\hbox~to~5.280182bp{\teaserFormulaRoman+\hss}}
          }
          {\hbox~to~3.398438bp{\teaserFormulaRoman##1\hss}}
      }
  }
\ExplSyntaxOff

\newcommand{\teaserLabel}[6]{%
  \node[anchor=base west,inner sep=0pt,outer sep=0pt,text=#1,rotate=#5]
    at (#2,#3) {\font\teaserLabelFont=cmr10 at #4bp\relax
      \teaserLabelFont\teaserUnkerned{#6}};%
}
\newcommand{\teaserTau}[4]{%
  \node[anchor=base west,inner sep=0pt,outer sep=0pt,text=#1] at (#2,#3)
    {\hbox to 4.163746bp{\teaserFormulaItalic\char"1C\hss}%
     \hbox to 6.473622bp{\teaserFormulaRoman=\hss}%
     \teaserValue{#4}};%
}

\begin{tikzpicture}[x=1bp,y=-1bp,
  teaser plane/.style={line width=0.45bp,draw opacity=0.40,
    fill opacity=0.17,line cap=butt,line join=miter},
  teaser grid/.style={draw=teaserGrid,line width=0.30bp,
    line cap=round,line join=round},
  teaser trajectory/.style={line width=0.95bp,line cap=round,line join=round},
  teaser marker/.style={line width=0.45bp,line cap=butt,line join=miter},
  teaser leader/.style={line width=0.35bp,line cap=round,line join=round},
  teaser highlight/.style={line width=0.40bp,line cap=butt,line join=miter}]

\path[fill=white] (0,0) rectangle (218.16,162);
\begin{scope}
\clip (0,0) rectangle (218.16,162);

\newcommand{\teaserPlane}[3]{%
  \begin{scope}[shift={(0,#1)}]
    \path[teaser plane,draw=#2,fill=#3]
      (10.08,109.44) -- (105.608727,145.186569)
      -- (136.587361,123.226153) -- (41.058634,87.479584) -- cycle;
    \draw[teaser grid] (18.528718,103.450796) -- (114.057445,139.197365);
    \draw[teaser grid] (36.133289,119.189064) -- (67.111923,97.228649);
    \draw[teaser grid] (29.793676,95.465190) -- (125.322403,131.211759);
    \draw[teaser grid] (70.871008,132.187817) -- (101.849642,110.227401);
  \end{scope}%
}
\teaserPlane{0}{teaserPublished}{teaserPublishedPlane}
\teaserPlane{-41.76}{teaserBackward}{teaserBackwardPlane}
\teaserPlane{-83.52}{teaserUncued}{teaserUncuedPlane}

\newcommand{\teaserCoordinates}[3]{%
  \foreach \rank [count=\event] in #3 {%
    \coordinate (#1-\event) at
      ({10.08+(\event-1)*95.528727/11+(\rank-1)*30.978634/11},
       {109.44+#2+(\event-1)*35.746569/11-(\rank-1)*21.960416/11});
  }%
}
\teaserCoordinates{published}{0}{\teaserPublishedOrder}
\teaserCoordinates{backward}{-41.76}{\teaserBackwardOrder}
\teaserCoordinates{uncued}{-83.52}{\teaserUncuedOrder}

\path[teaser highlight,draw=teaserPublished,fill=teaserHighlight,fill opacity=0.65]
  (90.719223,134.081207)
  .. controls (92.936958,134.322876) and (95.087242,134.342926) .. (96.696487,134.136917)
  .. controls (98.305740,133.930908) and (99.242569,133.515656) .. (99.300659,132.982635)
  .. controls (99.358742,132.449615) and (98.533340,131.842331) .. (97.006241,131.294540)
  .. controls (95.479134,130.746750) and (93.375015,130.303162) .. (91.157272,130.061478)
  .. controls (88.939537,129.819809) and (86.789261,129.799774) .. (85.180008,130.005783)
  .. controls (83.570763,130.211792) and (82.633934,130.627029) .. (82.575844,131.160049)
  .. controls (82.517761,131.693069) and (83.343163,132.300354) .. (84.870262,132.848145)
  .. controls (86.397369,133.395950) and (88.501488,133.839523) .. (90.719223,134.081207)
  -- cycle;
\path[teaser highlight,draw=teaserUncued,fill=teaserUncuedPlane,fill opacity=0.40]
  (60.285915,33.290344)
  .. controls (61.613338,34.477051) and (63.049587,35.432961) .. (64.278343,35.947563)
  .. controls (65.507095,36.462158) and (66.428032,36.493431) .. (66.838326,36.034477)
  .. controls (67.248619,35.575531) and (67.114777,34.663834) .. (66.466263,33.500183)
  .. controls (65.817757,32.336533) and (64.707535,31.015930) .. (63.380108,29.829239)
  .. controls (62.052685,28.642532) and (60.616436,27.686615) .. (59.387684,27.172012)
  .. controls (58.158928,26.657425) and (57.237991,26.626144) .. (56.827698,27.085098)
  .. controls (56.417400,27.544052) and (56.551247,28.455750) .. (57.199757,29.619400)
  .. controls (57.848267,30.783051) and (58.958488,32.103653) .. (60.285915,33.290344)
  -- cycle;

\newcommand{\teaserCircle}[2]{%
  \begin{scope}[shift={(#1)}]
    \path[teaser marker,draw=#2,fill=#2]
      (0,1.15)
      .. controls (0.304985,1.15) and (0.597520,1.028829) .. (0.813173,0.813173)
      .. controls (1.028829,0.597520) and (1.15,0.304985) .. (1.15,0)
      .. controls (1.15,-0.304985) and (1.028829,-0.597520) .. (0.813173,-0.813173)
      .. controls (0.597520,-1.028829) and (0.304985,-1.15) .. (0,-1.15)
      .. controls (-0.304985,-1.15) and (-0.597520,-1.028829) .. (-0.813173,-0.813173)
      .. controls (-1.028829,-0.597520) and (-1.15,-0.304985) .. (-1.15,0)
      .. controls (-1.15,0.304985) and (-1.028829,0.597520) .. (-0.813173,0.813173)
      .. controls (-0.597520,1.028829) and (-0.304985,1.15) .. (0,1.15) -- cycle;
  \end{scope}%
}
\newcommand{\teaserSquare}[2]{%
  \begin{scope}[shift={(#1)}]
    \path[teaser marker,draw=#2,fill=#2] (-1.15,-1.15) rectangle (1.15,1.15);
  \end{scope}%
}
\newcommand{\teaserTriangle}[2]{%
  \begin{scope}[shift={(#1)}]
    \path[teaser marker,draw=#2,fill=#2]
      (0,-1.15) -- (-1.15,1.15) -- (1.15,1.15) -- cycle;
  \end{scope}%
}
\draw[teaser trajectory,draw=teaserPublished]
  (published-1) \foreach \event in {2,...,12} {-- (published-\event)};
\foreach \event in {1,...,12} {\teaserCircle{published-\event}{teaserPublished}}

\draw[teaser trajectory,draw=teaserBackward,line cap=butt,
  dash pattern=on 2.945bp off 1.425bp]
  (backward-1) \foreach \event in {2,...,12} {-- (backward-\event)};
\foreach \event in {1,...,12} {\teaserSquare{backward-\event}{teaserBackward}}

\draw[teaser trajectory,draw=teaserUncued]
  (uncued-1) \foreach \event in {2,...,12} {-- (uncued-\event)};
\foreach \event in {1,...,12} {\teaserTriangle{uncued-\event}{teaserUncued}}

\draw[teaser leader,draw=teaserUncued]
  (139.68,22.32) -- (134.64,22.32) -- (64.754105,31.189087);
\draw[teaser leader,draw=teaserBackward]
  (139.68,68.40) -- (134.64,68.40) -- (73.333679,74.573074);
\draw[teaser leader,draw=teaserPublished]
  (139.68,111.60) -- (134.64,111.60) -- (94.214119,130.536819);

\teaserLabel{teaserUncued}{143.28}{15.581253}{6.9}{0}{No backward cue}
\teaserLabel{teaserUncued}{143.28}{26.757507}{6.3}{0}{One backward step}
\teaserTau{teaserUncued}{143.28}{38.872498}{+0.97}

\teaserLabel{teaserBackward}{143.28}{60.941254}{6.9}{0}{Backward cue}
\teaserLabel{teaserBackward}{143.28}{72.055000}{6.3}{0}{Exact event reversal}
\teaserTau{teaserBackward}{143.28}{84.169998}{-1.00}

\teaserLabel{teaserPublished}{143.28}{104.861252}{6.9}{0}{Published telling}
\teaserLabel{teaserPublished}{143.28}{115.974998}{6.3}{0}{Forward scene run}
\teaserTau{teaserPublished}{143.28}{128.089996}{-0.88}

\teaserLabel{teaserAxis}{24.682627}{125.597466}{6.5}{-20.515779}{telling (}
\node[anchor=base west,inner sep=0pt,outer sep=0pt,text=teaserAxis,rotate=-20.515779]
  at (45.618168,133.431488) {\teaserTellingItalic\teaserUnkerned{syuzhet}};
\teaserLabel{teaserAxis}{66.295097}{141.168732}{6.5}{-20.515779}{)}

\teaserLabel{teaserAxis}{7.232974}{96.047554}{6.2}{35.330000}{chronology}
\teaserLabel{teaserAxis}{13.849171}{99.842064}{6.2}{35.330000}{(}
\node[anchor=base west,inner sep=0pt,outer sep=0pt,text=teaserAxis,rotate=35.330000]
  at (15.808655,98.453011) {\teaserChronologyItalic\teaserUnkerned{fabula}};
\teaserLabel{teaserAxis}{30.160950}{88.278824}{6.2}{35.330000}{)}

\teaserLabel{teaserTick}{3.946875}{111.245316}{5.8}{0}{1}
\teaserLabel{teaserTick}{96.222473}{155.271881}{5.8}{0}{12}
\teaserLabel{teaserTick}{35.992382}{84.244896}{5.8}{0}{12}

\end{scope}
\pgfresetboundingbox
\path[use as bounding box] (0bp,0bp) rectangle (218.16bp,-162bp);
\end{tikzpicture}%
\endgroup
        \caption{Without a cue, the model stays near chronology; a backward cue
        yields exact reversal, while the published telling preserves forward
        scene runs. Kendall's $\tau$ measures agreement with chronology:
        $+1$ denotes chronological order and $-1$ exact reversal.}
        \label{fig:teaser-a}
        \vspace{-10pt}
    \end{figure}

	\vspace{-2mm}
	\section{Introduction}
\label{sec:intro}

A story doesn't need to start at the inception of its chronology.
The sequence in which a teller presents events may depart from the order in
which they actually happened, producing a nonlinear narrative
\citep{genette1980narrative}.
\emph{Memento} interleaves a backward strand with a forward one;
\emph{Citizen Kane} opens at the end of its protagonist's life and
reconstructs that life through flashbacks \citep{singer2016citizenkane}.
In both, much of the effect lies in the order in which we are allowed to
learn the events: the telling decides what a reader knows at each point
and how each new scene is read. We call the structural pattern of a
telling order its \emph{narrative shape}.

Story generation research has built systems that plan stories, revise
long drafts and build suspense
\citep{fan-etal-2018-hierarchical,yao2019plan,yang-etal-2022-re3,xie-riedl-2024-creating},
and some add flashbacks or rewrite a chronological draft into a planned
telling order \citep{han-etal-2022-go,yu-etal-2025-multi}. Yet the
telling order is handled tacitly, planned or prompted for but seldom
recorded, and model-written stories remain less structurally diverse than
human ones \citep{tian-etal-2024-large}. Without that record, the two
properties that make a story nonlinear cannot even be stated: a revelation
counts as one only if the explanation is narrated after the event it
explains, and a question withholds something only if its answer arrives
after it is asked. We therefore record both orders for \NWorks{} published
works, with \emph{story graphs} whose typed edges encode these relations
and whose nodes record world state and what characters and readers know.

We ask five questions:
\begin{description}\itemsep0pt\parskip0pt
    \item[RQ1] \emph{Controllability.} Can a model set the order in which it
    narrates a fixed set of events, unprompted and when instructed?
    (\S\ref{sec:capability})
    \item[RQ2] \emph{Task validity.} Can it recover chronology and name the
    shape, and how much of that score survives deleting the telling?
    (\S\ref{sec:reading})
    \item[RQ3] \emph{Representation.} Does a story graph beat a bare event
    list or a supplied order? (\S\ref{sec:graph})
    \item[RQ4] \emph{Cost of control.} What does a structural instruction
    cost in quality and in the placement of withheld information?
    (\S\ref{sec:worth})
    \item[RQ5] \emph{Evaluation validity.} Do a judge's verdicts comport with
    readers', and do our metrics track reading experience?
    (\S\ref{sec:human})
\end{description}
Given a shape stipulated as a rule on events, models return its formulaic
solution, which no published telling in our corpus uses
(Figure~\ref{fig:teaser-a}). Readers prefer a model's default chronological
telling to its instructed one; they also prefer graph-guided stories to
stories written to a supplied order, a preference the judge misses. Our
contributions are threefold:
\begin{enumerate}\itemsep0pt\parskip0pt
    \item We annotate \NWorks{} published works across seven narrative
    shapes, \NNonlinearWorks{} nonlinear and ten predominantly
    chronological, each with chronological and telling orders over one
    event list and a hand-written story graph (\S\ref{sec:corpus}).
    \item We build a five-task benchmark for producing and reading telling
    orders, with model-free baselines and a no-story floor for
    reconstruction, and evaluate \NModelsOrdering{} models on it
    (\S\ref{sec:setup}).
    \item We test an automated judge and our own structural metrics
    against \HNJudgments{} blind judgments from \HNReaders{} readers
    (\S\ref{sec:human}).
\end{enumerate}

\begin{figure*}[t]
    \centering
    \includegraphics[width=\linewidth]{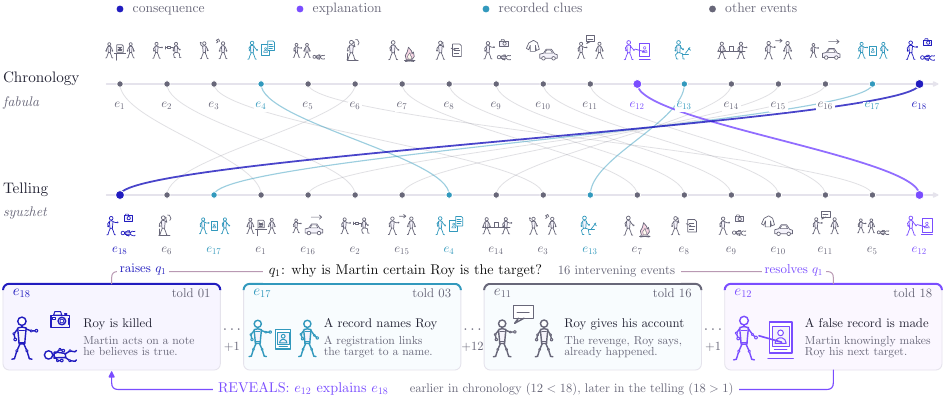}
    \caption{The same events, two orders: the explanation precedes the consequence in chronology but arrives last in the telling. Opening with Roy's death and withholding why he became the target across 16 intervening events turns a known outcome into an unresolved question.}
    \label{fig:overview}
    \vspace{-13pt}
\end{figure*}
	\vspace{-3mm}
	\section{Related work}
\label{sec:related}

Story generation with LLMs often separates planning from prose
realization: hierarchical generation and Plan-and-Write plan before
writing \citep{fan-etal-2018-hierarchical,yao2019plan}, and Re3 adds
recursive prompting and revision for longer stories
\citep{yang-etal-2022-re3}. Telling order has computational precedents:
\citet{montfort2007ordering} generates alternative tellings from a shared
event model using seven ordering categories. For neural generation,
\citet{han-etal-2022-go} show that temporal prompts produce more
interesting flashback stories while maintaining temporal coherence;
\citet{yu-etal-2025-multi} separate chronological character simulation
from rewriting into a planned telling order.
Our benchmark examines this control across seven structural categories
drawn from published narratives while holding the events fixed.

Temporal resources such as TimeBank and TempEval support annotation and
evaluation of temporal relations in text
\citep{pustejovsky2003timebank,uzzaman-etal-2013-semeval}, while DramaBank
represents agency, intentions, and plans \citep{elson-2012-dramabank}.
Story Curves directly aligns chronological and telling orders across ten
nonlinear films and identifies recurring narrative patterns
\citep{kim2018storycurves}. Our corpus extends this paired-order setting
with shared event descriptions and story graphs, making telling
order a target for controlled generation.

Assessing reading experience also requires human evaluation. The
Psychological Depth Scale measures empathy, engagement, emotional
provocation, and related dimensions \citep{harelcanada2024measuring}.
Prior work compares model story ratings with English teachers' ratings
\citep{chiang-lee-2023-large}; \citet{chhun-etal-2024-language} find stronger
agreement with human ratings when ranking generation systems than when
evaluating individual stories. LLM judges exhibit position, verbosity,
and self-enhancement biases \citep{zheng2023judging}, motivating
length-controlled evaluation \citep{dubois2024length}.

These lines leave open how reliably models control the shapes of
published narratives, whether reconstruction draws on the telling at all,
and whether structural and automated measures track readers.
\S\ref{sec:related-extra} discusses each line in detail.

	\vspace{-1mm}
	\section{Methodology}
\label{sec:methodology}
\label{sec:corpus}
\label{sec:setting}

Each dataset item represents one narrative work with 12--20 key story events. We write one self-contained sentence per event so that the same sentences
can be read in chronological order or in the work's telling order. Each entry records both orders, a structural label and a story graph.

For events $E=(e_1,\dots,e_n)$ in chronological order (the \fabula{}), the
telling (the \syuzhet{}) is a permutation $\sigma\in\mathcal{S}_n$ presenting
$T_\sigma=(e_{\sigma(1)},\dots,e_{\sigma(n)})$
\citep{genette1980narrative}. Thus $\sigma(i)$ gives the chronological rank
of the event at told position $i$. Strict chronological order is the identity,
and a telling is nonlinear when $\sigma\neq\mathrm{id}$.
Each work has an annotated reference (gold) telling $\sigma^\star$
and category $c^\star$. A scene is an annotated group of events that
remains contiguous in both orders. Structural labels describe scene
order, while $\sigma$ records event order.
A permutation narrates every selected event exactly once, so the
representation isolates the placement of plot events: repeated
narration, staged disclosure and omission fall outside it, and prose is
projected onto the event list at first mention (\S\ref{sec:alignvalidation}).

We selected \NWorks{} works by structural shape, beginning with films and
adding novels, plays, musicals and a short story until every category had
at least ten examples. This sampling provides coverage of the seven shapes
rather than estimates of their prevalence. Table~\ref{tab:corpus}
reports the corpus counts.

\begin{table}[t]
		\centering\small
		\fitcol{\begin{tabular}{lr}
\toprule
\textbf{Corpus}     &     \\
\midrule
works, each annotated in both orders & 83 \\
gold narrative graphs & 83 \\
event nodes & 1,433 \\
scenes & 1,221 \\
characters & 857 \\
mean $\tau$ of the told order & 0.456 \\
\midrule
\textbf{Typed edges} & 3,478 \\
\quad \texttt{CAUSES} & 1,330 \\
\quad \texttt{SHARES} & 1,397 \\
\quad \texttt{FORESHADOWS} & 444 \\
\quad \texttt{REVEALS} & 307 \\
\midrule
reader-question threads & 657 \\
pre-/post-conditions & 1,459 / 1,688 \\
character/reader knowledge claims & 2,504 \\
\bottomrule
\end{tabular}
}
		\caption{The corpus. Counts are recomputed from the annotation and
			independently re-derived by a separate check.}
		\label{tab:corpus}
		\vspace{-4mm}
	\end{table}

The seven structural labels describe scene order:
\textsc{chronicle} (predominantly chronological);
\textsc{retrograde} (scenes in reverse chronology, with events
within each scene proceeding forward);
\textsc{analepsis} (a dominant timeline interrupted by flashbacks,
followed by returns);
\textsc{prolepsis} (a dominant timeline interrupted by
flash-forwards, followed by returns);
\textsc{zigzag} (regular alternation between related scenes from
two periods, each progressing in a consistent temporal direction);
\textsc{syllepsis} (no dominant timeline, with scenes grouped by
place, character, theme or object); and
\textsc{achrony} (no consistent ordering rule).
\S\ref{sec:shapes} gives full definitions and illustrations,
and \S\ref{sec:prompts} gives the rules presented to models.
We include only works with a recoverable chronology, which provides
the reference order for evaluation. Our \textsc{achrony} label therefore 
denotes no consistent telling-order rule despite a recoverable chronology.

The story graph records causal dependencies, world-state constraints,
character knowledge, and reader questions that an event order alone
does not express. In our annotation, a \emph{reveal} is an event that
occurs before the event it explains in chronology but appears after
it in the telling. A \emph{reader question} is an uncertainty raised
by one narrated event and resolved by a later one.
A \texttt{REVEALS} edge points from the explanation to the event it
explains and is associated with a matching reader question.
Figure~\ref{fig:overview} illustrates a reveal and its reader question.

Node \texttt{ei} is event $i$, which links the graph to the event list.
Graphs supplied to models omit each event's chronological and told
positions, although \texttt{REVEALS} edges still constrain the telling
order. The graph
supplies constraints rather than a generation algorithm: causes support
consequences, world states stay consistent, characters act on what they
have learned, and questions the telling opens get answers.
\S\ref{sec:validation} specifies every field, and Figure~\ref{fig:graph}
draws one annotated work with its four reveals and the reader questions
they answer.

	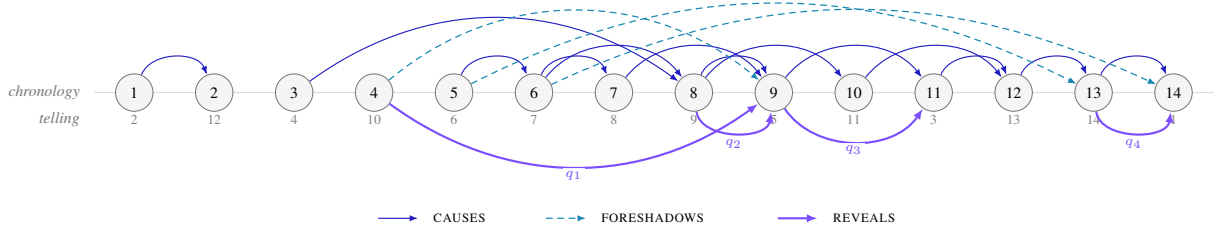
\begin{figure*}[t]
		\centering
		\begingroup%
\resizebox{\textwidth}{!}{%
\begin{tikzpicture}[
		x=1.24cm, y=1cm,
		evt/.style={circle, draw=storyGraphNodeOutline, line width=0.45pt,
			fill=storyGraphNodeFill, text=black,
			minimum size=5.8mm, inner sep=0pt, font=\scriptsize},
		told/.style={font=\tiny, text=black!50, inner sep=1pt},
		rowlab/.style={font=\scriptsize\itshape, text=black!55, anchor=east,
			inner sep=2pt},
		causes/.style={draw=storyGraphIndigo, line width=0.5pt,
			-{Latex[length=1.3mm,width=1.0mm]}},
		fores/.style={draw=storyGraphCyan, line width=0.55pt, densely dashed,
			-{Latex[length=1.3mm,width=1.0mm]}},
		reveals/.style={draw=storyGraphViolet, line width=0.95pt,
			-{Latex[length=1.7mm,width=1.3mm]}},
		qlab/.style={font=\tiny, text=storyGraphViolet, inner sep=1.1pt, fill=white,
			fill opacity=0.9, text opacity=1}]
\definecolor{storyGraphIndigo}{HTML}{211CBF}
\definecolor{storyGraphCyan}{HTML}{1888B2}
\definecolor{storyGraphViolet}{HTML}{7A4DFF}
\definecolor{storyGraphNodeFill}{HTML}{F5F5F5}
\definecolor{storyGraphNodeOutline}{HTML}{777777}
\definecolor{storyGraphBaseline}{HTML}{D9D9D9}

\draw[storyGraphBaseline, line width=0.4pt] (-0.5,0) -- (13.5,0);
		\foreach \i/\t in {0/2, 1/12, 2/4, 3/10, 4/6, 5/7, 6/8, 7/9, 8/5, 9/11,
			10/3, 11/13, 12/14, 13/1} {
			\pgfmathtruncatemacro{\lab}{\i+1}
			\node[evt] (e\lab) at (\i,0) {\lab};
			\node[told] at (\i,-0.42) {\t};
		}
		\node[rowlab] at (-0.62,0)     {chronology};
		\node[rowlab] at (-0.62,-0.42) {telling};
\foreach \a/\b in {1/2, 5/6, 6/7, 8/9, 11/12, 12/13, 13/14}
		\draw[causes] (e\a) to[out=68,in=112] (e\b);
		\foreach \a/\b in {6/8, 7/9, 8/10, 9/11, 10/12}
		\draw[causes] (e\a) to[out=52,in=128] (e\b);
		\draw[causes] (e3) to[out=36,in=144] (e8);
		
\draw[fores] (e4) to[out=40,in=140] (e9);
		\draw[fores] (e5) to[out=27,in=153] (e13);
		\draw[fores] (e6) to[out=25,in=155] (e14);
		
\draw[reveals] (e4)  to[out=-36,in=-144] node[qlab,below=0pt] {$q_1$} (e9);
		\draw[reveals] (e8)  to[out=-82,in=-98]  node[qlab,below=1pt] {$q_2$} (e9);
		\draw[reveals] (e9)  to[out=-56,in=-124] node[qlab,below=0pt] {$q_3$} (e11);
		\draw[reveals] (e13) to[out=-82,in=-98]  node[qlab,below=1pt] {$q_4$} (e14);
		
\begin{scope}[shift={(3.05,-1.95)}]
			\draw[causes] (0,0) -- (0.5,0);
			\node[font=\scriptsize, anchor=west] at (0.58,0) {\textsc{causes}};
			\draw[fores] (2.1,0) -- (2.6,0);
			\node[font=\scriptsize, anchor=west] at (2.68,0) {\textsc{foreshadows}};
			\draw[reveals] (5.0,0) -- (5.5,0);
			\node[font=\scriptsize, anchor=west] at (5.58,0) {\textsc{reveals}};
		\end{scope}
\end{tikzpicture}%
}%
\endgroup%
			\caption{Story graph of work S60, a musical labeled \textsc{zigzag}, with
			events in chronological order and told positions beneath. \textsc{causes}
			and \textsc{foreshadows} arcs run above the line; each \textsc{reveals} arc
			below points from a later-told explanation to the earlier-told event it
			explains, answering the reader question $q_i$ it carries (\textsc{shares}
			edges omitted).
			{\footnotesize\textbf{Events:} 1~a novelist meets an actress; 2~she tells
			him of the man she nearly settled for; 3~his first book is accepted;
			4~a casting panel ignores her; 5~he promises never to leave; 6~they marry;
			7~she takes a provincial engagement; 8~his book succeeds; 9~she confronts
			his absence; 10~a guest propositions him; 11~she resents his success;
			12~he is unfaithful; 13~he leaves a letter; 14~she finds the flat empty.
			\textbf{Questions:} $q_1$~what professional hurt lies beneath her accusation;
			$q_2$~how his success became distance; $q_3$~why she can no longer be
			glad for him; $q_4$~why the flat is empty.}}
		\label{fig:graph}
		\vspace{-4mm}
\end{figure*}

The authors annotated each work once, writing its event sentences, its
label with a rationale, and its graph. Machine checks reject any work that
violates a constraint; they catch inconsistency rather than misreading, so
every label carries its stated rationale. Models are evaluated on
anonymized versions that replace identifying details but preserve event
counts and permutations, and validation rejects source proper nouns and
verbatim source sentences (\S\ref{sec:validation}).

\vspace{-2mm}
\section{Experimental Setup}
\label{sec:setup}
\vspace{-1mm}
\subsection{Experimental Design}
\label{sec:tasks}
\label{sec:graphprotocol}

\paragraph{Tasks.}
Given events in chronological order, \Task{1} asks for a telling order that
realizes a named shape and \Task{2} for one without a shape instruction. Given
the same events in their gold telling order, \Task{3} asks for the underlying
chronology and \Task{4} for one of the seven structural labels. \ProseTask{a} and
\ProseTask{b} ask for a 550--750-word story from the chronological event list,
with a named shape in \ProseTask{a} and no shape instruction in \ProseTask{b}. Tables also label
the cued and free conditions \emph{instructed} and \emph{left alone},
respectively. The prose and graph studies use
\ProseNStories{} works, two per shape (models in Table~\ref{tab:roster}).

\paragraph{Inputs and decoding.}
Base prompts list one self-contained event sentence per line and mark
neither scene boundaries nor simultaneous events. Ordering uses temperature 0 and low
reasoning effort; prose uses 0.7. \S\ref{sec:prompts}
has the verbatim prompts; \S\ref{sec:grid-notes} has the
decoding and scoring conventions.

\paragraph{Reconstruction controls.}
\Task{3} receives the gold telling order. To test how much of its score
comes from that order rather than from the event descriptions, 
three strong models also reconstruct chronology from the same
events in random order; we call their score the \emph{no-story floor}. We
then repeat both runs after a model rewrites the events to remove explicit
time cues such as dates and ``years earlier'' (Table~\ref{tab:dedate};
\S\ref{sec:setup-extra}).
In both reconstruction conditions, indices identify event lines
rather than gold chronological ranks (\S\ref{sec:grid-notes}).

\paragraph{Graph conditions.}
All graph conditions retain the \ProseTask{a} target-shape cue and its
definition; only the added input differs. \GraphCondition{0} adds nothing. \GraphCondition{1} adds the full graph without
its chronological and telling position fields. \GraphCondition{2} adds no graph
but supplies the gold telling order and requires the story to follow it.
The graph still constrains order: its \texttt{REVEALS} edges place each
explanation after the event it explains, fixing \PinnedPct\% of
telling-order pairs. Five ablations each remove one annotation group from
\GraphCondition{1} and its prompt description, such as character knowledge or
world-state conditions (Table~\ref{tab:grapharm}; \S\ref{sec:setup-extra}).

\subsection{Evaluation}
\label{sec:evaluation}
\label{sec:autoeval}
\label{sec:humanproto}

\paragraph{Order metrics.}
For \Task{1} and \Task{2}, forwardness is Kendall's $\tau$ between telling and
chronology \citep{kendall1938new}, from $+1$ for chronological order to
$-1$ for exact reversal. A telling is effectively chronological at
$\tau\geq0.95$; gold tellings average $\tau=\CorpusTau$. The normalized
longest increasing subsequence, $\mathrm{LIS}$, is the largest fraction of
events told in chronological order, allowing gaps. For permutations and for
aligned prose, a \emph{band hit} means that both $\tau$ and $\mathrm{LIS}$
fall within the minimum--maximum range of the gold tellings for the target
shape: the requested shape for cued outputs (\Task{1}, \ProseTask{a} and \GraphCondition{0}--\GraphCondition{2}), and the
source work's gold category for free outputs (\Task{2} and \ProseTask{b}). Band hit
measures resemblance to corpus examples, not the shape itself: different
patterns can share both coordinates, so the band sees displacement from
chronology but not alternation or grouping (\S\ref{sec:setup-extra}).
Agreement with a work's gold telling is
Kendall's $\tau$ between the two orders. \Task{3} is scored by
$\tau_{\mathrm{rec}}$ against gold chronology and by exact match, and \Task{4} by
accuracy (Appendices~\ref{sec:formal} and~\ref{sec:grid-notes}).

\paragraph{Aligning prose.}
\label{sec:alignvalidation}
Gemini~3.1~Pro aligns each generated story back to the event list by first
mention, so the structural scores above apply to prose as well as permutations. An
aligner from another model family agrees on \InterJudgeN{} shared stories
at $r=\InterJudgeR$ (mean $|\Delta\tau|=\InterJudgeDiff$), and when given
\GraphCondition{2}, the primary aligner recovers the gold order at
$\tau=\GTwoAgree$. These checks validate alignment, not preference.

\paragraph{Automated story judgments.}
Gemini~3.1~Pro also serves as the primary preference judge. It compares two
stories generated by the same model for the same work under different
conditions and answers the question used in the reader study: ``Which is
the better story overall: more engaging, better told, more satisfying to
read?'' Each pair is judged twice with positions swapped, and disagreement 
is recorded as a tie. The judge also rates each story on the five 1--5 scales of the
Psychological Depth Scale (PDS; \citealp{harelcanada2024measuring}), which
readers use in close readings. \S\ref{sec:prompts} gives all judge prompts.

\paragraph{Reader protocol and inference.}
\HNReaders{} readers gave \HNJudgments{} responses over \HNSittings{}
sittings. Each sitting contains four blind comparisons and one close
reading. A comparison shows the same kind of pair the judge sees, with
sides randomized and conditions hidden. Readers answer the judge's overall
question plus ``Which was harder to follow?'' and ``Whose ending tied
things together better?'' (Story A, Story B or about the same). One of the 
four repeats an earlier pair with sides swapped, without telling the 
reader (\S\ref{sec:instrument}). Preference rates exclude ties. We report
exact binomial tests with reader-clustered bootstrap intervals, rates
adjusted for word-count difference with generalized estimating equations,
and rates on length-matched pairs within 40 words (\S\ref{sec:setup-extra}).

	\vspace{-2mm}
	\section{Results}
	\subsection{RQ1. Producing a telling}
	\label{sec:capability}

	\begin{table*}[t]
		\centering\scriptsize
		\setlength{\tabcolsep}{4pt}
		\renewcommand{\arraystretch}{0.92}
		\fitfull{\begin{tabular}{@{}l cccccc @{\hskip 9pt} l cccccc@{}}
\toprule
& \multicolumn{2}{c}{{\boldmath\Task{1}} cued} & \multicolumn{2}{c}{{\boldmath\Task{3}} reconstruct} & {\boldmath\Task{4}} & {\boldmath\Task{2}} & & \multicolumn{2}{c}{{\boldmath\Task{1}} cued} & \multicolumn{2}{c}{{\boldmath\Task{3}} reconstruct} & {\boldmath\Task{4}} & {\boldmath\Task{2}} \\
\cmidrule(lr){2-3}\cmidrule(lr){4-5}\cmidrule(lr){6-6}\cmidrule(lr){7-7}\cmidrule(lr){9-10}\cmidrule(lr){11-12}\cmidrule(lr){13-13}\cmidrule(l){14-14}
\textbf{Model} & \textbf{band hit} & $\tau$ & $\tau_{\mathrm{rec}}$ & \textbf{exact} & \textbf{acc.\ } & $\tau$ \textbf{free} & \textbf{Model} & \textbf{band hit} & $\tau$ & $\tau_{\mathrm{rec}}$ & \textbf{exact} & \textbf{acc.\ } & $\tau$ \textbf{free} \\
\midrule
\micon{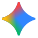}Gemini 3.1 Pro & \cellcolor{bcyan!45}{\textbf{0.602}} & 0.409 & 0.882 & \cellcolor{bcyan!45}\textbf{0.253} & 0.747 & 0.785 & \micon{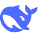}\textit{DeepSeek V4 Pro} & 0.470 & 0.442 & 0.519 & 0.060 & 0.470 & 0.678 \\
\micon{gemini}Gemini 3.6 Flash & 0.578 & \cellcolor{bcyan!45}\textbf{0.456} & 0.880 & 0.217 & \cellcolor{bcyan!45}\textbf{0.795} & 0.808 & \micon{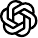}GPT-5.6 Terra & 0.438 & 0.410 & 0.882 & 0.171 & 0.648 & 0.531 \\
\micon{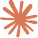}Claude Opus 5 & 0.566 & 0.416 & 0.886 & 0.193 & 0.614 & 0.615 & \micon{claude}Claude Haiku 4.5 & 0.434 & 0.376 & 0.799 & 0.108 & 0.470 & 0.681 \\
\micon{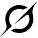}Grok 4.5 & 0.566 & 0.398 & 0.877 & 0.205 & 0.699 & 0.485 & \micon{gpt}\textit{gpt-oss-120b} & 0.434 & 0.218 & 0.675 & 0.024 & 0.277 & 0.568 \\
\micon{gpt}GPT-5.6 Luna & 0.550 & 0.413 & 0.815 & 0.129 & 0.486 & 0.664 & \micon{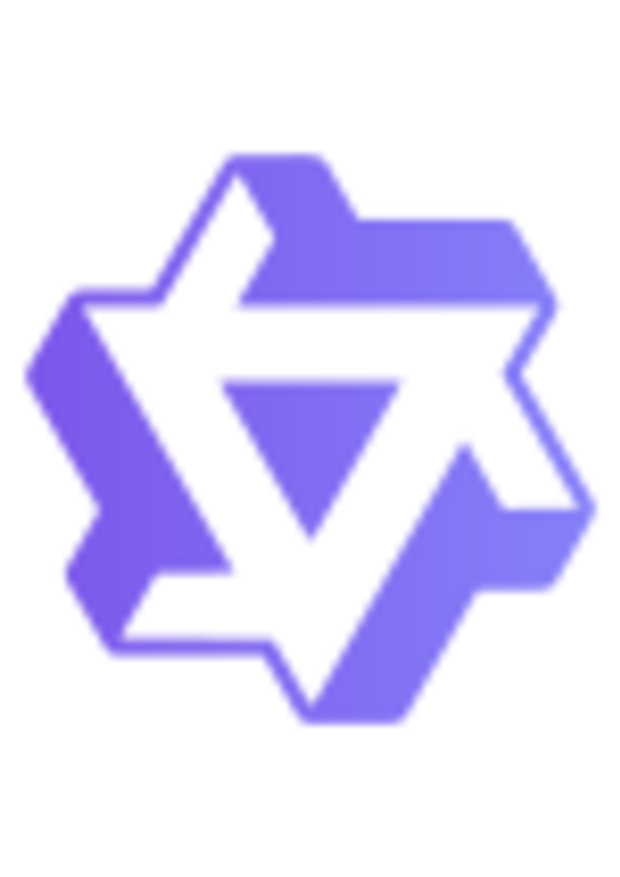}\textit{Qwen3-235B-A22B} & 0.410 & 0.314 & 0.683 & 0.012 & 0.410 & 0.532 \\
\micon{claude}Claude Sonnet 5 & 0.542 & 0.458 & 0.807 & 0.072 & 0.422 & 0.793 & \micon{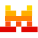}\textit{Mistral Small 3.2} & 0.398 & 0.311 & 0.598 & 0.024 & 0.313 & 0.610 \\
\micon{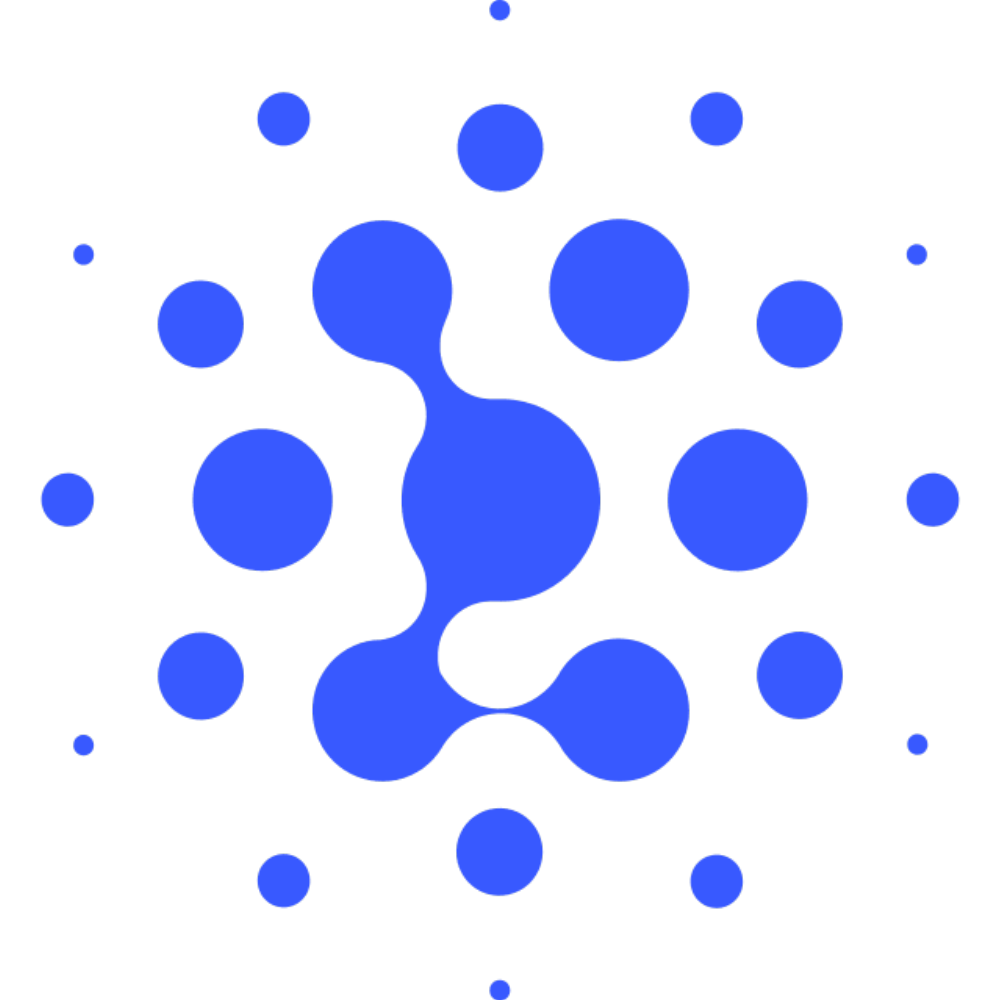}\textit{GLM-5.2} & 0.542 & 0.472 & 0.868 & 0.193 & 0.602 & 0.708 & \micon{deepseek}\textit{DeepSeek V4 Flash} & 0.386 & 0.447 & 0.684 & 0.108 & 0.410 & 0.604 \\
\micon{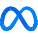}\textit{Llama 4 Maverick} & 0.530 & 0.408 & 0.590 & 0.000 & 0.301 & \cellcolor{bcyan!45}\textbf{0.480} & \micon{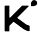}\textit{Kimi K2 0905} & 0.361 & 0.421 & 0.761 & 0.060 & 0.422 & 0.760 \\
\micon{gpt}GPT-5.6 Sol & 0.506 & 0.383 & \cellcolor{bcyan!45}\textbf{0.892} & 0.243 & 0.671 & 0.626 & \micon{meta}\textit{Llama 3.3 70B} & 0.289 & 0.108 & 0.672 & 0.012 & 0.313 & 0.614 \\
\micon{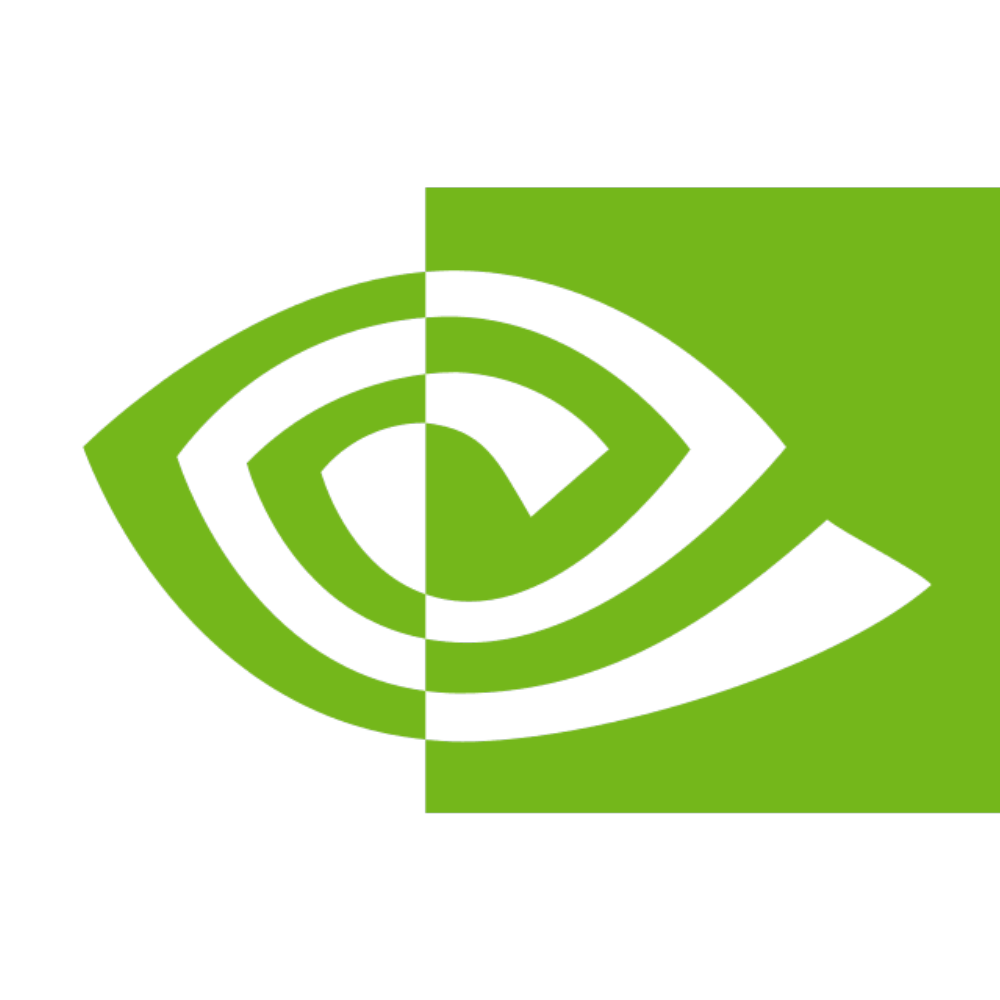}\textit{Nemotron 3 Ultra} & 0.494 & 0.484 & 0.409 & 0.060 & 0.506 & 0.562 &  &  &  &  &  &  &  \\
\midrule
\noicon\textit{pooled} & \textit{0.479} & \textit{0.386} & \textit{0.743} & \textit{0.111} & \textit{0.502} & \textit{0.638} &  &  &  &  &  &  &  \\
\noicon\textit{[reference] human} & \textit{1.000}$^{\dagger}$ & \textit{0.456} & --- & --- & --- & \textit{0.456} &  &  &  &  &  &  &  \\
\noicon\textit{[baseline] random} & \textit{0.084} & \textit{0.015} & --- & --- & \textit{0.143} & \textit{0.015} &  &  &  &  &  &  &  \\
\bottomrule
\end{tabular}
}
		\caption{Benchmark grid, sorted by \Task{1} band hit (left block, then
			right). \Task{1}: structural compliance. \Task{2}: spontaneous ordering. \Task{3}:
			reconstruction. \Task{4}: classification. \textit{Italic}: open-weights.
			\hd{Bold}: best per column (in the $\tau$ columns, closest to the human
			\CorpusTau, descriptive only). Human band hit ($\dagger$) is 1 by
			definition. Band-hit intervals are about $\pm\BandHitCIHalf$, so adjacent
			ranks are not distinguishable (\S\ref{sec:grid-notes}). Read \Task{3}
			$\tau_{\mathrm{rec}}$ against the no-story floor
			(\S\ref{sec:floor}).}
		\label{tab:main}
		\vspace{-4mm}
	\end{table*}

	\subsubsection{Left alone, they tell it straight}
	\label{sec:spontaneous}

	Given the events and no structural instruction (\Task{2}, every model in
	Table~\ref{tab:main}), the pooled told order sits at $\tau = \PoolTauFree$
	and $\mathrm{LIS} = \PoolLisFree$, against a human corpus mean of
	\CorpusTau{}. Models vary widely (\TauFreeLo--\TauFreeHi), yet all
	\NAboveHuman{} are more chronological than the human telling of the same
	works. Even the least chronological (\LeastChronoModel{},
	\LeastChronoTau{}) sits above the corpus mean (two-sided sign test, $p =
	\FreeSignP$). Prose is more chronological still ($\tau = \ProseFreeTau$),
	and \ProseFullChronoModels{} of \ProseNModels{} models write every story
	in chronological order. With no shape requested, \ProseFreeBandPct\% land
	in the band of the source work's own label, which measures resemblance,
	not compliance. Nonlinearity has to be requested, with two caveats: the
	reference works were chosen for their nonlinearity, and both free tasks
	list the events chronologically, so a preference for chronology cannot be
	told apart from a preference for the input order.

	\vspace{-2mm}
	\subsubsection{Asked for a shape, they comply literally}
	\label{sec:cued}

	Asked for a named shape (\Task{1}), models land in the human band
	\PoolBandHitPct\% of the time (\BandHitLoPct--\BandHitHiPct\% by model,
	led by \BestBandModel{}), against \RandomBandHitPct\% for a random
	permutation (\S\ref{sec:grid-notes}). Format is not the bottleneck
	(\ValidRate\% valid permutations, \PoolContigPct\% scene contiguity). What
	fails is structural choice, unevenly across shapes
	(Table~\ref{tab:category}, \S\ref{sec:tables}).

	Two shapes miss the band almost entirely (Figure~\ref{fig:plane},
	\S\ref{sec:tables}).
	\textsc{retrograde} scores \RetroBandPct\% and \textsc{chronicle}
	\ChronBandPct\%, because \RetroExactRev\% of \textsc{retrograde} outputs
	are the exact reversal ($\tau = -1$) and \ChronExactId\% of
	\textsc{chronicle} outputs the exact identity. The ten published works of
	each shape stop short of those ends ($\tau \in [\RetroTauLo, \RetroTauHi]$
	and $[\ChronTauLo, \ChronTauHi]$). Chronicles contain simultaneous events,
	and reverse chronology still plays forward inside each scene: published
	\textsc{retrograde} tellings reverse \SceneHumanBetween\% of the event pairs
	that span two scenes and \SceneHumanWithin\% of the pairs within one, where
	cued models reverse \SceneCuedBetween\% and \SceneCuedWithin\%
	(\S\ref{sec:scenes}). The prompts mark neither scenes nor simultaneity and
	state both shapes as rules over event ranks (\S\ref{sec:prompts}), each
	with a unique solution for a given event set (Prop.~\ref{prop:rule}). We
	call the result \emph{over-compliance} and read it as a specification
	effect. Models return the literal solution of the rule they were given,
	and whether they find the scene-level shape when scenes are marked is
	untested. Without the two shapes, pooled band hit is \BandHitNoRulePct\%.
	Two baselines that use no model support this reading. Reversing the
	scenes and playing each forward, the practice the band encodes, hits the
	\textsc{retrograde} band for \BaseSceneRetroHit{} of ten works against the
	models' \RetroBandPct\%, while the identity, which the \textsc{chronicle}
	prompt explicitly accepts, hits the \textsc{chronicle} band in
	\BaseIdChronHit{} of ten (\S\ref{sec:grid-notes}). On these two shapes,
	band hit measures the gap between the rule as stated and the practice
	recorded.

	Judgment-dependent shapes (\textit{i.e.}, \textsc{zigzag}, \textsc{prolepsis}, and \textsc{analepsis}) reach the human band \JudgedBandLo--\JudgedBandHi\% of the time. Because the rectangle is blind to alternation, \textsc{zigzag} yields the highest band hits but the lowest gold-telling agreement (\ZigzagAgree{}, \S\ref{sec:grid-notes}). By contrast, \textsc{achrony} fails through excessive scrambling ($\tau = \AchronyModelTau$ \textit{vs.} human $\AchronyHumanTau$).

	\vspace{-2mm}
	\subsection{RQ2. Reading a telling}
	\label{sec:reading}

	\subsubsection{Reconstruction without the story}
	\label{sec:floor}

	\Task{3} yields a pooled $\tau_{\mathrm{rec}} = \PoolTauRec$, with per-model
	scores ranging from \TauRecLo{} to \TauRecHi. The reconstruction controls
	(\S\ref{sec:tasks}) show how much of this score persists without the telling.
	\FigFloor
	\paragraph{Deleting the story.} Without the telling, three strong solvers
	reach $\tau = \DedTauFloor$,
	against \DedTauTold{} with the full telling (Figure~\ref{fig:floor}),
	so the telling is worth $\TellingWorth$
	$\tau$. Only \NBeatFloor{} of \NModels{} models clear this strong-solver
	floor (we ran no per-model floor), and per work the two conditions
	correlate at $r = \FloorCorr$. \paragraph{Removing every date.} The rewrite
	cut anchors such as ``Years earlier'' from \AnchorsBefore\% of events to
	\AnchorsAfter\% but moved the
	floor by only \DedDelta{} $\tau$ (Table~\ref{tab:dedate}), so the floor
	does not come from dates. Events can imply their own order (\textit{e.g.}, a man is shot
	before he bleeds to death), but these works are also recognizable, and
	de-dating leaves prior knowledge intact, so entailment and memory are not
	yet separated. After reading a generated story in full, readers ordered
	its events at only $\tau = \HChronTau$, where the floor on the same
	\HNStories{} works is \HFloorFourteen{} (\S\ref{sec:humanfloor}); the inputs
	differ, but entailment that a careful reader could also use does not
	explain that gap.

	The floor accounts for \DedTauShare\% of the $\tau$ score but only
	\DedExactShare\% of exact reconstructions. In pairwise accuracy it orders
	\FloorPairAcc\% of pairs correctly against \ToldPairAcc\% with the telling,
	so $\tau_{\mathrm{rec}}$ should not be read alone. Exact match is more
	revealing: the telling lifts it from \DedExactFloorPct\% to
	\DedExactToldPct\%, although a single misplaced event counts as total
	failure. No model reconstructs a \textsc{chronicle} exactly
	(\ChronRecExactPct\%), because near-chronological tellings deviate mostly
	through simultaneous events, which a total order scores as error.

	\vspace{-2mm}
	\subsubsection{Naming the shape}
	\label{sec:negatives}

	Classification (\Task{4}) reaches pooled accuracy \PoolClassifyPct\% and best
	\BestClassAccPct\% (\BestClassModel{}), against a seven-way random
	baseline of \ChanceClassifyPct\% and a majority-class baseline of
	\MajorityClassifyPct\%. Accuracy, band hit and $\tau$ do not share a
	difficulty scale, so we do not rank the tasks. The errors are systematic. Of the six other classes, \SinkN{} send
	\SinkLo--\SinkHi\% of their items to \textsc{analepsis}, so ``a dominant
	timeline with flashbacks'' is the models' fallback. \textsc{achrony} is
	recognized \ConfAchronySelfPct\% of the time, partly because the
	definition the models saw describes unrecoverable chronology, which the
	corpus excludes (\S\ref{sec:confusion} and \S\ref{sec:prompts}).
	Across models, naming a shape predicts both other abilities
    ($\rho = \RhoClassBand$ with \Task{1} band hit, $\rho = \RhoClassRec$ with \Task{3}
    $\tau_{\mathrm{rec}}$) while those abilities barely correlate
    ($\rho = \RhoCuedTauRec$): writing and reading are separate skills.

	\S\ref{sec:grid-notes} qualifies the grid. A small effort ablation
	cannot rule out an effect of reasoning effort, and the three GPT-5.6
	models refuse \RefusalLo--\RefusalHi\% of attempts at the dark-literary
	end of the corpus.

	\vspace{-2mm}
	\subsection{RQ3. Does a story graph help?}
	\label{sec:graph}

\begin{table}[t]
		\centering\footnotesize
		\setlength{\tabcolsep}{3pt}
		\fitcol{\begin{tabular}{lcccc}
\toprule
\textbf{Condition}     &     $\tau$     &     \textbf{band hit}     &     \textbf{agree}     &     \textbf{PDS} \\
\midrule
\GraphCondition{0} no graph & 0.362 & 0.411 & 0.554 & 3.85 \\
\GraphCondition{1} \textbf{full graph} & 0.415 & 0.518 & 0.554 & 4.10 \\
\GraphCondition{2} told order \textit{(oracle)} & 0.333 & 0.982 & 0.998 & 3.94 \\
\midrule
\quad $-$\,\texttt{CAUSES}/\texttt{FORESHADOWS} & 0.417 & 0.429 & 0.557 & 4.18 \\
\quad $-$\,\texttt{pre}/\texttt{post} & 0.440 & 0.464 & 0.557 & 4.14 \\
\quad $-$\,\texttt{learns} & 0.399 & 0.464 & 0.561 & 4.05 \\
\quad $-$\,questions/\texttt{REVEALS} & 0.383 & 0.464 & 0.502 & 3.99 \\
\quad $-$\,participants/location & 0.409 & 0.500 & 0.526 & 4.04 \\
\bottomrule
\end{tabular}
}
		\caption{The scaffolding study. Ablations below the rule each remove one
			graph layer from \GraphCondition{1}; agreement with the gold telling
			confirms that \GraphCondition{2} follows the order it is given. Rows use
			the \GraphCommonModels{} models that ran every condition, the only ones
			that allow a paired comparison; pooling all \GraphNModels{} models over
			the three conditions they share gives the same ordering at smaller
			margins.}
		\label{tab:grapharm}
		\vspace{-4mm}
	\end{table}

	\vspace{-1mm}
	\paragraph{The graph beats no graph, under the judge.}
	\GraphCondition{1} beats \GraphCondition{0} $\GOneVsGZeroW$--$\GOneVsGZeroL$ with
	$\GOneVsGZeroT$ ties ($p = \GOneVsGZeroP$, Figure~\ref{fig:grapharm} in
	\S\ref{sec:tables}), PDS
	rises on \PDSCompsUp{} of \PDSCompsN{} components (\GZeroPDS{} to
	\GOnePDS{}), and band hit rises from \GZeroBandPct\% to \GOneBandPct\%
	(Table~\ref{tab:grapharm}).
	The structural gain is thin. It comes almost entirely from one model
	(below), and agreement with the gold telling does not move (\GZeroAgree{}
	to \GOneAgree{}), so models do not use the order the \texttt{REVEALS}
	edges fix. The preference gain disappears after length adjustment
	(\S\ref{sec:humanlength}). \GraphCondition{1} texts are longer; readers pick the
	longer text \HGNoneLongerPct\% of the time, and the judge saw the same
	unequal texts.

	\vspace{-1mm}
	\paragraph{Against the told order, the judge sees no difference and
	readers do.}
	The judge splits \GraphCondition{1} \textit{vs.}\ \GraphCondition{2} at
	$\GOneVsGTwoW$--$\GOneVsGTwoL$, although \GraphCondition{2}, handed the
	answer, leads on structure (\GTwoBandPct\% band hit, McNemar
	$p \BandMcP$), while readers prefer \GraphCondition{1} in
	\HGOrderBetterPct\% of decisive verdicts (\S\ref{sec:humanpref}). The
	split is not a matter of sample size: at the reader rate, an outcome this
	even across $\HJudgeGOrderN$ judgments has probability
	$\HJudgeGOrderTail$ (\HJudgeGOrderTailLo{} at the floor of the readers'
	interval), taking that rate as known; a formal test would pair item-level
	decisions. The prompts differ beyond the graph (\S\ref{sec:humanpref}),
	but here the judge does not track readers.

	\vspace{-1mm}
	\paragraph{No single layer, and almost entirely one model.}
	Under the judge, all five leave-one-layer-out ablations match the full
	graph on preference ($p$ from \AblPrefPLo{} to \AblPrefPHi{}), and band
	hit falls for every removal without approaching significance (smallest
	paired McNemar $p = \AblBandPLo$, $n = \AblBandN$). Readers prefer every
	ablation to the full graph (\S\ref{sec:humanabl}). A graph that loses to
	each of its one-layer reductions yet ties no graph at equal length is consistent
	with excess prompt length or distraction, but these comparisons do not
	isolate either explanation. \GraphBestModel{} supplies almost all of the structural gain
	(\GraphBestDelta{} band hit, \GraphWorstModel{} \GraphWorstDelta{}, the
	other \GraphInertModels{} none) and most of the readers' ablation
	preference. We claim only that the graph raises structural compliance
	for one generator of four.

	\vspace{-2mm}
	\subsection{RQ4. The cost of control}
	\label{sec:worth}

	\vspace{-1mm}
	\paragraph{Cued tellings lose to the model's own free telling.}
	\label{sec:prosepref}
	In the paired cued-versus-free comparison (\S\ref{sec:autoeval}), free
	tellings win \ProseFreeWins{} to \ProseCuedWins{}
	with \ProseTies{} ties ($n = \ProsePairs$, $p \approx \ProsePairP$), all
	\ProseLoseAll{} models lose, and PDS agrees more mildly (\ProseFreePDSs{}
	\textit{vs.} \ProseCuedPDS{}). The judge (Gemini~3.1~Pro) shares a family with
	one generator, a documented source of bias \citep{zheng2023judging}, but
	readers confirm the result at a closely matched rate
	(Figure~\ref{fig:teaser-b}; \S\ref{sec:humanpref}). The cost is not only nonlinearity. The cued story
	also loses when the requested shape is \textsc{chronicle}
	(\HByCatCuedChron{}, $n = \HByCatCuedChronN$, uncorrected $p =
	\HByCatCuedChronP$), which asks for little reordering, and draws level on
	\textsc{retrograde} (\HByCatCuedRetro{}), the most nonlinear request
	(Table~\ref{tab:humanbycat}), so part of the penalty is the instruction
	itself.

	\vspace{-1mm}
	\paragraph{The evidence for premature disclosure is suggestive.}
	\label{sec:reveal}
	For each work we take the \texttt{REVEALS} edge whose explanation the
	gold telling places latest, and call its disclosure early if fewer than
	three events precede the explanation's first mention in a story. Two works
	inflate the early counts: one designates no reveal yet was scored early,
	and a chronicle's selected explanation is its third chronological event.
	Without them, disclosure is early in \EarlyCleanGZeroPct\% of stories
	without a graph, \EarlyCleanGOnePct\% with one and none when the model is
	handed the told order ($p = \EarlyCleanP$). That repair is partly by
	construction, since the reveal is defined by the order
	\GraphCondition{2} is handed; what remains is that \GraphCondition{1},
	which marks what is withheld, does no better than \GraphCondition{0}
	(\S\ref{sec:rcg}).

	\vspace{-2mm}
	\subsection{RQ5. What readers say}
	\label{sec:human}

	\begin{table*}[t]
		\centering\footnotesize
		\setlength{\tabcolsep}{4.5pt}
		\renewcommand{\arraystretch}{0.95}
		\fitfull{\begin{tabular}{@{}l rcc rcc rcc c}
\toprule
 &  \multicolumn{3}{c}{\textbf{better story}}
 &  \multicolumn{3}{c}{\textbf{harder to follow}}
 &  \multicolumn{3}{c}{\textbf{better ending}}
 &  \textbf{adj.} \\
\cmidrule(lr){2-4}\cmidrule(lr){5-7}\cmidrule(lr){8-10}\cmidrule(l){11-11}
\textbf{Comparison}  &  \textbf{split}  &  \textbf{rate}  &  $p$  &  \textbf{split}  &  \textbf{rate}  &  $p$  &  \textbf{split}  &  \textbf{rate}  &  $p$  &  \textbf{rate} \\
\midrule
instructed \textit{vs.}\ left alone & 64--200--30 & \negcell{0.242} & $<0.0001$ & 197--47--50 & \negcell{0.807} & $<0.0001$ & 114--134--46 & 0.460 & $0.23$ & \poscell{0.246} \\
full graph \textit{vs.}\ told order & 181--70--43 & \keycell{0.721} & $<0.0001$ & 115--108--71 & 0.516 & $0.69$ & 141--106--47 & 0.571 & $0.0303$ & \poscell{0.707} \\
full graph \textit{vs.}\ event list & 154--105--35 & \poscell{0.595} & $0.0028$ & 123--102--69 & 0.547 & $0.18$ & 133--103--58 & 0.564 & $0.0588$ & \warncell{0.518} \\
full graph \textit{vs.}\ one layer removed & 79--180--35 & \negcell{0.305} & $<0.0001$ & 140--86--68 & \negcell{0.619} & $0.0004$ & 110--119--65 & 0.480 & $0.60$ & \negcell{0.304} \\
\bottomrule
\end{tabular}
}
		\caption{Reader verdicts. Each pair is one model writing one story
			twice. \hd{split}: wins--losses--ties for the first-named condition.
			\hd{rate}: share of decisive judgments. \hd{adj.}: rate net of word count
			\citep{dubois2024length}. \cKey{} headline, \cPos{} survives every check,
			\cNeg{} against hypothesis, \cWarn{} fails length adjustment.}
		\label{tab:humanpref}
		\vspace{-4mm}
	\end{table*}

	The reader study tests the automated preference verdicts with the protocol
	of \S\ref{sec:humanproto} (Table~\ref{tab:humanpref}), and
	Figure~\ref{fig:teaser-b} sets the unadjusted reader rates beside the
	judge's. All three analyses agree on the headline results, and every
	\emph{better story} claim below clears the Bonferroni threshold of
	$0.0042$. Readers agree with each other on that question but not on the
	other two (Krippendorff's $\alpha = \HAlphaBetter{}$ against \HAlphaConf{}
	and \HAlphaEnd{}), and on the unannounced repeats they reproduce their own
	decisive verdict \HRepeatBetterDecPct\% of the time (\HRepeatBetterPct\%
	over all repeats, \S\ref{sec:human-extra}).

	\begin{figure}[t]
		\centering
		\input{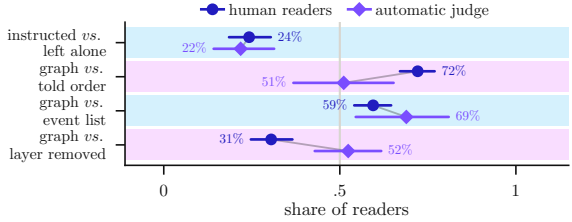}
		\caption{Unadjusted preference rates. Readers favor graph-guided stories
		over those given a telling order or an event list, and the judge's ties
		on two comparisons miss clear reader preferences. ``Layer removed''
		pools the five one-layer ablations.}
        \vspace{-4mm}
		\label{fig:teaser-b}
	\end{figure}

	\vspace{-1mm}
	\subsubsection{The three preference results}
	\label{sec:humanpref}

	\vspace{-1mm}
	\paragraph{Instructed tellings lose, and length does not explain it.}
	Against the same model's free telling of the same events, readers chose
	the free version \HCuedBetterInvPct\% of the time
	(\HCuedBetterW--\HCuedBetterL{}, \HCuedBetterT{} ties, reader-clustered \HCuedBetterCI{}), called the
	instructed version harder to follow \HCuedConfPct\% of the time, and split evenly on endings. \HCuedReadersAgainst{} of
	\HCuedReadersN{} readers and all \HCuedModelsN{} generators favor the free
	telling. Length favors the instructed texts, which are longer
	(\HCuedWordsHyp{} against \HCuedWordsOther{} words), and the longer text
	won \HCuedLongerPct\% of the time here. Among length-matched pairs,
	the instructed telling wins \HCuedLMPct\%.
	Of \HNCommentsCued{} written comments, \HNCommentsCuedOrder{} name order
	or difficulty following.

	\vspace{-1mm}
	\paragraph{The graph beats a supplied told order, which the judge
	missed.}
	This is the largest gap between readers and the judge. Readers chose the
	full graph's story over the told order's \HGOrderBetterPct\% of the time
	(\HGOrderBetterW--\HGOrderBetterL{}, $p \HGOrderBetterPrel$,
	reader-clustered \HGOrderBetterCI{}), with no cost in how hard it was to
	follow ($p \HGOrderConfPrel$) and a smaller advantage on endings
	(\HGOrderEndR{}). \HGOrderReadersFor{} of \HGOrderReadersN{} readers,
	\HGOrderStoriesFor{} of \HGOrderStoriesN{} stories and \HGOrderModelsFor{}
	of \HGOrderModelsN{} generators agree (the fourth is even), the
	length-adjusted rate is \HGOrderAdjPct\%, and every reader subset
	reproduces it (Table~\ref{tab:humansubsets}). But the two conditions
	differ in more than the graph, above all in ordering freedom. \GraphCondition{1}
	fixes \PinnedPct\% of told-order pairs and \GraphCondition{2} all of them, so
	they realize different orders (agreement with the gold telling
	\GOneAgree{} against \GTwoAgree{}). Instruction, content and prompt length
	differ too, and the graph also wins on \textsc{chronicle}
	(\HByCatOrderChron{}, $n = \HByCatOrderChronN$), where both are nearly
	chronological. We read this as a preference for graph-scaffolded over
	order-constrained writing, not as readers rewarding the graph's reasons
	for an order.

	\vspace{-1mm}
	\paragraph{The graph over a bare event list does not survive length.}
	\label{sec:humanlength}
    \textit{Prima facie} the graph wins here too (\HGNoneBetterPct\%, $p
	\HGNoneBetterPrel$), but this is the one comparison where length and
	condition are tangled. The graph's texts average \HGNoneWordsHyp{} words
	\textit{vs.} \HGNoneWordsOther{}, readers chose the longer text
	\HGNoneLongerPct\% of the time, and the word-count coefficient is the
	largest in the study ($\HGNoneLenBeta$ per hundred words, $p
	\HGNoneLenBetaPrel$). Adjusted for length, the rate is \HGNoneAdjPct\% ($p
	\HGNoneAdjPrel$), and the length-matched subset agrees (\HGNoneLMPct\%).
	It is the only row that moves (Figure~\ref{fig:humanjudge},
	\S\ref{sec:tables}). Since the graph itself lengthens the texts,
	length here is an outcome, not a prior covariate. The raw
	\HGNoneBetterPct\% records observed preference and the adjusted
	\HGNoneAdjPct\% estimates preference at equal word count; they answer
	different questions, and the second does not correct the first. The
	graph makes models write more and readers prefer more, and a run that did
	not hold length fixed cannot say whether it also makes them write better.
	The close-reading PDS gap (\HPdsGOne{} \textit{vs.} \HPdsGZero{}, $d =
	\HPdsGOneVsGZeroD$) carries the same caveat.

	\vspace{-2mm}
	\subsubsection{Readers prefer the graph, bar a layer}
	\label{sec:humanabl}

	No single layer stands out under either instrument, but pooled over
	layers, readers preferred the reduced graph's story \HGAblBetterInvPct\% of the time
	($p \HGAblBetterPrel$) and called the full graph's story harder to follow
	\HGAblConfPct\% of the time ($p \HGAblConfPrel$). Length does not explain
	this, since the conditions differ by a handful of words and the full graph
	still wins only \HGAblLMPct\% when length is matched. Every layer points the
	same way, the full graph winning between \HAblEpistemicR{} and
	\HAblQuestionsR{} of decisive judgments against each reduction
	(Table~\ref{tab:humanabl}). The main caveat is that this is almost entirely one generator.
	\GraphBestModel{}'s full graph wins \HAblClaudeSonnetW{} of
	\HAblClaudeSonnetN{} decisive judgments against its own reduced-graph
	output (\HAblClaudeSonnetPct\%, $p \HAblClaudeSonnetPrel$), losing about
	four in five, while \GraphWorstModel{} is at chance
	(\HAblGeminiFlashPct\%). The per-layer tests are also underpowered,
	so no layer is singled out and none is shown to be inert.

	\vspace{-2mm}
	\subsubsection{Does the judge agree with readers?}
	\label{sec:humanjudge}

	Pooled over the four comparisons, the judge and the readers' majority
	agree on \HJHPoolRaw{} of items where both are decisive, $\kappa =
	\HJHPoolKappa$ ($\HJHPoolCI$), and the judge agrees with individual
	readers on \HJHIndAgree{} of \HJHIndN{} judgments. Comparison by
	comparison (Figure~\ref{fig:teaser-b}; Table~\ref{tab:judgehuman}), the judge's rate lands within
	\HGapCuedPts{} points of the readers' on instructed versus free, though
	item-level agreement there is weak ($\kappa = \HJHCuedKappa$). Item-level
	agreement is highest on graph versus event list ($\kappa =
	\HJHGNoneKappa$, $\HJHGNoneCI$), the comparison with the smallest reader
	effect, where both instruments may be answering the same length gap
	(\HGNoneWordsHyp{} \textit{vs.} \HGNoneWordsOther{} words). On graph versus
	told order the judge sits at \HJudgeGOrderR{} and readers at
	\HGOrderBetterR{} ($\kappa = \HJHGOrderKappa$), and on the layer removals
	agreement is indistinguishable from none, $\kappa = \HJHGAblKappa$
	($\HJHGAblCI$). Item-level agreement therefore does not track effect size,
	and four comparisons cannot separate effect size from comparison type.
	What holds is narrower: the judge matched readers on the one large effect,
	and readers overturned both of its ties, so its ties are not evidence of
	no difference. This is one judge and one cohort, and the second-model
	check of \S\ref{sec:alignvalidation} tests alignment, not preference. The
	two sides also abstain differently: a reader tie means ``about the same'',
	a judge tie that its two passes disagreed.

	The depth scale fares better. Judge and reader PDS on the same
	\HPdsCorrN{} texts correlate at $r = \HPdsCorr$, but the judge rates
	uniformly high ($\HPdsGap$ above readers, \S\ref{sec:humanclose}),
	so PDS differences can be reported but a PDS \emph{level} from one judge
	cannot.

	\vspace{-2mm}
	\subsubsection{Our metrics against readers}
	\label{sec:bandvalid}

	\paragraph{$\tau$ predicts what readers find hard.} Across \HConfTauN{}
	decisive judgments, the text a reader called harder to follow is the less
	chronological of the pair by $\Delta\tau = \HConfTauDelta$ ($p \HConfTauPrel$,
	item-level $r = \HConfTauR$), so $\tau$ tracks a palpable reading cost readers report,
	though agreement on that question is low ($\alpha = \HAlphaConf$) and the
	association runs partly through condition (below).

	\vspace{-1mm}
	\paragraph{Band hit does not predict what readers prefer.} In the
	\HBandSplitN{} decisive judgments where exactly one text falls inside the
	human band, readers chose that text \HBandSplitPct\% of the time ($p
	\HBandSplitPrel$). The association stays negative in a model that includes
	condition-pair fixed effects and the word-count and $\tau$ differences
	(band hit $\beta = \HAdjBetaBand$, $p \HAdjPBandrel$, while $\tau$ enters
	positively at $\beta = \HAdjBetaTau$), and the close readings agree
	(\S\ref{sec:humanclose}). Both associations partly restate the
	condition effect, since the in-band text is almost always the instructed
	one. What band hit rewards is a simulacrum of published practice: the coordinates of a real telling without whatever made a reader want it. Depth and
	nonlinearity are in any case unrelated across these outputs
	(\S\ref{sec:humanclose}).

	\vspace{-2mm}
	\section{Discussion}
\label{sec:discussion}

Compliance with an ordering rule, resemblance to published tellings and
reader preference are distinct desiderata, and our results pull them
apart: over-compliance satisfies the rule and misses the band, band hit
tracks resemblance yet runs against preference, and readers prefer the
telling that leaves chronology alone. An evaluation that reports only one
of the three can certify a failure as a success. Part of the gap lies in
the specification: published reverse chronology is a rule applied to
scenes, our prompts stated it on events, and whether models find the
scene-level shape once scenes are marked is the most germane open
question. A scaffold that encodes discourse relations likewise encodes
part of the answer, so its value can be read only against a condition that
supplies the order directly, at matched output length.
\S\ref{sec:recommendations} distills these lessons into six practices for a
more stringent protocol, none of which needs new annotation.

\section{Conclusion}
	\label{sec:conclusion}

	Narrative order is a variable that story-generation research has largely
	left unmeasured. Recording it for \NWorks{} published works turns
	nonlinearity into a permutation that can be specified and scored
	(Table~\ref{tab:rqsummary}, \S\ref{sec:tables}).
	Left to choose, all \NAboveHuman{} models narrate more chronologically
	than published tellings chosen for their nonlinearity. Given a shape
	stated as a rule on events, they apply it literally and miss the band,
	whereas a scene-level reversal computed from the corpus alone lands inside
	it.

	Our own instruments needed as much scrutiny as the models. A task that
	resembles reading comprehension keeps \DedTauShare\% of its score with
	the telling removed, resemblance to published tellings does not predict
	what readers prefer, and the judge we tested called a tie where readers
	split \HGOrderBetterPct{} to \HGOrderBetterInvPct; a few hundred reader
	judgments on the same items are the cheapest litmus test of such a tie.
	Until our models apprehend the subtler art by which human tellers
	make time serve meaning, they may remain meticulous accountants of
	chronology, yet strangers to the spirit of story.

	\section*{Limitations}

	\vspace{-1mm}
	\paragraph{One judge still does all the extraction.}
	The human study replaces the judge for preference, not for measurement. Every
	structural score on generated prose, and the premature-disclosure rate, still
	depends on one model (Gemini~3.1~Pro) aligning prose back to the event list. The
	gates in \S\ref{sec:alignvalidation} bound that alignment error, and the reader data do
	not touch it. The judge also still shares a model family with one generator in
	each prose comparison, which we flag rather than eliminate, since no single
	judge can
	be cross-family to every model on the roster. Ordering, reconstruction and
	classification are scored deterministically from model outputs and need no
	judge at all.

	\vspace{-1mm}
	\paragraph{The reader study cannot reach past fourteen stories.}
	Every preference and close-reading result comes from the \HNStories{} works for
	which generated prose exists, because every prose-generating run used that
	subset to control cost. The shape-comprehension items reach all
	\HNShapeStories{} works, but they involve no generated text. Widening the
	preference evidence means a new generation run, not more readers, and the same
	ceiling is what forces a reader to meet the same story again in a later sitting
	(\S\ref{sec:humanchecks} shows that this does not move the results, but it is a
	constraint of the design rather than a feature of it). The generation
	side is narrower still: \GraphNModels{} generators for the scaffolding study,
	and \GraphCommonModels{} generators with one replicate for the ablations.

	\vspace{-1mm}
	\paragraph{The length confound is diagnosed, not removed.}
	We can say that the graph-versus-event-list preference does not survive
	adjusting for word count, and we cannot say what that comparison would look like
	with length held fixed at generation time. A regression adjustment answers a
	counterfactual question about the texts we have
	\citep{dubois2024length}; a length-matched generation run would answer the
	question directly, and it is the single cheapest experiment left in this paper.
	The same caveat attaches to the close-reading PDS gap, since the graph's texts
	were also read for longer. Our own Recommendation~5 asks for length matching, and here we only
	adjusted for length.

	\vspace{-1mm}
	\paragraph{The reader cohort is not a probability sample.}
	\HNReaders{} readers were recruited from university and secondary-school
	settings plus a group of general readers, so the cohort is more literate and
	more interested in fiction than a general population.
	\HReadsMonthly{} finish at least a book a month and \HWriteSome{} plus
	\HWriteDegree{} have some writing training. Table~\ref{tab:humansubsets} shows
	the results hold within every subset we can cut, including the least
	practiced
	readers, but a cohort drawn differently could differ, and we make no claim about
	readers in general. Agreement on \emph{better ending} is low enough
	($\alpha = \HAlphaEnd$) that the endings rows should be read as suggestive.

	\vspace{-1mm}
	\paragraph{Input order is a confound for the free tasks.}
	\Task{2} and \ProseTask{b} present the events in chronological order, and say so, so a model
	that merely keeps its input order scores as chronological. The ordering outputs
	show the default is not pure copying, since the pooled $\tau = \PoolTauFree$
	sits well short of the identity. In prose the two readings coincide, and
	separating a preference for chronology from a preference for the given order
	needs a shuffled-input control we have not run. The reference is also a selected sample. Of the \NWorks{} works,
	\NNonlinearWorks{} were chosen because they are widely recognized as
	nonlinear, while most published narrative is chronological
	(\S\ref{sec:shapes}), so against this corpus ``more chronological
	than the published telling'' is close to guaranteed. The informative
	comparison is a human storyteller given the same event list and prompt.

	\vspace{-1mm}
	\paragraph{Band hit measures resemblance, and resemblance is not quality.}
	It asks whether a telling looks like real works of its class. \S\ref{sec:bandvalid}
	measures it against reader preference and finds the association negative, so the
	metric should be read as a structural diagnostic and never as a quality proxy.
	That association also partly restates the condition effect, since the
	in-band text in an instructed-against-free pair is almost always the
	instructed one. And because the rectangle sees position in the
	$(\tau,\mathrm{LIS})$ plane rather than alternation or grouping, a telling
	can land in it without having the shape (\S\ref{sec:grid-notes}).
	What we cannot separate is how much of that negative sign is about band hit and
	how much is about these particular outputs. Readers prefer chronology in
	model-written prose, and hitting the band requires leaving it; a corpus of
	generated tellings that actually earned their disorder might reverse the sign.
	Testing that needs generators better at this task than the ones we have.

	\vspace{-1mm}
	\paragraph{The layer-removal result rests on two generators that disagree.}
	The ablations run on \AblBandN{} paired band-hit observations,
	\GraphCommonModels{} models and one replicate, and the reader study adds
	only
	\HGAblBetterN{} decisive judgments across five layers. At \HAblLayerN{} judgments
	per layer, 80\% power reaches only
	$|\text{rate} - 0.5| \geq \HAblMDE$, and the two generators that ran them go
	opposite ways ($\chi^2 = \HAblHetChi$, $p \HAblHetPrel$). Readers and the judge
	both fail to find a per-layer effect, which is now two instruments agreeing on a
	no-difference verdict rather than one. Neither amounts to an equivalence test, which would need a margin fixed in
	advance and enough precision to rule out everything beyond it. With
	\HAblLayerN{} judgments per layer, the smallest detectable departure is
	already \HAblMDE{}. That verdict, from a
	set of runs this size and this
	one-model-dominated, is not evidence the layers are inert. The prompt-length
	confound is also unresolved: \GraphCondition{1} is the longest prompt, the ablations
	are shorter, readers find \GraphCondition{1}'s output the harder one to follow, and the
	length-matched \emph{prompt} control that would separate ``the layer was
	useless'' from ``the prompt was diluted'' has not been run. We also lack an
	equal-compute agentic baseline. \GraphCondition{0} is a single call, so part of the
	\GraphCondition{1} structural gain may be ``more structure in the prompt'' rather than
	``this structure.'' Taken together, the reader data do not show that the graph helps. Its edge
	over no graph is a length effect, and its edge over the told order is
	confounded (\S\ref{sec:humanpref}). It also loses to every one of its
	one-layer reductions while tying no graph at equal length, and a prompt
	that is too long or distracting explains that non-monotone pattern more
	simply than any account of what the layers do.

	\vspace{-1mm}
	\paragraph{Scale and coverage.}
	The ordering and classification tasks run on all \NWorks{} works. The prose and
	scaffolding studies run on \ProseNStories{}, two per structural category. All
	decoding used low reasoning effort, checked only by a small ablation
	(\EffortNModels{} models on the same \ProseNStories{} works, \Task{1} and \Task{3}) that
	cannot rule out an effect of test-time compute
	(\S\ref{sec:grid-notes}).

	\vspace{-1mm}
	\paragraph{Contamination is reduced, not eliminated.}
	Anonymization removes names, places and quotations but cannot hide a premise, and
	the works added to fill the thin structural classes are the leakier ones,
	{precisely because a class was often thin for want of works whose structure
	does not by itself identify them.} The works are recognizable, and the example in Figure~\ref{fig:overview}
	is a well-known film even after anonymization. This matters most for the
	no-story floor of \S\ref{sec:floor}, which we read as entailment but which
	prior knowledge of a work would also produce, since de-dating removes
	explicit time anchors but not memory. Readers who had read the story
	reached only $\tau = \HChronTau$ on stories where the model floor is
	\HFloorFourteen{} (\S\ref{sec:humanfloor}), which a pure entailment
	account does not predict.

	\vspace{-1mm}
	\paragraph{English only.}
	Every event sentence, every authored graph string and the anonymized twin are
	written in English. The works themselves are not all Anglophone, and the corpus
	includes works first told in Korean, Japanese, Russian, German, French and
	Spanish. Each is annotated from an English-language telling or translation,
	however, and whether the taxonomy, the bands or the results transfer to other
	languages' narrative traditions is untested.

	\vspace{-1mm}
	\paragraph{No commissioned human writing.}
	Our human reference is the annotation of a published work, which is the right
	ceiling for the ordering, reconstruction and classification tasks but is not a
	human writer given the same brief as the models. We therefore make no
	human-level claim, and the paper measures model capability against a corpus
	reference. The reader study inherits this limit in its sharpest form: every text
	a reader judged was model-generated, so nothing in \S\ref{sec:human} speaks to
	the gap between these tellings and published fiction. When readers prefer a
	model's chronological telling to its instructed nonlinear one, that is a
	statement about two model outputs and not about nonlinearity.

	\vspace{-1mm}
	\paragraph{Annotation is interpretive, and unreplicated.}
	No work was annotated independently by a second person, so we report no
	inter-annotator agreement and treat measuring it as required future work.
	The band construction makes this worse. Each band is the rectangle spanned
	by the extreme $\tau$ and LIS of 10--22 single-annotator tellings, so a
	single segmentation decision can move a band edge. Leaving each work out
	of its own band moves pooled band hit between \LowoBandLo{} and
	\LowoBandHi{} (\LowoMovers{} works move it by more than 0.01), and a
	held-out work falls inside its peers' band \HeldOutInBandPct\% of the time
	(\S\ref{sec:grid-notes}). Some labels are also a stretch. Our ten
	chronicles span $\tau$ from \ChronTauLo{} to \ChronTauHi{}, and no work in
	the corpus is told exactly in order.
	Three of the seven structural labels turn on what scenes are about, and works
	with genuinely ambiguous chronology are annotated with the most defensible
	reading and the ambiguity recorded, so reasonable annotators would differ on
	some labels. We release the reasons alongside the labels so disagreements are
	locatable, and the machine checks of \S\ref{sec:validation} catch
	inconsistency but not misreading.

	\vspace{-1mm}
	\paragraph{Over-compliance may be a property of the prompt.}
	\Task{1} lists events one per line, marks no scene boundaries or simultaneous
	sets, and defines \textsc{retrograde} and \textsc{chronicle} as rules on
	event ranks, so the exact reversal and the identity are the literal
	solutions (Prop.~\ref{prop:rule}). The model never receives the
	scene-level practice that defines the human band. The \textsc{achrony}
	definition has a related problem. It describes unrecoverable chronology,
	which the corpus excludes, and reads as an annotation procedure rather
	than a description of a shape (\S\ref{sec:prompts}). Until the
	scene-marked rerun, these results measure the specification as much as the
	model.

	\vspace{-1mm}
	\paragraph{A permutation is not a discourse.}
	Our object is a bijection between a fixed event list and a told order, so
	every telling narrates every event exactly once. Published narrative does
	not always do this. It returns to an event, releases it in pieces, drops
	it, or tells it twice in versions that cannot both be true. None of these
	has a distinct place in a permutation, and the twelve to twenty beats we
	record summarize a work's discourse rather than reproduce it. Generated
	prose reaches the object through an aligner that must emit one position
	per event (\S\ref{sec:prompts}), so omissions and repetitions are
	projected onto a first mention, and the resulting error is unmeasured.
	What we study, and what the benchmark scores, is control over the order of
	selected plot beats.

	\vspace{-1mm}
	\paragraph{Simultaneity is scored as error.}
	The gold chronology is a total order, so ties between simultaneous events
	are broken arbitrarily. This affects the \textsc{chronicle} gold tellings,
	\Task{3} exact match (\ChronRecExactPct\% on \textsc{chronicle}) and the $\tau$
	of the chronicles themselves. We have not yet encoded the chronology as a
	weak order and scored it with Kendall's $\tau_b$ or equivalence-class
	exact match, as our own Recommendation~2 advises. The same choice affects
	the targets. The \textsc{chronicle} prompt accepts simultaneous events in
	either order, so exact recovery is scored against one arbitrary
	linearization of a partial order, which is why \ChronRecExactPct\% should
	not be read as a model result.

	\vspace{-1mm}
	\paragraph{The cost of control is not isolated.}
	Instructed tellings lose to free ones even when the requested shape is
	\textsc{chronicle} (Table~\ref{tab:humanbycat}), so the reader result
	mixes the cost of nonlinearity with the cost of the instruction itself
	(the definition, the added constraint, a longer prompt). A
	chronology-instructed control on all \ProseNStories{} works would separate
	the two, but we have not run it.

	\vspace{-1mm}
	\paragraph{The reader statistics do not model the crossed design.}
	Headline $p$-values are exact binomials over judgments, and the bootstrap
	intervals cluster on one factor at a time. The design is crossed (readers
	by items, nested in stories and generators), and with \HNStories{} stories
	and \GraphNModels{} generators the story- and generator-level variance
	matters most. We have not yet fit a mixed-effects model with crossed
	random intercepts. Three further problems follow. The \HNSittings{} hidden
	repeats are included in the rates we report, so \HNFirstPres{} of the
	\HNPresentations{} presentations are first presentations and some
	reader--item pairs are counted twice. The first-presentation analysis
	should be the primary one. The Bonferroni threshold corrects for twelve
	tests but not for correlated observations, which is the larger of the two
	problems. Finally, a \GraphCondition{1} text appears in more than one contrast,
	so the comparisons are not independent of each other either.

	\vspace{-1mm}
	\paragraph{Premature disclosure rests on few stories.}
	Once the work with no designated reveal is set aside, the contrast in
	\S\ref{sec:reveal} rests on 52 stories per condition, and half of its
	early disclosures come from one chronicle. The scorer also counted the
	missing reveal as early. We therefore report the result as suggestive
	(\S\ref{sec:rcg}).

	\section*{Ethics statement}

	The corpus is derived from published narrative works. We release hand-authored
	event summaries, structural annotations and graphs rather than source text, and
	the anonymized twin that models are evaluated on replaces all identifying
	surface detail. Anonymization addresses contamination, not rights: what we
	release is our own writing about published works, and nothing in it
	reconstitutes a copyrighted work or substitutes for reading one. The full
	release, comprising both corpora, the story graphs, all prompts, the stored
	model replies, the reader responses and the analysis code, is distributed
	under the MIT License.

	The refusal asymmetry in \S\ref{sec:negatives} deserves note as a fairness
	observation rather than a criticism of any safety policy. The three
	GPT-5.6 models, all from one lab,
	declined a consistent set of works at the dark-literary end of the corpus,
	including on tasks as innocuous as sorting events into chronological order. Those
	are disproportionately the works whose structure is most interesting, so a
	literary benchmark built on a single vendor will quietly under-sample them. We
	report the affected subset and repeat every comparison on the works all models
	answered.

	Annotation was carried out by the authors. The reader study of
	\S\ref{sec:human} involved \HNReaders{} adult participants recruited from
	university and secondary-school settings and from a group of general readers.
	Each sitting opened with a consent step stating what the study measured, how
	long it would take and that participation could be stopped at any point; no
	sitting proceeded without it. Only two background items were collected, reading
	habit and writing experience, both as coarse categories and both kept solely as
	covariates. No name, contact detail, demographic attribute or free-text
	identifier is stored. Participants are identified in the study data by an
	integer, and the free-text comments were read for content about the stories
	only. The material readers were shown is model-generated prose built from
	anonymized event summaries, and the works at the dark-literary end of the corpus
	carry the same content as their published sources; the consent step said so.

		\bibliography{custom}

\begin{thebibliography}{41}
\providecommand{\natexlab}[1]{#1}

\bibitem[{Artstein and Poesio(2008)}]{artstein-poesio-2008-survey}
Ron Artstein and Massimo Poesio. 2008.
\newblock \href {https://doi.org/10.1162/coli.07-034-R2} {Inter-coder agreement
  for computational linguistics}.
\newblock \emph{Computational Linguistics}, 34(4):555--596.

\bibitem[{Bae and Young(2008)}]{bae2008flashback}
Byung-Chull Bae and R.~Michael Young. 2008.
\newblock \href {https://doi.org/10.1007/978-3-540-89454-4_22} {A use of
  flashback and foreshadowing for surprise arousal in narrative using a
  plan-based approach}.
\newblock In \emph{Interactive Storytelling: First Joint International
  Conference on Interactive Digital Storytelling ({ICIDS} 2008)}, volume 5334
  of \emph{Lecture Notes in Computer Science}, pages 156--167. Springer.

\bibitem[{Chakrabarty et~al.(2024)Chakrabarty, Laban, Agarwal, Muresan, and
  Wu}]{chakrabarty2024art}
Tuhin Chakrabarty, Philippe Laban, Divyansh Agarwal, Smaranda Muresan, and
  Chien-Sheng Wu. 2024.
\newblock \href {https://doi.org/10.1145/3613904.3642731} {Art or artifice?
  {L}arge language models and the false promise of creativity}.
\newblock In \emph{Proceedings of the 2024 {CHI} Conference on Human Factors in
  Computing Systems}. Association for Computing Machinery.

\bibitem[{Chambers and Jurafsky(2008)}]{chambers-jurafsky-2008-unsupervised}
Nathanael Chambers and Dan Jurafsky. 2008.
\newblock Unsupervised learning of narrative event chains.
\newblock In \emph{Proceedings of ACL-08: HLT}, pages 789--797, Columbus, Ohio,
  USA. Association for Computational Linguistics.

\bibitem[{Cheong and Young(2008)}]{cheong2008suspense}
Yun-Gyung Cheong and R.~Michael Young. 2008.
\newblock \href {https://doi.org/10.1007/978-3-540-89454-4_21} {Narrative
  generation for suspense: Modeling and evaluation}.
\newblock In \emph{Interactive Storytelling: First Joint International
  Conference on Interactive Digital Storytelling ({ICIDS} 2008)}, volume 5334
  of \emph{Lecture Notes in Computer Science}, pages 144--155. Springer.

\bibitem[{Chhun et~al.(2022)Chhun, Colombo, Suchanek, and
  Clavel}]{chhun-etal-2022-human}
Cyril Chhun, Pierre Colombo, Fabian~M. Suchanek, and Chlo{\'e} Clavel. 2022.
\newblock Of human criteria and automatic metrics: A benchmark of the
  evaluation of story generation.
\newblock In \emph{Proceedings of the 29th International Conference on
  Computational Linguistics}, pages 5794--5836, Gyeongju, Republic of Korea.
  International Committee on Computational Linguistics.

\bibitem[{Chhun et~al.(2024)Chhun, Suchanek, and
  Clavel}]{chhun-etal-2024-language}
Cyril Chhun, Fabian~M. Suchanek, and Chlo{\'e} Clavel. 2024.
\newblock \href {https://doi.org/10.1162/tacl_a_00689} {Do language models
  enjoy their own stories? {P}rompting large language models for automatic
  story evaluation}.
\newblock \emph{Transactions of the Association for Computational Linguistics},
  12:1122--1142.

\bibitem[{Chiang and Lee(2023)}]{chiang-lee-2023-large}
Cheng-Han Chiang and Hung-yi Lee. 2023.
\newblock \href {https://doi.org/10.18653/v1/2023.acl-long.870} {Can large
  language models be an alternative to human evaluations?}
\newblock In \emph{Proceedings of the 61st Annual Meeting of the Association
  for Computational Linguistics (Volume 1: Long Papers)}, pages 15607--15631,
  Toronto, Canada. Association for Computational Linguistics.

\bibitem[{Chiang(1998)}]{chiang1998story}
Ted Chiang. 1998.
\newblock Story of your life.
\newblock In Patrick Nielsen~Hayden, editor, \emph{Starlight 2}, pages
  257--314. Tor Books, New York.

\bibitem[{Clark et~al.(2021)Clark, August, Serrano, Haduong, Gururangan, and
  Smith}]{clark-etal-2021-thats}
Elizabeth Clark, Tal August, Sofia Serrano, Nikita Haduong, Suchin Gururangan,
  and Noah~A. Smith. 2021.
\newblock \href {https://doi.org/10.18653/v1/2021.acl-long.565} {All that's
  `human' is not gold: Evaluating human evaluation of generated text}.
\newblock In \emph{Proceedings of the 59th Annual Meeting of the Association
  for Computational Linguistics and the 11th International Joint Conference on
  Natural Language Processing (Volume 1: Long Papers)}, pages 7282--7296.

\bibitem[{Dror et~al.(2018)Dror, Baumer, Shlomov, and
  Reichart}]{dror-etal-2018-hitchhikers}
Rotem Dror, Gili Baumer, Segev Shlomov, and Roi Reichart. 2018.
\newblock \href {https://doi.org/10.18653/v1/P18-1128} {The hitchhiker's guide
  to testing statistical significance in natural language processing}.
\newblock In \emph{Proceedings of the 56th Annual Meeting of the Association
  for Computational Linguistics (Volume 1: Long Papers)}, pages 1383--1392.

\bibitem[{Dubois et~al.(2024)Dubois, Galambosi, Liang, and
  Hashimoto}]{dubois2024length}
Yann Dubois, Bal{\'a}zs Galambosi, Percy Liang, and Tatsunori~B. Hashimoto.
  2024.
\newblock \href {https://arxiv.org/abs/2404.04475} {Length-controlled
  {AlpacaEval}: A simple way to debias automatic evaluators}.
\newblock In \emph{Proceedings of the First Conference on Language Modeling
  (COLM)}.

\bibitem[{Elson(2012)}]{elson-2012-dramabank}
David~K. Elson. 2012.
\newblock {DramaBank}: Annotating agency in narrative discourse.
\newblock In \emph{Proceedings of the Eighth International Conference on
  Language Resources and Evaluation ({LREC}'12)}, pages 2813--2819, Istanbul,
  Turkey. European Language Resources Association (ELRA).

\bibitem[{Fan et~al.(2018)Fan, Lewis, and Dauphin}]{fan-etal-2018-hierarchical}
Angela Fan, Mike Lewis, and Yann Dauphin. 2018.
\newblock \href {https://doi.org/10.18653/v1/P18-1082} {Hierarchical neural
  story generation}.
\newblock In \emph{Proceedings of the 56th Annual Meeting of the Association
  for Computational Linguistics (Volume 1: Long Papers)}, pages 889--898,
  Melbourne, Australia. Association for Computational Linguistics.

\bibitem[{Genette(1980)}]{genette1980narrative}
G{\'e}rard Genette. 1980.
\newblock \emph{Narrative Discourse: An Essay in Method}.
\newblock Cornell University Press, Ithaca, NY.
\newblock Translated by Jane E. Lewin.

\bibitem[{Han et~al.(2022)Han, Chen, Tian, and Peng}]{han-etal-2022-go}
Rujun Han, Hong Chen, Yufei Tian, and Nanyun Peng. 2022.
\newblock \href {https://doi.org/10.18653/v1/2022.naacl-main.104} {Go back in
  time: Generating flashbacks in stories with event temporal prompts}.
\newblock In \emph{Proceedings of the 2022 Conference of the North American
  Chapter of the Association for Computational Linguistics: Human Language
  Technologies}, pages 1450--1470, Seattle, United States. Association for
  Computational Linguistics.

\bibitem[{Harel-Canada et~al.(2024)Harel-Canada, Zhou, Muppalla, Yildiz, Kim,
  Sahai, and Peng}]{harelcanada2024measuring}
Fabrice~Y Harel-Canada, Hanyu Zhou, Sreya Muppalla, Zeynep~Senahan Yildiz,
  Miryung Kim, Amit Sahai, and Nanyun Peng. 2024.
\newblock \href {https://doi.org/10.18653/v1/2024.emnlp-main.953} {Measuring
  psychological depth in language models}.
\newblock In \emph{Proceedings of the 2024 Conference on Empirical Methods in
  Natural Language Processing}, pages 17162--17196, Miami, Florida, USA.
  Association for Computational Linguistics.

\bibitem[{Jung(1969)}]{jung1969archetypesessay}
Carl~Gustav Jung. 1969.
\newblock Archetypes of the collective unconscious.
\newblock In \emph{The Archetypes and the Collective Unconscious}, second
  edition, volume 9, Part 1 of \emph{The Collected Works of C. G. Jung}, pages
  3--41. Princeton University Press, Princeton, NJ.
\newblock Bollingen Series XX. Translated by R. F. C. Hull.

\bibitem[{Karpinska et~al.(2021)Karpinska, Akoury, and
  Iyyer}]{karpinska-etal-2021-perils}
Marzena Karpinska, Nader Akoury, and Mohit Iyyer. 2021.
\newblock \href {https://doi.org/10.18653/v1/2021.emnlp-main.97} {The perils of
  using {M}echanical {T}urk to evaluate open-ended text generation}.
\newblock In \emph{Proceedings of the 2021 Conference on Empirical Methods in
  Natural Language Processing}, pages 1265--1285.

\bibitem[{Kendall(1938)}]{kendall1938new}
Maurice~G. Kendall. 1938.
\newblock \href {https://doi.org/10.1093/biomet/30.1-2.81} {A new measure of
  rank correlation}.
\newblock \emph{Biometrika}, 30(1/2):81--93.

\bibitem[{Kim et~al.(2018)Kim, Bach, Im, Schriber, Gross, and
  Pfister}]{kim2018storycurves}
Nam~Wook Kim, Benjamin Bach, Hyejin Im, Sasha Schriber, Markus Gross, and
  Hanspeter Pfister. 2018.
\newblock \href {https://doi.org/10.1109/TVCG.2017.2744118} {Visualizing
  nonlinear narratives with story curves}.
\newblock \emph{IEEE Transactions on Visualization and Computer Graphics},
  24(1):595--604.

\bibitem[{Krippendorff(2018)}]{krippendorff2018content}
Klaus Krippendorff. 2018.
\newblock \emph{Content Analysis: An Introduction to Its Methodology}, fourth
  edition.
\newblock Sage.

\bibitem[{Li et~al.(2025)Li, Chen, Liu, Tan, Zhang, Li, Luo, Chen, Luo, Argha,
  Alinejad-Rokny, Zhou, and Yang}]{li-etal-2025-storyteller}
Jiaming Li, Yukun Chen, Ziqiang Liu, Minghuan Tan, Lei Zhang, Yunshui Li, Run
  Luo, Longze Chen, Jing Luo, Ahmadreza Argha, Hamid Alinejad-Rokny, Wei Zhou,
  and Min Yang. 2025.
\newblock \href {https://doi.org/10.18653/v1/2025.findings-acl.1071}
  {{STORYTELLER}: An enhanced plot-planning framework for coherent and cohesive
  story generation}.
\newblock In \emph{Findings of the Association for Computational Linguistics:
  ACL 2025}, pages 20818--20846. Association for Computational Linguistics.

\bibitem[{Mani(2012)}]{mani2012computational}
Inderjeet Mani. 2012.
\newblock \href {https://doi.org/10.2200/S00459ED1V01Y201212HLT018}
  {\emph{Computational Modeling of Narrative}}.
\newblock Number~18 in Synthesis Lectures on Human Language Technologies.
  Morgan \& Claypool.

\bibitem[{Montfort(2007)}]{montfort2007ordering}
Nick Montfort. 2007.
\newblock Ordering events in interactive fiction narratives.
\newblock In \emph{Intelligent Narrative Technologies: Papers from the 2007
  {AAAI} Fall Symposium (Technical Report FS-07-05)}. AAAI Press.

\bibitem[{Mostafazadeh et~al.(2016)Mostafazadeh, Chambers, He, Parikh, Batra,
  Vanderwende, Kohli, and Allen}]{mostafazadeh-etal-2016-corpus}
Nasrin Mostafazadeh, Nathanael Chambers, Xiaodong He, Devi Parikh, Dhruv Batra,
  Lucy Vanderwende, Pushmeet Kohli, and James Allen. 2016.
\newblock \href {https://doi.org/10.18653/v1/N16-1098} {A corpus and cloze
  evaluation for deeper understanding of commonsense stories}.
\newblock In \emph{Proceedings of the 2016 Conference of the North {A}merican
  Chapter of the Association for Computational Linguistics: Human Language
  Technologies}, pages 839--849, San Diego, California. Association for
  Computational Linguistics.

\bibitem[{Ning et~al.(2018)Ning, Wu, and Roth}]{ning-etal-2018-multi}
Qiang Ning, Hao Wu, and Dan Roth. 2018.
\newblock A multi-axis annotation scheme for event temporal relations.
\newblock In \emph{Proceedings of the 56th Annual Meeting of the Association
  for Computational Linguistics (Volume 1: Long Papers)}, pages 1318--1328,
  Melbourne, Australia. Association for Computational Linguistics.

\bibitem[{Piper and Bagga(2024)}]{piper-bagga-2024-using}
Andrew Piper and Sunyam Bagga. 2024.
\newblock Using large language models for understanding narrative discourse.
\newblock In \emph{Proceedings of the 6th Workshop on Narrative Understanding},
  pages 37--46. Association for Computational Linguistics.

\bibitem[{Pustejovsky et~al.(2003)Pustejovsky, Hanks, Saur{\'\i}, See,
  Gaizauskas, Setzer, Radev, Sundheim, Day, Ferro, and
  Lazo}]{pustejovsky2003timebank}
James Pustejovsky, Patrick Hanks, Roser Saur{\'\i}, Andrew See, Robert
  Gaizauskas, Andrea Setzer, Dragomir Radev, Beth Sundheim, David Day, Lisa
  Ferro, and Marcia Lazo. 2003.
\newblock The {TIMEBANK} corpus.
\newblock In \emph{Proceedings of Corpus Linguistics}, pages 647--656.

\bibitem[{Robbe-Grillet(1959)}]{robbegrillet1959jealousy}
Alain Robbe-Grillet. 1959.
\newblock \emph{Jealousy}.
\newblock Grove Press, New York.
\newblock Translated from the French by Richard Howard.

\bibitem[{Singer(2016)}]{singer2016citizenkane}
Leigh Singer. 2016.
\newblock \href
  {https://www.bfi.org.uk/features/20-inspired-visual-moments-citizen-kane} {20
  inspired visual moments in {Citizen Kane}}.
\newblock British Film Institute.

\bibitem[{Thucydides(2009)}]{thucydides2009peloponnesian}
Thucydides. 2009.
\newblock \emph{The Peloponnesian War}.
\newblock Oxford World's Classics. Oxford University Press, Oxford.
\newblock Translated by Martin Hammond; introduction and notes by P. J. Rhodes.

\bibitem[{Tian et~al.(2024)Tian, Huang, Liu, Jiang, Spangher, Chen, May, and
  Peng}]{tian-etal-2024-large}
Yufei Tian, Tenghao Huang, Miri Liu, Derek Jiang, Alexander Spangher, Muhao
  Chen, Jonathan May, and Nanyun Peng. 2024.
\newblock \href {https://doi.org/10.18653/v1/2024.emnlp-main.978} {Are large
  language models capable of generating human-level narratives?}
\newblock In \emph{Proceedings of the 2024 Conference on Empirical Methods in
  Natural Language Processing}, pages 17659--17681, Miami, Florida, USA.
  Association for Computational Linguistics.

\bibitem[{UzZaman et~al.(2013)UzZaman, Llorens, Derczynski, Allen, Verhagen,
  and Pustejovsky}]{uzzaman-etal-2013-semeval}
Naushad UzZaman, Hector Llorens, Leon Derczynski, James Allen, Marc Verhagen,
  and James Pustejovsky. 2013.
\newblock {SemEval}-2013 task 1: {TempEval}-3: Evaluating time expressions,
  events, and temporal relations.
\newblock In \emph{Second Joint Conference on Lexical and Computational
  Semantics (*SEM), Volume 2: Proceedings of the Seventh International Workshop
  on Semantic Evaluation (SemEval 2013)}, pages 1--9, Atlanta, Georgia, USA.
  Association for Computational Linguistics.

\bibitem[{Vogler(2007)}]{vogler2007writers}
Christopher Vogler. 2007.
\newblock \emph{The Writer's Journey: Mythic Structure for Writers}, 3rd
  edition.
\newblock Michael Wiese Productions, Studio City, CA.

\bibitem[{Wagner et~al.(2025)Wagner, Keydar, and Abend}]{wagner2025fair}
Eitan Wagner, Renana Keydar, and Omri Abend. 2025.
\newblock \href {https://arxiv.org/abs/2507.13841} {The challenge and reward of
  fair play in narrative: A computational approach}.
\newblock \emph{Computing Research Repository}, arXiv:2507.13841.

\bibitem[{Xie and Riedl(2024)}]{xie-riedl-2024-creating}
Kaige Xie and Mark Riedl. 2024.
\newblock \href {https://doi.org/10.18653/v1/2024.eacl-long.147} {Creating
  suspenseful stories: Iterative planning with large language models}.
\newblock In \emph{Proceedings of the 18th Conference of the European Chapter
  of the Association for Computational Linguistics (Volume 1: Long Papers)},
  pages 2391--2407, St. Julian's, Malta. Association for Computational
  Linguistics.

\bibitem[{Yang et~al.(2022)Yang, Tian, Peng, and Klein}]{yang-etal-2022-re3}
Kevin Yang, Yuandong Tian, Nanyun Peng, and Dan Klein. 2022.
\newblock \href {https://doi.org/10.18653/v1/2022.emnlp-main.296} {{Re3}:
  Generating longer stories with recursive reprompting and revision}.
\newblock In \emph{Proceedings of the 2022 Conference on Empirical Methods in
  Natural Language Processing}, pages 4393--4479, Abu Dhabi, United Arab
  Emirates. Association for Computational Linguistics.

\bibitem[{Yao et~al.(2019)Yao, Peng, Weischedel, Knight, Zhao, and
  Yan}]{yao2019plan}
Lili Yao, Nanyun Peng, Ralph Weischedel, Kevin Knight, Dongyan Zhao, and Rui
  Yan. 2019.
\newblock \href {https://doi.org/10.1609/aaai.v33i01.33017378} {Plan-and-write:
  Towards better automatic storytelling}.
\newblock In \emph{Proceedings of the AAAI Conference on Artificial
  Intelligence}, volume~33, pages 7378--7385. AAAI Press.

\bibitem[{Yu et~al.(2025)Yu, Shi, Zhao, and Penn}]{yu-etal-2025-multi}
Tian Yu, Ken Shi, Zixin Zhao, and Gerald Penn. 2025.
\newblock \href {https://doi.org/10.18653/v1/2025.in2writing-1.9} {Multi-agent
  based character simulation for story writing}.
\newblock In \emph{Proceedings of the Fourth Workshop on Intelligent and
  Interactive Writing Assistants (In2Writing 2025)}, pages 87--108. Association
  for Computational Linguistics.

\bibitem[{Zheng et~al.(2023)Zheng, Chiang, Sheng, Zhuang, Wu, Zhuang, Lin, Li,
  Li, Xing, Zhang, Gonzalez, and Stoica}]{zheng2023judging}
Lianmin Zheng, Wei-Lin Chiang, Ying Sheng, Siyuan Zhuang, Zhanghao Wu, Yonghao
  Zhuang, Zi~Lin, Zhuohan Li, Dacheng Li, Eric~P. Xing, Hao Zhang, Joseph~E.
  Gonzalez, and Ion Stoica. 2023.
\newblock \href {https://arxiv.org/abs/2306.05685} {Judging {LLM}-as-a-judge
  with {MT-Bench} and {Chatbot Arena}}.
\newblock In \emph{Advances in Neural Information Processing Systems 36: Annual
  Conference on Neural Information Processing Systems 2023, Datasets and
  Benchmarks Track}.

\end{thebibliography}
		\appendix

	\section*{Contents of the appendix}
	\vspace{-2pt}
	{\small
		\begin{tabbing}
			\hspace{2.1em}\=\hspace{20.5em}\=\kill
			\ref{sec:shapes} \> Preliminaries: the seven shapes
			\> p.~\pageref{sec:shapes} \\
			\ref{sec:setup-extra} \> Experimental setup in detail
			\> p.~\pageref{sec:setup-extra} \\
			\ref{sec:related-extra} \> Related work, in detail
			\> p.~\pageref{sec:related-extra} \\
			\ref{sec:formal} \> Formal analysis, with proofs
			\> p.~\pageref{sec:formal} \\
			\ref{sec:tables} \> Supporting tables and figures
			\> p.~\pageref{sec:tables} \\
			\ref{sec:human-extra} \> The reader study in full
			\> p.~\pageref{sec:human-extra} \\
			\ref{sec:recommendations} \> Reporting recommendations
			\> p.~\pageref{sec:recommendations} \\
			\ref{sec:validation} \> What the corpus build checks
			\> p.~\pageref{sec:validation} \\
			\ref{sec:grid-notes} \> Notes on the benchmark grid
			\> p.~\pageref{sec:grid-notes} \\
			\ref{sec:rcg} \> Retrospective coherence gain
			\> p.~\pageref{sec:rcg} \\
			\ref{sec:confusion} \> Structural classification
			\> p.~\pageref{sec:confusion} \\
			\ref{sec:prompts} \> Task and judge prompts
			\> p.~\pageref{sec:prompts} \\
			\ref{sec:instrument} \> The reader instrument
			\> p.~\pageref{sec:instrument} \\
			\ref{sec:appendix} \> Reproduction
			\> p.~\pageref{sec:appendix} \\
			\ref{sec:proofs} \> Proofs
			\> p.~\pageref{sec:proofs} \\
	\end{tabbing}}
	\vspace{-6pt}

	\section{Preliminaries: the seven shapes}
	\label{sec:shapes}

	\begin{figure*}[t]
		\centering
		\begin{subfigure}{0.235\textwidth}\centering
			\includegraphics[width=\linewidth]{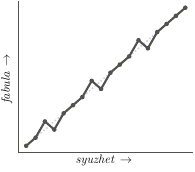}
			\caption{\textsc{chronicle}, $\tau=\ShapeTauChronicle$}\label{fig:shape-chronicle}
		\end{subfigure}\hfill
		\begin{subfigure}{0.235\textwidth}\centering
			\includegraphics[width=\linewidth]{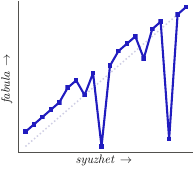}
			\caption{\textsc{analepsis}, $\tau=\ShapeTauAnalepsis$}\label{fig:shape-analepsis}
		\end{subfigure}\hfill
		\begin{subfigure}{0.235\textwidth}\centering
			\includegraphics[width=\linewidth]{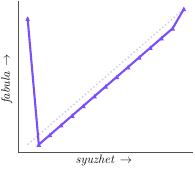}
			\caption{\textsc{prolepsis}, $\tau=\ShapeTauProlepsis$}\label{fig:shape-prolepsis}
		\end{subfigure}\hfill
		\begin{subfigure}{0.235\textwidth}\centering
			\includegraphics[width=\linewidth]{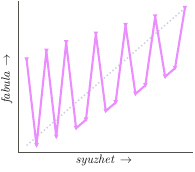}
			\caption{\textsc{zigzag}, $\tau=\ShapeTauZigzag$}\label{fig:shape-zigzag}
		\end{subfigure}

		\vspace{1.5mm}

		\begin{subfigure}{0.235\textwidth}\centering
			\includegraphics[width=\linewidth]{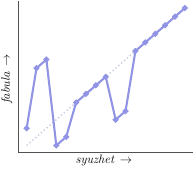}
			\caption{\textsc{syllepsis}, $\tau=\ShapeTauSyllepsis$}\label{fig:shape-syllepsis}
		\end{subfigure}\hfill
		\begin{subfigure}{0.235\textwidth}\centering
			\includegraphics[width=\linewidth]{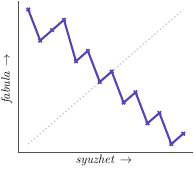}
			\caption{\textsc{retrograde}, $\tau=\ShapeTauRetrograde$}\label{fig:shape-retrograde}
		\end{subfigure}\hfill
		\begin{subfigure}{0.235\textwidth}\centering
			\includegraphics[width=\linewidth]{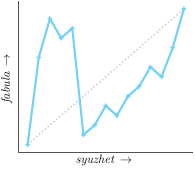}
			\caption{\textsc{achrony}, $\tau=\ShapeTauAchrony$}\label{fig:shape-achrony}
		\end{subfigure}\hfill

		\caption{The seven shapes, each drawn from the corpus work nearest its
			category's median in $(\tau,\mathrm{LIS})$. Told position runs across and
			chronological position up, so chronological order is the dotted diagonal.
			Published \textsc{retrograde} descends in steps with short forward runs,
			not in the straight corner-to-corner line models return
			(\S\ref{sec:cued}). Colors match Figure~\ref{fig:plane}.}
		\label{fig:shapes}
	\end{figure*}

	This section restates \S\ref{sec:setting} for readers new to the
	vocabulary, with a picture of each shape.

	A narrative has two orders. The \fabula{} is the order in which things
	happened, and the \syuzhet{} is the order in which the telling presents them
	\citep{genette1980narrative}. Number the events $1,\dots,n$ by when they
	happened, then read the telling from its first moment to its last and write
	down the number of whatever is narrated at each step. That list of numbers is
	the whole object this paper studies: it is a permutation, it can be compared
	to another permutation by arithmetic, and it is what the corpus records
	alongside the events themselves.

	Figure~\ref{fig:shapes} draws that permutation for one work of each shape,
	told position running across and chronological position running up. A story
	told in the order things happened would trace the dotted diagonal exactly.
	Everything interesting in this paper is a departure from that diagonal:
	Kendall's $\tau$ measures how much of the diagonal's forward direction
	survives, and the normalized longest increasing subsequence measures the largest share
	of events that are already in chronological order, whether or not they are
	adjacent in the telling. Together they
	distinguish a telling that moves one block out of place from one that
	interleaves two timelines throughout, although the two can have the same $\tau$.

The seven shapes shown in Figure~\ref{fig:shapes} describe the corpus labels at the scene level.
A scene may contain several events, so its position in the telling
must be distinguished from the order of events within it.
\S\ref{sec:prompts} reproduces the definitions used in the
experiments and identifies their differences from these corpus definitions.

    \paragraph{\textsc{chronicle}.}
    The telling follows chronology. The line sits on the diagonal, broken at most
    where two things happen at once and the telling has to pick an order for them.
    \textcolor{black}{For example, Montfort identifies chronological stretches
    of Thucydides' \textit{The Peloponnesian War} as chronicles
    \citep{thucydides2009peloponnesian}. The corpus includes this shape as the case in
    which the correct ordering leaves chronology intact.}

    \paragraph{\textsc{analepsis}.}
    One timeline dominates , and the telling interrupts it with material from earlier,
    then returns to the dominant timeline. This is the flashback, and it is by a wide margin the most common nonlinear shape in the corpus.
    \textcolor{black}{In \textit{Lost}, for example, flashbacks revisit the
    characters' lives before they were stranded on the island
    \citep{kim2018storycurves}. In the curve, a flashback appears as a backward
    excursion from the surrounding chronological progression, followed by a
    return.}

    \paragraph{\textsc{prolepsis}.}
    A mirror of \textsc{analepsis} where one timeline dominates and the telling  interrupts it with material from later, then
    returns to the dominant timeline.
    The interruption must be a real scene the reader can place afterwards, not a
    vague foreboding, which is what separates a flash-forward from atmosphere.
    \textcolor{black}{\textit{Arrival}, for example, presents episodes from a
    character's future \citep{chiang1998story}. In the curve, these anticipations
    appear as forward excursions followed by a return to the dominant timeline.}

    \paragraph{\textsc{zigzag}.}
    \textcolor{black}{Two periods alternate on a regular beat, each
    progressing in a consistent direction, with paired scenes related by echo,
    explanation or contrast.} The line makes a regular sawtooth.
    \textcolor{black}{In \textit{Memento}, the black-and-white scenes move
    forward while the color scenes move backward; their interleaving illustrates
    how a zigzag can combine strands with different temporal directions
    \citep{kim2018storycurves}.}

    \paragraph{\textsc{syllepsis}.}
    No timeline dominates; the scenes are grouped by something other than time,
    such as place, character, theme or object. The test the corpus applies is
    whether the grouping principle can be stated in a sentence. The line is
    blocky, orderly inside each block and unordered between them.
    \textcolor{black}{For instance, the multiple
    plotlines of \textit{Pulp Fiction} and \textit{Love Actually} can be used to illustrate this
    grouping principle \citep{kim2018storycurves}.}

    \paragraph{\textsc{retrograde}.}
    The telling runs backwards. The panel shows what published reverse chronology
    actually looks like and why it matters to this paper: a descending staircase,
    not a descending line. Scenes are told in reverse order while each scene plays
    forward inside itself, so the line falls step by step with a short forward run
    on every step.
    \textcolor{black}{The color strand of \textit{Memento} offers a familiar
    example: successive scenes lead back to earlier events
    \citep{kim2018storycurves}.}
    The pure reversal, a single line from top left to bottom right,
    is the shape no work in the corpus has and the one models produce when
    instructed (\S\ref{sec:cued}).

    \paragraph{\textsc{achrony}.}
    \textcolor{black}{In the broader narratological sense, achrony describes a
    telling whose temporal relations cannot readily be established, \textit{e.g.}, \citet{robbegrillet1959jealousy}'s \textit{La Jalousie}. The corpus uses a narrower operational label: the
    chronology must remain reconstructable, while the telling follows no statable
    ordering pattern and is not merely awaiting a late reveal
    (\S\ref{sec:corpus}).} The line wanders. The definition the models saw (\S\ref{sec:prompts})	describes a telling whose chronological ranks cannot be assigned, which is
	the case the corpus excludes. This partly explains the low \textsc{achrony} recall reported in \S\ref{sec:negatives}.

\section{Experimental setup in detail}
\label{sec:setup-extra}

This appendix gives the setup details that \S\ref{sec:setup} summarizes.

\paragraph{What the models see.}
Each base prompt for \Task{1}--\Task{4} and \ProseTask{a}/\ProseTask{b} lists one self-contained event
sentence per line. \Task{1}, \Task{2}, \ProseTask{a} and \ProseTask{b} receive the events in chronological
order; \Task{3} and \Task{4} receive them in the gold telling order. A model never sees
the story id, either order column, scene annotations, metrics, annotator
notes, graphs, or the de-anonymization key. \S\ref{sec:prompts} gives the
verbatim prompts and \S\ref{sec:grid-notes} the decoding and scoring
conventions.

\paragraph{Reconstruction controls.}
\Task{3} normally receives the gold telling order. To test how much of its score
comes from that order rather than from the event descriptions, three
selected models also reconstruct chronology from a shuffled event list.
We call this shuffled-input score the reconstruction baseline or
\emph{no-story floor}; it is measured for these three models rather than
for every evaluated model. The full-telling and shuffled-input conditions
are also rerun on descriptions rewritten by a model to remove explicit
time cues, such as dates and ``years earlier'' (Table~\ref{tab:dedate}).
In both reconstruction conditions, indices identify event lines
rather than gold chronological ranks (\S\ref{sec:grid-notes}).

\paragraph{Graph conditions.}
\GraphCondition{0}, \GraphCondition{1} and \GraphCondition{2} all retain the \ProseTask{a} target-shape
cue; only the graph or order input differs, and none uses \ProseTask{b}. Each
ablation removes one annotation group from \GraphCondition{1} together with its
description in the prompt. The main graph conditions use \GraphNModels{}
models; the ablations use the \GraphCommonModels{} generators that ran all
conditions, with one replicate. The graph study comprises \GraphNStories{}
stories and \GraphNJudged{} blindly judged pairs (Table~\ref{tab:grapharm}).

\paragraph{Band hit and pooled scores.}
A band hit measures resemblance to corpus examples; it does not fully test
a narrative shape, because different structures can share both scores and
a telling outside the band may still follow the shape. Pooled scores
average valid rows of the \NModelsOrdering{} evaluated models;
higher-effort runs, permutation baselines and the human reference are
reported separately.

\paragraph{Psychological depth.}
The judge rates authenticity, empathy, engagement, emotion provocation and
narrative complexity from 1 to 5 using the Psychological Depth Scale
\citep{harelcanada2024measuring}; readers use the same scales during close
readings.

\paragraph{Premature disclosure.}
For each work with a designated reveal, we take the \texttt{REVEALS} edge
whose explanation appears latest in the gold telling. Disclosure is early
if fewer than three events precede that explanation's aligned first mention
(\S\ref{sec:rcg}).

\paragraph{Reader statistics.}
Reader preference rates exclude ties. We report exact binomial tests and
95\% cluster-bootstrap intervals, resampling separately over readers and
stories \citep{dror-etal-2018-hitchhikers}; the Bonferroni threshold is
$0.0042$ for twelve tests. Length-adjusted rates use generalized estimating
equations with word-count difference and reader clusters; length-matched
subsets contain pairs whose lengths differ by at most 40 words
(\S\ref{sec:appendix}).
	\section{Related work, in detail}
	\label{sec:related-extra}

	\S\ref{sec:related} states where this work sits. This section discusses each
	line of work it draws on.

	\paragraph{Narratology.}
	The \fabula{}~/~\syuzhet{} distinction and the vocabulary our taxonomy uses,
	from analepsis and prolepsis to syllepsis and achrony, come from Genette
	\citeyearpar{genette1980narrative}. The seven structural shapes of
	\S\ref{sec:setting} partition that vocabulary into labels a scene permutation
	can be read against, which is the step that turns a critical vocabulary into
	something an experiment can score.

	\paragraph{Story generation with language models.}
	Hierarchical generation separates a premise from its realization
	\citep{fan-etal-2018-hierarchical}, Plan-and-Write constructs a storyline
	before prose \citep{yao2019plan}, and Re3 combines planning, repeated
	prompting and revision for longer stories \citep{yang-etal-2022-re3}.
	\citet{tian-etal-2024-large} show that model-written narratives are
	structurally homogeneous relative to human ones; \S\ref{sec:spontaneous}
	reproduces that result on a controlled corpus and sharpens it, because holding
	the events fixed removes the possibility that the difference is in what
	happens rather than in how it is told. \citet{xie-riedl-2024-creating} plan
	iteratively for suspense and \citet{li-etal-2025-storyteller} for plot
	coherence. Suspense planning is a close case for us, since withholding is a
	discourse decision and therefore a statement about the told order.

	\paragraph{Controlling the order of telling.}
	Two lines target the variable we study. \citet{yu-etal-2025-multi} simulate
	characters with multiple agents and then rewrite the chronological simulation
	into narrative order, which makes the two orders explicit inside a system
	rather than inside a corpus. \citet{han-etal-2022-go} guide flashback
	generation through temporal prompts between neighboring events, which is local
	control of the same variable. Our benchmark complements both by representing
	the telling order as a permutation that can be specified, scored and compared
	against published practice across seven structural categories while the events
	stay fixed, which is what makes the over-compliance of \S\ref{sec:cued} visible at
	all. A check on direction alone would score it as compliance.

	\paragraph{Discourse order before language models.}
	Separating the order of telling from chronology has computational
	precedent. \citet{montfort2007ordering} formalizes Genette's categories of
	order in an interactive-fiction system that narrates one simulated event
	model in different, non-chronological orders. Plan-based generators place
	flashback and foreshadowing to arouse surprise \citep{bae2008flashback} and
	select and order story content to evoke suspense \citep{cheong2008suspense}.
	Story Intention Graphs annotate a text against a timeline layer
	\citep{elson-2012-dramabank}, and \citet{mani2012computational} surveys how
	narrative time has been modeled computationally. Story Curves aligns the two
	scene orders for a set of films \citep{kim2018storycurves}. What we add is
	scale across seven shapes and \NWorks{} works, a typed graph over the same
	events, and a benchmark that scores language models against published
	practice.

	\paragraph{Temporal ordering.}
	Recovering chronology from text is a long-standing task in its own right, from
	TimeBank \citep{pustejovsky2003timebank} and TempEval
	\citep{uzzaman-etal-2013-semeval} to the multi-axis scheme of
	\citet{ning-etal-2018-multi}. \citet{chambers-jurafsky-2008-unsupervised} is
	the closest ancestor of our framing, because it treats event order as
	narrative structure rather than as pairwise relation labeling. Those resources
	annotate relations inside a text as given; we annotate the permutation between
	two orders of the same events, which makes the telling order a target one can
	specify rather than only a label one can recover. Our \Task{3} is a document-level
	instance of their task, and \S\ref{sec:floor} is why the connection matters:
	any ordering evaluation over self-contained event sentences is exposed to the
	entailment floor we measure there. Sentence-ordering benchmarks built on ROCStories
	\citep{mostafazadeh-etal-2016-corpus} are the nearest analogue of \Task{3}. They
	also ask for the order of self-contained sentences, so the same floor
	applies to them.

	\paragraph{Evaluation instruments.}
	We adopt the Psychological Depth Scale of \citet{harelcanada2024measuring}
	unchanged for reader-experience ratings, and follow
	\citet{piper-bagga-2024-using} in treating narrative discourse properties as
	things a language model can be probed for. Closest to our
	retrospective-coherence measure is \citet{wagner2025fair}, who pair surprise
	with coherence without a reference by simulating readers who guess a culprit
	in detective fiction. Because our corpus carries a gold chronology and gold
	reveal annotations, we score against a known answer rather than a proxy, and
	across all seven structural categories rather than whodunits alone.

	\paragraph{Judges and readers.}
	Automated pairwise preference is standard and its biases are catalogued:
	position, verbosity and self-enhancement \citep{zheng2023judging}, with the
	verbosity effect large enough to warrant the regression adjustment of
	\citet{dubois2024length} that removes one of our own results
	(\S\ref{sec:humanlength}). Readers are themselves barely above chance at
	spotting machine text \citep{clark-etal-2021-thats} and such studies rarely
	report enough to interpret \citep{karpinska-etal-2021-perils}, so we hold the
	texts, the question and the items fixed and ask where a judge and a reader
	cohort part company. Agreement is reported with chance-corrected coefficients
	throughout \citep{artstein-poesio-2008-survey,krippendorff2018content} and
	every test is stated against the dependence structure the design has
	\citep{dror-etal-2018-hitchhikers}.
	The closest work to our RQ5 is \citet{chiang-lee-2023-large}, who compare
	model ratings of stories with those of English teachers. HANNA and its
	successor test automatic story evaluation against human criteria
	\citep{chhun-etal-2022-human,chhun-etal-2024-language}, and
	\citet{chakrabarty2024art} find model judges misaligned with expert
	assessments of creative writing. \citet{zheng2023judging} report that
	judge--human agreement rises with the performance gap between the systems
	compared. Our four comparisons show no such ordering at the item level,
	and our judge returned ties on two comparisons that readers split about 70
	to 30. We therefore state the lesson as ``a judge's tie is not evidence of
	no difference'' and not as a threshold on effect size.

	\paragraph{Visualization.}
	Story Curves already aligns narrative and chronological scene orders in
	nonlinear films \citep{kim2018storycurves}. Figure~\ref{fig:curve} adapts its
	story curve; our corpus adds shared event descriptions, typed narrative graphs
	and model tasks to that paired-order setting. Color encodes jump direction
	rather than focal character, because our stories reach \MaxChars{} characters,
	and direction is encoded twice, by hue and by slope, which addresses the axis
	confusion the original study reports in its own participants.

	\section{Formal analysis}
	\label{sec:formal}

	This appendix collects the formal statements \S\ref{sec:setting} and
	\S\ref{sec:evaluation} summarize, and the three short results \S\ref{sec:capability}
	points at. Nothing here is assumed by the experiments; it says what the scores
	can and cannot see. \S\ref{sec:proofs} proves everything stated.

	\subsection{Definitions used throughout}
	\label{sec:formal-defs}

	Let $E=(e_1,\ldots,e_n)$ be the events of a work in chronological order and
	$\mathcal{S}_n$ the permutations of $\{1,\ldots,n\}$. A telling is
	\begin{equation}
		T_\sigma=(e_{\sigma(1)},\ldots,e_{\sigma(n)}),
		\qquad \sigma\in\mathcal{S}_n ,
		\label{eq:telling}
	\end{equation}
	so $\sigma(i)$ is the chronological rank at told position $i$, chronology is
	the identity, and the telling is nonlinear when $\sigma\neq\mathrm{id}$. Each
	work supplies a gold telling $\sigma^\star$ and a category
	$c^\star\in\mathcal{C}$ with $|\mathcal{C}|=7$.

	Writing $D(\sigma)$ for the number of told position pairs $i<j$ with
	$\sigma(i)>\sigma(j)$, forwardness is Kendall's $\tau$
	\citep{kendall1938new},
	\begin{equation}
		\tau(\sigma)=1-\frac{2D(\sigma)}{\binom{n}{2}},
		\label{eq:tau}
	\end{equation}
	and $\mathrm{LIS}(\sigma)=\ell(\sigma)/n$ is the normalized longest increasing
	subsequence, whose $\ell(\sigma)$ events need not be adjacent in the telling.
	For category $c$, let $\tau_c^-,\tau_c^+$ and $L_c^-,L_c^+$ be the minima and
	maxima of $\tau$ and LIS over the gold tellings of that category. Band hit is
	\begin{align}
		B_c&=[\tau_c^-,\tau_c^+]\times[L_c^-,L_c^+], \nonumber\\
		\mathrm{hit}(\sigma,c)
		&=\mathbf{1}\{(\tau(\sigma),\mathrm{LIS}(\sigma))\in B_c\}.
		\label{eq:band}
	\end{align}
	Agreement with the gold telling of a model order $\hat{\sigma}$ is
	$\tau((\sigma^\star)^{-1}\circ\hat{\sigma})$.

	\paragraph{What the pair $(\tau,\mathrm{LIS})$ cannot separate.}
	Band hit reads two coordinates, and distinct structures share them. With
	$n = 6$, the tellings $(1,2,5,6,3,4)$ and $(1,3,5,2,6,4)$ both have
	$\tau = 7/15$ and $\mathrm{LIS} = 4/6$, so no band can tell them apart,
	yet the first has two hops and the second five. Enumerating every telling
	at $n = \CollideN$ gives \CollideCells{} distinct $(\tau,\mathrm{LIS})$
	cells, \CollideHops{} of which hold more than one hop count and
	\CollideLabels{} more than one label under the operational tests of
	\S\ref{sec:prompts}, up to three of the seven in a single cell. Two
	summary statistics therefore cannot identify seven semantically defined
	shapes, which is why band hit is reported as resemblance to published
	practice and never as a shape test, and why
	Tables~\ref{tab:category} and~\ref{tab:grapharm} print agreement with the
	gold telling beside it.

	The five tasks of \S\ref{sec:tasks} map as
	\begin{align}
		\Task{1}:&\quad (E,c^\star)\longmapsto\hat{\sigma}\in\mathcal{S}_n,
		\nonumber\\
		\Task{2}:&\quad E\longmapsto\hat{\sigma}\in\mathcal{S}_n,
		\nonumber\\
		\Task{3}:&\quad T_{\sigma^\star}\longmapsto\hat{\pi}\in\mathcal{S}_n,
		\nonumber\\
		\Task{4}:&\quad T_{\sigma^\star}\longmapsto\hat{c}\in\mathcal{C},
		\nonumber\\
		\ProseTask{a}/\ProseTask{b}:&\quad (E,c^\star)\ \text{or}\ E\longmapsto Y,
		\label{eq:tasks}
	\end{align}
	with $Y$ story prose. \Task{3} returns indices into the presented telling arranged
	chronologically, so its target is $\pi^\star=(\sigma^\star)^{-1}$ and its two
	scores are
	\begin{equation}
		\tau_{\mathrm{rec}}=\tau(\sigma^\star\circ\hat{\pi}),
		\qquad
		\mathrm{exact}=\mathbf{1}\{\hat{\pi}=\pi^\star\}.
		\label{eq:reconstruction}
	\end{equation}

	\subsection{Hops}
	\label{sec:formal-results}

	The rest of this appendix states three results and what each one predicts.
	Write a telling of $n$ events as a permutation $\sigma$, where $\sigma(i)$ is
	the chronological rank of the event told at position $i$, so that chronology
	is the identity.
	{

		\begin{definition}[Hop, run]\label{def:hop}
			Position $i<n$ is a \emph{hop} of $\sigma$ if $\sigma(i+1)\neq\sigma(i)+1$,
			that is, if the telling does not simply continue with the next thing that
			happened. $h(\sigma)$ is the number of hops. The maximal hop-free stretches of
			$\sigma$ are its \emph{runs}, with lengths $b_1,\dots,b_{h+1}$ summing to $n$.
		\end{definition}

		A hop is the unit of nonlinearity, and a nonlinear telling is a multi-hop
		plan. At every hop the writer must decide where in chronology to go next, and
		the reader must work out where they have landed. A hop also has a price in the
		scores of \S\ref{sec:setting}, because both are arithmetic on the permutation.

		\begin{proposition}[What hops cost]\label{prop:hops}
			For every telling $\sigma$ of $n$ events with runs $b_1,\dots,b_{h+1}$,
			\begin{enumerate}\itemsep1pt\parskip0pt
				\item[(i)] $h(\sigma)=0$ if and only if $\sigma$ is chronology;
				\item[(ii)] $\mathrm{LIS}(\sigma)\ge\max_j b_j/n\ge 1/(h+1)$;
				\item[(iii)] $\tau(\sigma)\ge-1+2\sum_j\binom{b_j}{2}\big/\binom{n}{2}$,
				with equality exactly when the runs are told in decreasing chronological
				order, in which case $\mathrm{LIS}(\sigma)=\max_j b_j/n$;
				\item[(iv)] at most $(h+1)(n-1)^h$ tellings have exactly $h$ hops.
			\end{enumerate}
		\end{proposition}
		\emph{Takeaway. Nonlinearity is paid for hop by hop. A telling with few hops
			must keep a long forward run, but it need not have a positive $\tau$. With $n = 6$, the telling
			$(4,5,6,1,2,3)$ has one hop and $\tau = -0.2$, the value (iii) gives for
			two runs of three, so (ii) and (iii) constrain run length and still allow
			a negative $\tau$. The number of tellings to
			choose among grows at most like $n^h$, and the only telling that costs no hop is
			chronology.}

		\begin{remark}
			On \Task{2} the human tellings hop at \HumanHopRate{} of their positions and the
			pooled model output at \PoolHopRateFree{}, and \NFewerHops{} of the
			\NModelsOrdering{} models hop less often than the human telling of the same
			works. When the model writes prose, where every hop must also be written,
			hops all but
			vanish ($\tau=\ProseFreeTau$). The proposition does not say why a generator
			spends hops sparingly, only that sparing hops and telling it straight are the
			same thing, which is RQ1.
		\end{remark}
	}

	{
		\begin{proposition}[The rule and the practice]\label{prop:rule}
			Let the $n$ events be partitioned into $k$ scenes of sizes $s_1,\dots,s_k$,
			each contiguous in chronology.
			\begin{enumerate}\itemsep1pt\parskip0pt
				\item[(i)] The reversal $n,\dots,1$ is the only telling with $\tau=-1$. It
				has $n-1$ hops and $\mathrm{LIS}=1/n$.
				\item[(ii)] The scene-level retrograde, which tells the scenes in reverse
				chronological order and the events of each scene forward, has $k-1$ hops,
				$\mathrm{LIS}=\max_j s_j/n$, and
				\[
				\tau\;=\;-1+2\sum_{j=1}^{k}\binom{s_j}{2}\Big/\binom{n}{2},
				\]
				which exceeds $-1$ as soon as one scene holds two events.
				\item[(iii)] Symmetrically, the identity is the only telling with $\tau=+1$,
				so a chronicle that reorders even one pair of simultaneous events leaves it.
			\end{enumerate}
		\end{proposition}
		\emph{Takeaway. The \textnormal{\textsc{retrograde}} rule of \S\ref{sec:prompts} is
			stated on events, and on events it has exactly one solution, the mechanical
			extreme. The human practice is the same rule applied one level up, on scenes,
			and reaching it takes a step the rule never states, deciding where the scenes
			are.}

		\begin{remark}
			\RetroExactRev\% of \textsc{retrograde} outputs are the reversal of (i), the
			point $(-1,1/n)$ where the violet stack sits in
			Figure~\ref{fig:plane}. No human
			work is there. Every retrograde work in the corpus keeps a forward run of at
			least \RetroRunLo{} events, and \RetroSceneExact{} of the ten are exactly the
			scene-level retrograde of (ii), whose formula returns their $\tau$ to three
			decimals. The rest sit above that value, never below it. The proposition says
			the missing step is a segmentation, not that it is hard. Our prompts
			never mark the scenes (\S\ref{sec:tasks}), so the rule the models were given
			is the event-level one, and the small effort ablation of
			\S\ref{sec:grid-notes} cannot say whether more test-time compute would
			supply the step.
		\end{remark}
	}

	{
		\subsection{Pinned pairs}
		The second half of the paper turns on scores that give partial credit, and
		one piece of arithmetic covers the no-story floor, the RCG ceiling and the
		scaffold leak.

		\begin{definition}[Pinned pair]\label{def:pinned}
			Fix a target order on $n$ events and a solver that returns an order. A pair of
			events is \emph{pinned} if the solver places it as the target does with
			probability one. Write $\lambda$ for the share of the $\binom{n}{2}$ pairs that
			are pinned and $\alpha$ for the solver's mean accuracy on the rest.
		\end{definition}

		\begin{lemma}[Partial credit]\label{lem:pinned}
			The solver's expected Kendall $\tau$ against the target is
			\[
			\mathbb{E}[\tau]\;=\;\lambda+(1-\lambda)(2\alpha-1).
			\]
			If the solver is exchangeable over some set of $w$ events, no two of which are
			pinned to each other, then $\Pr[\text{exact match}]\le 1/w!$.
		\end{lemma}
		\emph{Takeaway. A partial-credit score starts at the pinned share, which is
			$\lambda$ for a solver that knows nothing else ($\alpha=\tfrac12$), and every
			further point has to be earned on the unpinned pairs. An all-or-nothing score
			is at most $1/w!$ however many pairs are pinned.}

		\begin{corollary}[The floor, the ceiling and the leak]\label{cor:pinned}
			Let $\lambda_E$ be the share of event pairs whose order the event sentences
			themselves entail, and $\lambda_G$ the share of told-order pairs fixed by the
			transitive closure of the graph's
			\textnormal{\texttt{REVEALS}} edges and question pairs. Assume throughout that the solver obeys every entailed pair, so that
			entailment and pinning coincide. Entailment alone does not make a pair
			pinned in the sense of Definition~\ref{def:pinned}.
			\begin{enumerate}\itemsep1pt\parskip0pt
				\item[(a)] A solver that reads no telling and is otherwise uninformed has
				$\mathbb{E}[\tau_{\mathrm{rec}}]=\lambda_E$, and no telling can add more
				than $1-\lambda_E$ to it.
				\item[(b)] For any solver that keeps the entailed pairs and is no worse than
				chance on the rest, $\mathbb{E}[\mathrm{RCG}]\le 1-\lambda_E$ with
				$\lambda_E$ read off the prefix, and so is any difference of two
				reconstruction scores.
				\item[(c)] A generator with pair accuracy $\alpha$ against the gold telling
				gains exactly $2\lambda_G(1-\alpha)\le 2\lambda_G$ in expected agreement
				when the graph's pinned pairs are handed to it and nothing else changes. Supplying the told
				order directly sets $\lambda=1$.
			\end{enumerate}
		\end{corollary}

		\begin{remark}
			(a) reads the floor of \S\ref{sec:floor} as the entailed share. Here $\lambda_E$ covers anything that pins a pair without the telling,
			whether entailment between the sentences or, for a recognizable work,
			prior knowledge of it, and the corollary cannot tell the two apart. A no-story
			solver reaches \DedTauFloor{}, so the telling can add at most
			\TellingWorthMax{}, and it adds \TellingWorth{}. Removing every date moved the
			floor by \DedDelta{} (Table~\ref{tab:dedate}), so dates carry almost no pins,
			and the exact-match floor of \DedExactFloorPct\% is the $1/w!$ term. A solver
			above chance on the unpinned pairs makes the entailed share smaller than the
			floor and the headroom larger, so the floor bounds $\lambda_E$ from above. (b)
			is why \S\ref{sec:worth} declines to interpret RCG, whose ceiling of
			\RcgCeiling{} is computed exactly as $1-\tau_{\mathrm{rec}}(T_{<m})$. (c) with
			$\lambda_G=\PinnedPct\%$ caps the graph's free agreement with the gold telling
			at \GraphLeakMax{} $\tau$. That cap is on $\tau$ alone. The pinned pairs are
			the ones a reveal turns on, so their weight in a reader's preference need not
			be their weight in $\tau$, which is why \S\ref{sec:graph} controls with a
			condition that supplies the order rather than with the share.
			Two limits apply when reading the corollary empirically. First, as noted
			above, a measured no-telling score bounds $\lambda_E$ from above but does
			not identify it, because a solver above chance on the unpinned pairs
			reaches the same score with a smaller entailed share. Second, (c) holds
			only while accuracy on the unpinned pairs is unchanged, so a graph that
			also changes how a model orders everything else can move agreement by more
			or less than $2\lambda_G(1-\alpha)$.
		\end{remark}
	}
	\section{Supporting tables and figures}
	\label{sec:tables}

	\begin{table*}[t]
		\centering\footnotesize
		\setlength{\tabcolsep}{5pt}
		\newcommand{\rqcell}[2]{\parbox[t]{#1\textwidth}{\raggedright #2\strut}}
		\fitfull{%
			\begin{tabular}{l l l}
				\toprule
				& \textbf{What we found} & \textbf{What to be careful about} \\
				\midrule
				RQ1 & \rqcell{0.40}{Every model tells a story more chronologically than its
				published telling; cued for a shape, band hit is \PoolBandHitPct\%, and
				the two misses are the literal solutions of rules stated on events.}
				& \rqcell{0.43}{The reference works were chosen for nonlinearity, both free tasks list
				events chronologically, and no prompt marks scenes. Band hit is negatively
				associated with reader preference, partly through condition
				(\S\ref{sec:bandvalid}).} \\
				RQ2 & \rqcell{0.40}{Reconstruction reaches
				$\tau_{\mathrm{rec}} = \PoolTauRec$; naming the shape reaches
				\PoolClassifyPct\%, with \textsc{analepsis} a sink.}
				& \rqcell{0.43}{\DedTauShare\% of that $\tau$ is reachable with the story
				deleted (by entailment or prior knowledge of the work), and readers
				reach only \HChronTau{}: report the floor, prefer exact match.} \\
				\keyrow
				RQ3 & \rqcell{0.40}{The graph raises compliance from \GZeroBandPct\% to
				\GOneBandPct\%, almost all in one model, and readers prefer it to a
				supplied told order \HGOrderBetterPct\% of the time.}
				& \rqcell{0.43}{The judge saw no effect, and the two prompts differ in more than the
				graph. The edge over an event list is a length effect, and readers prefer
				every one-layer ablation to the full graph.} \\
				RQ4 & \rqcell{0.40}{Cued tellings lose to a model's own free telling
				under the judge (\ProseFreeWins--\ProseCuedWins) and readers
				(\HCuedBetterInvPct\%), even for \textsc{chronicle}. As scored, the
				reveal comes after fewer than three events in \EarlyGZeroPct\% of stories.}
				& \rqcell{0.43}{Two scoring cases inflate the early rate (\EarlyCleanGZeroPct\% without
				them). Retrospective coherence gain is capped
				near \RcgCeiling{} by the same floor.} \\
				\keyrow
				RQ5 & \rqcell{0.40}{The judge matches \HNReaders{} readers on the one
				large effect, and readers overturn both of its ties
				($\kappa = \HJHGAblKappa$ on layer removal).}
				& \rqcell{0.43}{Item-level agreement does not track effect size;
				fourteen stories, two generators in the layer-removal study, no comparison
				against human-written prose.} \\
				\bottomrule
		\end{tabular}}
		\caption{The paper at a glance. The reader study revised the
			\cswatch{blav!50}{shaded} rows.}
		\label{tab:rqsummary}
	\end{table*}

	\begin{figure}[h]
		\centering
		\includegraphics[width=\columnwidth]{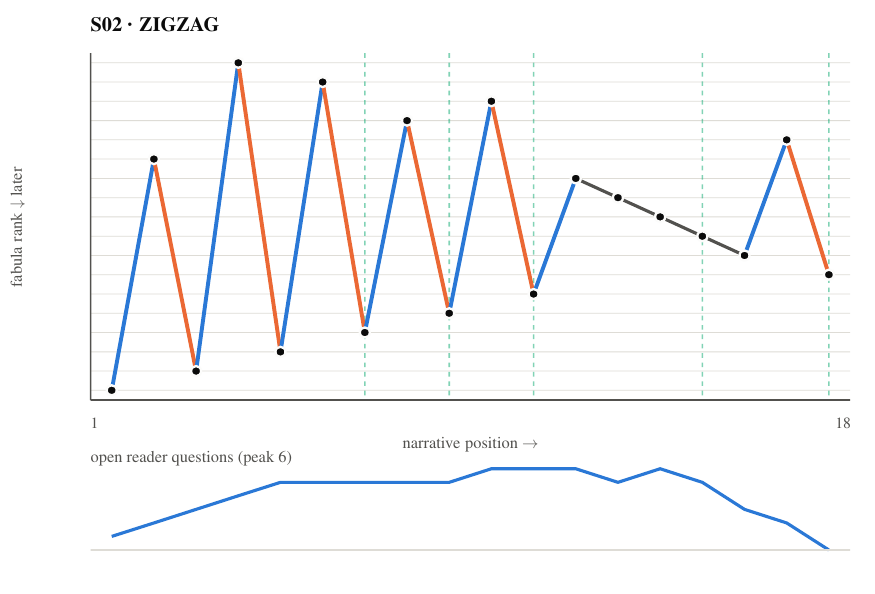}
		\caption{Story curve of one annotated work. Top: told position
			across, chronological rank up. Color and slope both encode jump direction,
			and dashed rules mark where a withheld explanation lands. Bottom: reader
			questions left open after each told position. This work opens six and
			closes all of them.}
		\label{fig:curve}
	\end{figure}

	\begin{table}[h]
		\centering\footnotesize
		\setlength{\tabcolsep}{3pt}
		\fitcol{\begin{tabular}{@{}lccccc@{}}
\toprule
\textbf{Requested}     &     \textbf{band}     &     \textbf{model}     &     \textbf{human}     &     \textbf{agree}     &     \textbf{\textit{exact}} \\
\textbf{structure}     &     \textbf{hit}     &     $\tau$     &     $\tau$     &     \textbf{gold}     &     \textbf{\textit{form.}} \\
\midrule
Zigzag & 0.826 & 0.407 & 0.423 & 0.396 & --- \\
Prolepsis & 0.721 & 0.671 & 0.755 & 0.635 & --- \\
Analepsis & 0.684 & 0.637 & 0.696 & 0.552 & --- \\
Syllepsis & 0.507 & 0.516 & 0.526 & 0.355 & --- \\
Achrony & 0.300 & 0.171 & 0.363 & 0.220 & --- \\
Chronicle & 0.053 & 0.957 & 0.912 & 0.878 & \textit{0.889} \\
Retrograde & 0.011 & $-0.973$ & $-0.781$ & 0.764 & \textit{0.936} \\
\bottomrule
\end{tabular}
}
		\caption{\Task{1} by requested shape, pooled over models
			(\S\ref{sec:cued}). The last column is the share of exact identities
			(\textsc{chronicle}) or reversals (\textsc{retrograde}), which make up the
			two near-zero band-hit rows.}
		\label{tab:category}
	\end{table}

	\FigPlane\FigGraphArm\FigHumanJudge
	\subsection{Where a reversal happens}
\label{sec:scenes}

\S\ref{sec:cued} says that published reverse chronology runs backwards between
scenes while playing forwards inside them, and that a model told to run
backwards reverses everything. That is a claim about pairs of events, so it can
be counted. For every pair in a work we ask whether the telling presents the
two in reverse chronological order, and we split the pairs by whether they sit
inside one annotated scene or across two.

The two populations separate on the within-scene axis alone
(Figure~\ref{fig:scenes}). Across the \SceneNWorks{} \textsc{retrograde} works,
published tellings reverse \SceneHumanBetween\% of between-scene pairs and
\SceneHumanWithin\% of within-scene pairs: the scene is the unit that gets
reversed, and nothing inside it is disturbed. The \SceneNCued{} cued model
outputs for the same works reverse \SceneCuedBetween\% of between-scene pairs,
which matches the human rate, and \SceneCuedWithin\% of within-scene pairs,
which does not. \SceneCuedAllRev\% of them reverse every pair in the work.

This is the mechanism behind the band-hit figure rather than a second result.
A model asked for reverse chronology does the between-scene half correctly and
then keeps going, and the extra work it does is precisely the work a published
author declines to do. The distinction is invisible to $\tau$, which counts all
pairs alike, and it is why \S\ref{sec:cued} reports the exact-permutation rate
beside the band.

\begin{figure}[t]
\centering
\includegraphics[width=\columnwidth]{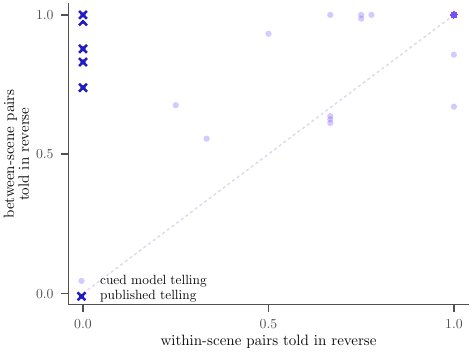}
\caption{Where a reversal happens, for the \SceneNWorks{}
\textsc{retrograde} works. Each point is one telling: the share of its
within-scene event pairs told in reverse against the share of its
between-scene pairs. Published tellings sit on the left wall, reversing scenes
and leaving their contents alone. Cued model outputs pile into the top right
corner, reversing both. The two populations differ on the horizontal axis
only.}
\label{fig:scenes}
\end{figure}

	\begin{table}
		[h]
		\centering\footnotesize
		\setlength{\tabcolsep}{2.5pt}
		\fitcol{\begin{tabular}{@{}l rr rr rr rr@{}}
\toprule
 &  \multicolumn{2}{c}{\textbf{instructed}}  &  \multicolumn{2}{c}{\textbf{graph \textit{vs.}}}  & 
\multicolumn{2}{c}{\textbf{graph \textit{vs.}}}  &  \multicolumn{2}{c}{\textbf{graph \textit{vs.}}} \\
 &  \multicolumn{2}{c}{\textbf{\textit{vs.} free}}  &  \multicolumn{2}{c}{\textbf{told order}}  & 
\multicolumn{2}{c}{\textbf{event list}}  &  \multicolumn{2}{c}{\textbf{layer removed}} \\
\cmidrule(lr){2-3}\cmidrule(lr){4-5}\cmidrule(lr){6-7}\cmidrule(l){8-9}
\textbf{Requested structure}  &  \textbf{rate}  &  $n$  &  \textbf{rate}  &  $n$  &  \textbf{rate}  &  $n$  &  \textbf{rate}  &  $n$ \\
\midrule
\textsc{zigzag} & \negcell{0.12} & 33 & \poscell{0.96} & 23 & 0.65 & 17 & 0.43 & 37 \\
\textsc{achrony} & \negcell{0.05} & 40 & 0.73 & 37 & 0.49 & 39 & 0.28 & 36 \\
\textsc{prolepsis} & \negcell{0.15} & 41 & 0.63 & 35 & 0.72 & 29 & 0.26 & 39 \\
\textsc{syllepsis} & 0.21 & 33 & \poscell{0.88} & 34 & 0.90 & 40 & 0.26 & 38 \\
\textsc{analepsis} & 0.34 & 38 & 0.74 & 47 & 0.48 & 42 & 0.16 & 32 \\
\textsc{chronicle} & 0.32 & 37 & 0.65 & 31 & 0.54 & 41 & 0.32 & 31 \\
\textsc{retrograde} & 0.48 & 42 & 0.57 & 44 & 0.49 & 51 & 0.39 & 46 \\
\bottomrule
\end{tabular}
}
		\caption{Reader preference by requested structure. \emph{rate} is
			the share of decisive judgments favoring the first-named condition. The
			instructed telling loses most on \textsc{zigzag}, \textsc{achrony} and
			\textsc{prolepsis}, loses on \textsc{chronicle} too, and draws level on
			\textsc{retrograde}. The graph's advantage over the told order is largest
			where the shape is least mechanical, and holds on \textsc{chronicle}.}
		\label{tab:humanbycat}
	\end{table}

	Table~\ref{tab:humanbycat} breaks the four reader comparisons down by requested
	structure. Two patterns stand out, and both match
	\S\ref{sec:cued}. The instructed telling loses most heavily on
	\textsc{achrony} and \textsc{zigzag} and not at all on \textsc{retrograde}:
	where the instruction reduces to a mechanical rule, a model that applies
	it literally still produces something readers find no worse than
	its own chronological telling, because a clean reversal is at least
	navigable. And the graph's advantage over a supplied told order is largest on
	\textsc{zigzag} and \textsc{syllepsis}, the two shapes whose definition turns on
	something other than a position in time. Those are exactly the shapes where a
	permutation alone underspecifies the telling and the graph's typed content has
	something left to say. Per-category counts are small, so we read the ordering
	and not the individual rates. Two rows qualify \S\ref{sec:prosepref} and \S\ref{sec:humanpref}. The
	instructed telling also loses on \textsc{chronicle} (\HByCatCuedChron{},
	$n = \HByCatCuedChronN$, exact binomial $p = \HByCatCuedChronP$ before
	correction for the 28 cells). There the instruction asks for chronological
	order, so the cued story should differ little in order from the free one,
	and part of the cost comes from the instruction rather than from
	nonlinearity. The graph also still beats the told order on
	\textsc{chronicle} (\HByCatOrderChron{}, $n = \HByCatOrderChronN$), where
	both conditions are essentially chronological and there is no order to
	justify. This points to the graph's content rather than to a reason for
	the order.

	\begin{table}
		[h]
		\centering\footnotesize
		\setlength{\tabcolsep}{3pt}
		\fitcol{\begin{tabular}{llcc}
\toprule
\textbf{Event list}     &     \textbf{Telling}     &     $\tau$     &     \textbf{exact} \\
\midrule
as written & shuffled, none & 0.849 & 0.085 \\
as written & gold told order & 0.884 & 0.202 \\
de-dated & shuffled, none & 0.844 & 0.096 \\
de-dated & gold told order & 0.878 & 0.181 \\
\bottomrule
\end{tabular}
}
		\caption{The de-dating $2\times2$ (\S\ref{sec:floor}), pooled over three strong
		models. Removing every absolute time anchor moves the no-story floor by
		only \DedDelta{} $\tau$, so the floor does not come from explicit dates.
		Prior knowledge of the works is not ruled out.}
		\label{tab:dedate}
	\end{table}

	\section{The reader study in full}
	\label{sec:human-extra}

	\S\ref{sec:human} reports the four preference comparisons and the two
	instrument checks the paper's conclusions turn on. This appendix holds the
	rest of the reader study: how far readers agree with each other and with
	themselves, how their chronology reconstruction compares with the models',
	three checks on the design, and what the single close readings add. Every
	number here is computed from the recorded responses.

	\begin{table}[t]
		\centering\footnotesize
		\setlength{\tabcolsep}{4pt}
		\fitcol{\begin{tabular}{@{}l rr cc@{}}
\toprule
 &  \multicolumn{2}{c}{\textbf{full graph}}  &   &  \\
 &  \textbf{pref.}  &  \textbf{of}  &  \textbf{rate}  &  $p$ \\
\midrule
\multicolumn{5}{@{}l}{\textit{by layer removed}} \\
\quad $-$\,\texttt{CAUSES}/\texttt{FORESHADOWS} & 15 & 54 & 0.278 & $0.0015$ \\
\quad $-$\,\texttt{pre}/\texttt{post} & 16 & 49 & 0.327 & $0.0213$ \\
\quad $-$\,\texttt{learns} & 13 & 48 & 0.271 & $0.0021$ \\
\quad $-$\,questions/\texttt{REVEALS} & 17 & 51 & 0.333 & $0.0241$ \\
\quad $-$\,participants/location & 18 & 57 & 0.316 & $0.0075$ \\
\midrule
\multicolumn{5}{@{}l}{\textit{by generator, pooled over the five layers}} \\
\quad Claude Sonnet 5 & 26 & 147 & \negcell{0.177} & $<0.0001$ \\
\quad Gemini 3.6 Flash & 53 & 112 & \warncell{0.473} & $0.64$ \\
\midrule
pooled & 79 & 259 & \negcell{0.305} & $<0.0001$ \\
\bottomrule
\end{tabular}
}
		\caption{Reader verdicts on removing one graph layer. \hd{pref.} counts decisive
		judgments for the full graph, so a rate below 0.5 favors the reduced
		graph. Every layer points the same way. The split by generator (\cNeg{}
		large effect, \cWarn{} none) limits the result.}
		\label{tab:humanabl}
	\end{table}
	\subsection{Do readers agree with each other?}
	\label{sec:humanrel}

	\begin{table}
		[t]
		\centering\footnotesize
		\setlength{\tabcolsep}{4pt}
		\fitcol{\begin{tabular}{@{}l ccc@{}}
\toprule
 &  \textbf{better}  &  \textbf{harder to}  &  \textbf{better} \\
 &  \textbf{story}  &  \textbf{follow}  &  \textbf{ending} \\
\midrule
\multicolumn{4}{@{}l}{\textit{Between readers}: Krippendorff's $\alpha$, 311 items} \\
\quad all three levels & \keycell{0.415} & 0.112 & 0.073 \\
\quad 95\% CI & {\tiny [0.359, 0.471]} & {\tiny [0.065, 0.157]} & {\tiny [0.031, 0.113]} \\
\quad decisive only & 0.595 & 0.213 & 0.131 \\
\midrule
\multicolumn{4}{@{}l}{\textit{Within a reader}: hidden repeat, $n = 294$ each} \\
\quad agrees with own answer & \keycell{0.759} & 0.619 & 0.571 \\
\quad chance at that mix & 0.395 & 0.376 & 0.361 \\
\quad Cohen's $\kappa$ & 0.601 & 0.390 & 0.331 \\
\midrule
\quad tie rate & 0.12 & 0.22 & 0.18 \\
\bottomrule
\end{tabular}
}
		\caption{Rater reliability. Between readers: nominal Krippendorff's
			$\alpha$ with unequal coders per unit \citep{krippendorff2018content}, on a
			\{favored side, other side, tie\} recoding. Within a reader: agreement on a
			hidden repeat with sides swapped. Low agreement on the two right-hand
			questions is why the paper leads with \emph{better story}.}
		\label{tab:humanrel}
	\end{table}

	Which is the better story is a question readers answer consistently enough to
	build on. Which has the better ending is not.

	\paragraph{Between readers.}
	On \emph{better story}, $\alpha = \HAlphaBetter{}$
	(\HAlphaBetterCI{} over \HAlphaBetterN{} items), rising to
	\HAlphaBetterDec{} once ties are set aside. On \emph{harder to follow} it is
	\HAlphaConf{} and on \emph{better ending} \HAlphaEnd{}
	(Table~\ref{tab:humanrel}). These are low numbers by the standards of a
	labeling task and ordinary ones for aesthetic judgment; what matters is the
	ordering, which is stable across all eight comparison kinds
	(\HAlphaKindLo--\HAlphaKindHi{} on \emph{better story}).

	\paragraph{Within a reader.}
	Once per sitting a comparison the reader had already judged came back with the
	sides swapped and no announcement. Someone answering by position rather than by
	reading scores near zero here. Readers reproduced their own earlier verdict
	\HRepeatBetterPct\% of the time on \emph{better story}
	($\kappa = \HRepeatBetterKappa$), \HRepeatConfPct\% on \emph{harder to follow}
	and \HRepeatEndPct\% on \emph{better ending}, against
	\HRepeatBetterChancePct\%, \HRepeatConfChancePct\% and
	\HRepeatEndChancePct\% expected from answering each question at the same mix
	at random (Table~\ref{tab:humanrel}). Restricted to repeats where both answers were decisive, agreement on
	\emph{better story} is \HRepeatBetterDecPct\%.

	\paragraph{Two more instrument checks.}
	Side choice leans very slightly right (\HSideA{} left, \HSideB{} right,
	\HSideTie{} ties of \HSideN{}; $p \HSidePrel$), which is small and, because sides
	are randomized per reader, favors no condition. On the shape-comprehension
	items readers scored \HLitOverall\% (\HLitCorrect{} of \HLitN{}) against a
	chance baseline of \HLitChance\%: \HLitOpen\% on whether the telling opens at
	the earliest event, \HLitEnd\% on whether it ends at the latest, and
	\HLitJumps\% on how many times it jumps backwards. That last figure is a useful
	calibration point. Counting backward jumps is the shape judgment to trust least,
	in a reader and, we would guess, in a model.

	\paragraph{What reliability buys.}
	Within-reader consistency is a reference point, and it is not a ceiling
	for a judge scored against a majority. It is \HRepeatBetterPct\% over all
	repeats and \HRepeatBetterDecPct\% where both answers were decisive, and
	\S\ref{sec:humanjudge} reports the judge's agreement on decisive items
	next to the latter figure. The right yardstick for agreement with a
	majority is how often one reader agrees with the majority of the others on
	the same items. We have not yet computed this leave-one-reader-out
	baseline (\S\ref{sec:humanproto}).

	\subsection{Reconstruction, done by a person}
	\label{sec:humanfloor}

	\S\ref{sec:floor} showed that \Task{3} keeps \DedTauShare\% of its score with the
	story deleted, and concluded that events imply their own order. The close
	readings let us put a human number beside it. After reading one story properly,
	readers dragged its shuffled events into the order they believed things
	happened, scored against the same gold chronology the models are scored on.

	They reach $\tau = \HChronTau$ (\HChronCI{} clustered on readers,
	\HChronSCI{} on stories) and recover the chronology exactly
	\HChronExact\% of the time, over \HChronN{} readings. The same
	\HNStories{} stories give the models $\tau_{\mathrm{rec}} = \HModelTauFourteen$,
	and \HModelsAboveHuman{} of \HModelsN{} models sit above the reader mean. The
	no-story floor on those stories is \HFloorFourteen{}, which is to say that a
	strong solver handed the events with the telling deleted beats a person who read
	the telling.

	This does not show that models read nonlinear narrative better than
	people. Readers worked from prose and event cards and models from the
	segmented telling, so the inputs are not comparable, and for the same
	reason it does not license a human ceiling. It does sharpen the reading of
	\S\ref{sec:floor}. Readers had strictly more information than the no-story
	solver, since they also had the telling, yet ordered the same events less
	well ($\tau = \HChronTau$ against a floor of \HFloorFourteen{}). Either
	the drag-and-drop task is noisy (12--20 items, \HChronExact\% exact), or
	the model floor draws on knowledge readers lack, such as familiarity with
	a recognizable work. Entailment alone, which careful readers
	could also use, does not explain the gap.

	One further asymmetry is worth naming. What is hard for readers is not what is
	hard for models. Readers do worst on \HChronHardest{}, the two shapes
	whose ordering rule is hardest to infer from the telling, one because it
	alternates and one because it has no rule, where they reach
	$\tau = \HChronHardestTau$ and below,
	and best on \HChronEasiest{} at $\HChronEasiestTau$. Across the seven categories
	the reader and model difficulty profiles show no rank correlation
	($\rho = \HCatRho$, $p \HCatRhoPrel$). With seven points that is
	weak evidence, but it is the right shape of check to run before treating a
	model's category-level difficulty as a proxy for anybody's.

	\subsection{Three checks on the design itself}
	\label{sec:humanchecks}

	\begin{table}
		[t]
		\centering\footnotesize
		\setlength{\tabcolsep}{3.5pt}
		\fitcol{\begin{tabular}{@{}l@{\hskip 6pt}c@{\hskip 5pt}cccc@{}}
\toprule
 &   &  \textbf{cued}  &  \multicolumn{3}{c}{\textbf{full graph against}} \\
\cmidrule(lr){3-3}\cmidrule(l){4-6}
\textbf{Reader subset}  &  $n$  &  \textbf{\textit{vs.}\ free}  &  \textbf{order}  &  \textbf{list}  &  \textbf{ablation} \\
\midrule
all readers & 71 & 0.242 & 0.721 & 0.595 & 0.305 \\
shape score $\geq$ 60\% & 68 & 0.242 & 0.722 & 0.602 & 0.307 \\
shape score $\geq$ 75\% & 54 & 0.268 & 0.725 & 0.611 & 0.304 \\
median reading time $\geq$ 120\,s & 70 & 0.245 & 0.722 & 0.590 & 0.304 \\
reads $\geq$ 1 book/month & 43 & 0.299 & 0.731 & 0.628 & 0.255 \\
has writing training & 40 & 0.309 & 0.728 & 0.642 & 0.274 \\
\bottomrule
\end{tabular}
}
		\caption{Headline rates under stricter reader subsets. \emph{order},
			\emph{list} and \emph{ablation} compare the full graph against the told
			order, the event list and the graph with one layer removed. Cells are the
			share of decisive judgments favoring the first-named condition, and $n$
			counts readers. Nothing moves by more than about \HSubsetMaxMovePts{} points,
			including the comparison that length adjustment removes.}
		\label{tab:humansubsets}
	\end{table}

	\paragraph{Did familiarity with a story help?}
	Because only \HNStories{} works have prose, a reader doing several sittings met
	the same underlying story again under different conditions, up to
	\HExpoMax{} times, and \HRepeatPairs{} of \HRepeatPairsN{} reader--story
	combinations occur more than once. Items were never repeated, so no single
	judgment is contaminated, but familiarity could plausibly make a later telling
	easier to follow. It does not. Regressing the verdict on exposure count, with
	readers as clusters, gives $\beta = \HExpoBetterBeta$
	($p \HExpoBetterPrel$) on \emph{better story} and $\beta = \HExpoConfBeta$
	($p \HExpoConfPrel$) on \emph{harder to follow}. The ``harder to follow'' rows
	are the ones this would have threatened, and they are flat.

	\paragraph{Was confidence informative?}
	Mean confidence was \HConfMean{} of 5 (sd \HConfSd{}), high throughout,
	including on items where the cohort split near evenly. It is not pure
	noise: across items, distance from an even split tracks mean confidence at
	$r = \HConfConsensusR$ ($p \HConfConsensusPrel$, \HConfConsensusN{} items). So
	confidence carries information about how hard an item was, and none we can
	detect about how good a reader is. Individual confidence should not be read as
	evidence that a difference is real.

	\paragraph{Does any reader subset change the picture?}
	No. Table~\ref{tab:humansubsets} re-runs all four comparisons on readers who
	scored at least 60\% and at least 75\% on the shape items, readers whose median
	reading time was at least two minutes, readers who finish at least a book a
	month, and readers with writing training. The largest movement is about
	\HSubsetMaxMovePts{} points, on cued against free among readers with writing
	training.
	Median reading time was \HRtMedian{}\,s per close-read text
	(interquartile \HRtIqr{}\,s), and no reader had a median under
	\HRtMinReader{}\,s, so nobody was clicking through.

	\subsection{What the close readings add}
	\label{sec:humanclose}

	\begin{table}
		[t]
		\centering\footnotesize
		\setlength{\tabcolsep}{4pt}
		\fitcol{\begin{tabular}{@{}l ccc}
\toprule
 &  {\boldmath\textbf{\GraphCondition{0} events}}  &  {\boldmath\textbf{\GraphCondition{1} full}}  &  \textbf{instructed} \\
 &  \,$+$\,shape  &  \textbf{graph}  &  {\boldmath\textbf{prose (\ProseTask{a})}} \\
\midrule
close readings ($n$) & 69 & 77 & 148 \\
\midrule
\multicolumn{4}{@{}l}{\textit{Psychological Depth Scale}, 1--5} \\
\quad authenticity & 3.61 & 3.99 & 3.22 \\
\quad empathy & 3.33 & 3.74 & 3.11 \\
\quad engagement & 3.68 & 3.91 & 3.42 \\
\quad emotional effect & 3.38 & 3.74 & 3.16 \\
\quad narrative complexity & 3.87 & 4.13 & 3.66 \\
\quad \textbf{composite} & 3.57 & \keycell{3.90} & \negcell{3.31} \\
\midrule
\multicolumn{4}{@{}l}{\textit{Chronology recovered from the prose}} \\
\quad mean $\tau$ against gold & 0.572 & 0.519 & 0.531 \\
\quad exact (\%) & 1.4 & 1.3 & 4.1 \\
\midrule
\multicolumn{4}{@{}l}{\textit{What readers reported}} \\
\quad gold questions answered (\%) & 72 & 76 & 77 \\
\quad left open (\%) & 14 & 10 & 12 \\
\quad consistency flags per reader & 0.70 & 0.73 & \negcell{0.99} \\
\quad guessed machine-written (\%) & 38 & 36 & \negcell{55} \\
\quad median reading time (s) & 179 & 188 & 217 \\
\bottomrule
\end{tabular}
}
		\caption{The \HNSittings{} close readings (instrument in
			\S\ref{sec:instrument}). \GraphCondition{0} and instructed prose share one
			prompt (\ProseTask{a}) and differ in generators: the \GraphNModels{} graph-study
			models against the \ProseNModels{} prose models, with only GLM-5.2 in
			both. Instructed prose is rated lowest, draws the most consistency
			complaints, takes longest to read, and is most often guessed
			machine-written.}
		\label{tab:humanclose}
	\end{table}

	\S\ref{sec:bandvalid} makes two claims that rest on these readings. The first
	is that band hit does not predict rated depth: regressed on condition and
	length together, it contributes $\beta = \HBandDepthBeta$
	($p \HBandDepthPrel$), which is nothing. The second is that reader-rated depth
	and nonlinearity are unrelated across these outputs, composite PDS
	correlating with $\tau$ at $r = \HDepthTauR$ and \HTopTwentyNonChrono{} of the
	twenty highest-rated readings not being effectively chronological. A reader
	who rates a text deeply does not thereby rate it linear, which is why
	\S\ref{sec:humanpref} reads as a verdict on instructed nonlinearity
	rather than on nonlinearity.

	The pairwise blocks say which of two texts a reader preferred. The close
	readings say what reading one was like (Table~\ref{tab:humanclose}), and they
	converge on the same ordering: full graph \HPdsGOne{}, events and shape
	\HPdsGZero{}, instructed prose \HPdsCued{} on composite PDS
	($\HPdsGOne$ \textit{vs.} $\HPdsCued$ at $p \HPdsGOneVsCuedPrel$,
	$d = \HPdsGOneVsCuedD$), though the last two share a prompt and differ
	in generators.
	The judge's PDS on the same texts runs above readers by $\HPdsGap$
	overall, $\HPdsGapComplex$ on narrative complexity and $\HPdsGapEngage$ on
	engagement.

	The instructed texts separate on three measures that have nothing to do with
	preference. Readers flagged \HViolCued{} consistency problems per reading
	against \HViolGZero{} and \HViolGOne{} elsewhere
	($p \HViolCuedVsGZeroPrel$ \textit{vs.} \GraphCondition{0}), guessed them
	machine-written \HMachineCuedPct\% of the time against
	\HMachineGZeroPct\% and \HMachineGOnePct\% ($p \HMachinePrel$), and spent
	\HRtCued{}\,s on them against \HRtGZero{} and \HRtGOne{}\,s. Because \GraphCondition{0} uses the same prompt as the instructed prose, these
	gaps cannot be credited to the instruction. They separate the
	\ProseNModels{} prose models from the \GraphNModels{} graph-study models,
	or two samples from one prompt. They are consistent with the preference
	split of \S\ref{sec:humanpref} but do not explain it.

	One measure does not separate the conditions. The checklist item for a
	prematurely disclosed reveal sits between \HSpoilerGZeroPct\% and
	\HSpoilerCuedPct\% across the three conditions with no detectable difference,
	where \S\ref{sec:reveal} found a suggestive automated gap. The two are not in
	conflict: the close readings cover \GraphCondition{0}, \GraphCondition{1} and the instructed
	condition, and the automated measure finds no difference between
	\GraphCondition{0} and \GraphCondition{1} either. The contrast it does find is against the supplied told order, which
	no reader
	close-read. Readers neither confirm nor contradict \S\ref{sec:reveal}.

	\begin{table*}[t]
		\centering\small
		\setlength{\tabcolsep}{5pt}
		\fitfull{\begin{tabular}{@{}l rc rc c rcc l}
\toprule
 &  \multicolumn{2}{c}{\textbf{judge}}  &  \multicolumn{2}{c}{\textbf{readers}}  &   & 
\multicolumn{3}{c}{\textbf{item-level concordance}}  &  \\
\cmidrule(lr){2-3}\cmidrule(lr){4-5}\cmidrule(lr){7-9}
\textbf{Comparison}  &  \textbf{split}  &  \textbf{rate}  &  \textbf{split}  &  \textbf{rate}  &  $\Delta$  &  $n$  &  $\kappa$  &  \textbf{95\% CI}  &  \textbf{verdict} \\
\midrule
instructed \textit{vs.}\ left alone & 19--68 & 0.218 & 64--200 & 0.242 & $+0.024$ & 67 & $+0.278$ & $[+0.01, +0.54]$ & \warncell{weak} \\
full graph \textit{vs.}\ told order & 23--22 & 0.511 & 181--70 & 0.721 & $+0.210$ & 35 & $+0.242$ & $[-0.11, +0.56]$ & \warncell{weak} \\
full graph \textit{vs.}\ event list & 31--14 & 0.689 & 154--105 & 0.595 & $-0.094$ & 31 & $+0.625$ & $[+0.28, +0.91]$ & \poscell{agrees} \\
full graph \textit{vs.}\ one layer removed & 55--50 & 0.524 & 79--180 & 0.305 & $-0.219$ & 77 & $-0.030$ & $[-0.23, +0.18]$ & \negcell{unrelated} \\
\bottomrule
\end{tabular}
}
		\caption{The judge against \HNReaders{} readers on the same texts and question.
		\hd{split} and \hd{rate} count decisive verdicts for the first-named
		condition, and $\Delta$ is readers minus judge. $\kappa$ compares the
		judge with the readers' majority on items where both were decisive, and
		the last column grades it: \cPos{} for graph against event list (the
		smallest reader effect), \cWarn{} for instructed against free (aggregate
		rates within \HGapCuedPts{} points) and for graph against told order, and
		\cNeg{} for one layer removed. Item-level agreement does not follow the
		size of the reader effect.}
		\label{tab:judgehuman}
	\end{table*}

	\section{Reporting recommendations}
	\label{sec:recommendations}

	Three of our own headline metrics needed a caveat attached once we had
	checked them. The six practices below are what we would have wanted to know
	before running this study; none costs new annotation. We follow them only in part ourselves. Our floor beside
	Table~\ref{tab:main} comes from three strong solvers rather than from each
	solver we score (1), simultaneous events are not yet an equivalence class
	(2), and we adjusted for length rather than matching on it (5).
	Limitations discusses each.

	\begin{enumerate}\itemsep1pt\parskip0pt
		\item \textbf{Publish the no-story floor beside any reconstruction score,} at
		one extra call per item. Run the solver once on the shuffled event list
		with the text deleted and print that number next to the one with the text:
		\DedTauShare\% of ours is reachable with no story at all
		(\S\ref{sec:floor}). Print a human number where one exists; ours is
		$\tau = \HChronTau$, below the floor.
		\item \textbf{Make exact match the primary reconstruction score.} Ranking by
		$\tau$ can invert the order: our weakest model gains the most $\tau$ from
		the telling and almost no exact recovery. Treat simultaneous events as an
		equivalence class, or no model will ever score on a \textsc{chronicle}.
		\item \textbf{Add a condition that supplies the target order directly.} A
		scaffold encoding discourse relations leaks the ordering it is meant to
		help produce, and ours fixes \PinnedPct\% of told-order pairs, so a gain
		over a bare baseline cannot be separated from the model having been handed
		part of the answer. That condition is also where our judge and our readers
		disagreed most, so run it even when it looks flat (\S\ref{sec:humanjudge}).
		Pair it with a condition that adds the scaffold to the supplied order,
		so that ordering information is held fixed (\S\ref{sec:humanpref}).
		\item \textbf{Score structural compliance for magnitude, not direction,} and
		print agreement with a reference telling beside any resemblance metric.
		Two of our seven shapes miss the band almost entirely because the outputs
		are the literal solution of the rule as stated (\S\ref{sec:cued}), and a
		check on direction alone scores them as successes. State a shape at the
		level where practice applies it (scenes, simultaneous events) before
		scoring its magnitude. Do not treat resemblance to published practice as a
		quality proxy, since on our texts it is negatively associated with reader
		preference (\S\ref{sec:bandvalid}).
		Before reporting either, run the scripted tellings that each shape names
		(the identity, the reversal and the scene-level reversal) through the same
		scorer. Ours show that the step models miss on \textsc{retrograde} needs
		no model, and that the answer the \textsc{chronicle} prompt asks for
		scores zero (\S\ref{sec:grid-notes}). A structural metric that no
		script can satisfy measures the specification rather than the model.
		\item \textbf{Adjust any preference comparison for length, or match on it.}
		Conditions that differ in what the model is given also differ in how much
		it writes. One regression with the word-count difference as a covariate
		\citep{dubois2024length} removed one of our three reader results and left
		the other two untouched (\S\ref{sec:humanlength}). Report the rate at which
		the longer text was chosen whichever way the adjustment comes out.
		\item \textbf{Do not read ``no difference'' from a judge as no difference.}
		Ours matched readers within \HGapCuedPts{} points on the one comparison
		readers split three to one, and returned ties on two comparisons that
		readers split about 70 to 30 ($\kappa = \HJHGAblKappa$ on one of them).
		Nothing in its output marked the difference, and item-level agreement did
		not follow the size of the reader effect. A few hundred reader judgments
		on the same items is the cheapest way to check a judge's tie.
	\end{enumerate}

	\section{What the corpus build checks}
	\label{sec:validation}

	As delineated in \S\ref{sec:corpus}, the annotation is machine-checkable and that any work
	failing a check is refused rather than emitted. Each node carries a
	subject--verb--object core, elaborating content, participants and location,
	world-state pre- and post-conditions, what the event teaches each character and
	the reader, the questions it raises or resolves, and for characters an
	archetype role from Vogler's eight \citeyearpar{vogler2007writers}, inspired by Jungian archetypes \cite{jung1969archetypesessay}. These are
	the checks.

	The anonymized twin must match its source in event count and in permutation
	exactly. It must contain no proper noun from the source anywhere, including
	inside every authored string in the graph, which is where most of the authored
	prose sits, and it must share no event sentence verbatim with the source.
	Replaying the post-conditions in chronological order must satisfy every
	precondition, so an event that would be impossible at its own story time fails
	the build. Scenes must stay contiguous in both orders. Each edge type carries
	its own constraint: \texttt{CAUSES} and \texttt{FORESHADOWS} must run forward
	in chronology, \texttt{REVEALS} must run forward in chronology and backward in
	the telling and be backed by a matching question pair, and \texttt{SHARES}
	must name what its endpoints share. Every reader question must be raised
	exactly once and resolved exactly once. Finally, a separate script re-derives
	the permutations, the metrics and the graph anchors from the prose columns
	alone, so a transcription error in a stored metric cannot survive.

	What none of this reaches is misreading. The checks cannot say whether a
	second annotator would have segmented a work into the same events or assigned
	the same structural label, and the three labels that turn on what scenes are
	about, \textsc{syllepsis}, \textsc{achrony} and the boundary between
	\textsc{analepsis} and \textsc{prolepsis}, are the most open to disagreement.

	\section{Notes on the benchmark grid}
	\label{sec:grid-notes}

	The notes below answer no research question but qualify every number
	in Table~\ref{tab:main}.

	\paragraph{The effort ablation is too small to rule out an effect.}
	Re-running \EffortNModels{} frontier models at high effort on the
	\ProseNStories{} prose works (\Task{1} and \Task{3} only) moves band hit by at most
	\EffortMaxBand{} and $\tau_{\mathrm{rec}}$ by at most \EffortMaxTauRec{},
	inconsistently in sign. A shift of \EffortMaxBand{} is larger than the gap
	between the first and the eleventh model of Table~\ref{tab:main}, and on
	\ProseNStories{} works per model we can put no useful interval on it, so the
	ablation rules out neither a gain nor a loss from test-time compute. Every
	headline figure is reported at low effort, for cost and comparability.

	\paragraph{Adjacent ranks are not separable.}
	Resampling the \NWorks{} works (4{,}000 draws) gives 95\% intervals of
	about $\pm\BandHitCIHalf$ on each model's band hit. Every pair of adjacent
	ranks in Table~\ref{tab:main} overlaps, and \BandHitCIOverlap{} of the
	\NModelsOrdering{} intervals overlap the top model's, so the table ranks
	tiers at most. Bolding the model whose mean $\tau$ is closest to the human
	mean is descriptive only, since a mixture of reversals and identities can
	match a mean without any telling resembling the human one.

	\paragraph{How much one work moves a band.}
	Each band is spanned by the extreme gold tellings of its category, so we
	recomputed every band with each work left out of its own category. Pooled
	band hit then ranges from \LowoBandLo{} to \LowoBandHi{}, and
	\LowoMovers{} works move it by more than 0.01. The largest mover is S02,
	whose $\tau$ sets the lower edge of the \textsc{zigzag} band. A held-out
	published work falls inside the band spanned by its peers \HeldOutInBand{}
	of \NWorks{} times (\HeldOutInBandPct\%), which is the cross-validated
	counterpart of the human band hit of 1 that holds by definition.

	\paragraph{What band hit sees.}
	The $(\tau,\mathrm{LIS})$ rectangle cannot see alternation on a beat
	or grouping by a non-temporal principle. \textsc{zigzag} has the highest
	band hit and, of the three shapes that need judgment, the lowest agreement
	with the gold telling (\ZigzagAgree{}), and its band spans the widest range
	of $\tau$, so a high hit there says as much about the band as about zigzags.
	Bands also depend on $n$, since $\mathrm{LIS} \geq 1/n$ with 12--20 events,
	and on category size. Agreement with the gold telling is highest for the two
	rule-shaped categories (\RetroAgree{} and \ChronAgree{}), because a perfect
	reversal of a mostly reversed telling agrees with it closely, so the
	exact-permutation rate carries the over-compliance claim of \S\ref{sec:cued}.

	\paragraph{Structural baselines, with no model.}
	We score four tellings that any script can produce exactly as we score a
	model's. The identity hits \BaseIdChronHit{} of the ten \textsc{chronicle}
	bands and the exact reversal \BaseRevRetroHit{} of the ten
	\textsc{retrograde} bands, because those bands stop at $\tau =
	\ChronBandTauHi$ and $\tau = \RetroBandTauLo$, although the two prompts
	explicitly accept chronological and fully reversed order
	(\S\ref{sec:prompts}). Both answers are therefore legal responses
	that band hit scores as misses. The scene reversal of
	Prop.~\ref{prop:rule}(ii), applied to each work's own gold scene
	partition, hits the \textsc{retrograde} band \BaseSceneRetroHit{} times in
	ten (\BaseScenePooledPct\% over all \NWorks{} works, since it fits no
	other shape's band), against \RetroBandPct\% for the \NModelsOrdering{}
	models. The step the models miss on this shape is therefore available
	without a language model, given the scene boundaries that no prompt
	supplies. Shuffling the same scenes at random, which aims at no shape,
	hits \BaseBlockPct\% pooled. That is a caution about pooled band hit
	rather than a result about scripts.

	\paragraph{How the random baseline was drawn.}
	The published \RandomBandHitPct\% is one uniform permutation per work at a
	fixed seed, so it counts 7 of \NWorks{} works and carries the sampling
	noise of a single draw. Repeating the draw \RandomDraws{} times per work
	gives \RandomMCPct\%, and most of that comes from two shapes,
	\RandomMCHi{} for \textsc{achrony} and 0.125 for \textsc{syllepsis},
	against \RandomMCLo{} for \textsc{chronicle} and \textsc{prolepsis}. A
	single pooled chance figure hides a range that wide, so per-shape band hit
	should be read against the per-shape chance rate rather than against
	\RandomBandHitPct\%.

	\paragraph{Exact outputs by shape.}
	Of the \CuedNValid{} valid cued tellings, \CuedRevN{} are exact
	reversals and \CuedIdN{} exact identities: \RetroRevN{} of \RetroValidN{}
	\textsc{retrograde} requests and \ChronIdN{} of \ChronValidN{}
	\textsc{chronicle} requests, against \OtherRevN{} reversals and \OtherIdN{}
	identities across the \OtherValidN{} requests for the other five shapes.
	Without the two rule-shaped categories, pooled band hit is
	\BandHitNoRulePct\%.

	\paragraph{The floor is a strong-solver floor.}
	The no-story floor of \S\ref{sec:floor} comes from three strong
	solvers. Recommendation~1 asks for the floor of the solver being scored,
	which we have not run for all \NModelsOrdering{} models, so ``only
	\NBeatFloor{} of \NModels{} clear the floor'' compares weaker models with a
	strong-model floor.
	\paragraph{What the indices carry, and what de-dating rules out.} Neither
	reconstruction condition numbers its events by chronological rank. \Task{3}
	numbers them $1 \ldots n$ by told position, and the no-story condition
	shuffles them and numbers them by presentation position, so in both the
	index a model is asked to order is independent of the answer. De-dating
	removes explicit time anchors and moves the floor by \DedDelta{} $\tau$.
	Each rewrite is checked for an unchanged event count, nonempty and
	distinct descriptions, and retention of detected recurring names.
	It leaves untouched causal phrasing, the choice of which beats became
	events, stereotyped event wording and prior knowledge of a recognizable
	work. The floor is therefore not explained by dates, which is a weaker
	claim than saying it is free of annotation artifacts.

	\paragraph{Refusals are real and asymmetric.}
	The three GPT-5.6 models (Sol, Terra and Luna) refuse
	\RefusalLo--\RefusalHi\% of attempts, while the other \NNonRefusers{},
	OpenAI's open-weights \textit{gpt-oss-120b} among them, refuse nothing, and
	\NRefusedShared{} of the \NWorks{} works are refused by all three even when
	the task is sorting events into chronological order. Those are the
	dark-literary end of the corpus, so this is a benchmark-design problem rather
	than a capability result. It makes the affected means non-comparable unless a
	common subset is used; repeating every comparison on the works all models
	answered leaves the ordering of Table~\ref{tab:main} essentially unchanged.
	Any literary benchmark relying on one vendor will under-sample exactly the
	works whose structure is most interesting. Open-weight models trail on all
	three axes, but roster coverage, missingness and scale all differ, so we draw
	no conclusion from that.

	\section{Retrospective coherence gain}
	\label{sec:rcg}

	\FigReveal

	\begin{table}
		[h]
		\centering\footnotesize
		\setlength{\tabcolsep}{3pt}
		\fitcol{\begin{tabular}{lcccccc}
\toprule
\textbf{Condition}     &     $n$     &     $\tau_{\mathrm{pre}}$     &     $\tau_{\mathrm{post}}$     &     \textbf{RCG}     &     \textbf{guess}     &     \textbf{early} \\
\midrule
\GraphCondition{0} no graph & 44 & 0.903 & 0.967 & 0.064 & 0.773 & \textbf{0.214} \\
\GraphCondition{1} full graph & 43 & 0.884 & 0.918 & 0.034 & 0.674 & \textbf{0.218} \\
\GraphCondition{2} told order & 52 & 0.861 & 0.886 & 0.025 & 0.865 & \textbf{0.071} \\
\textbf{H} human telling & 13 & 0.832 & 0.887 & 0.054 & 0.846 & \textbf{0.071} \\
\bottomrule
\end{tabular}
}
		\caption{Retrospective coherence gain, its paired guess task, and premature
		disclosure as scored (\textbf{early}: share of stories with fewer than
		three events before the reveal). $n$ counts the stories RCG could be
		scored on, while \textbf{early} uses every story (\EarlyGZeroN{},
		\EarlyGOneN{}, \EarlyGTwoN{} and \EarlyHumanN{}), since a story that
		discloses early cannot be scored for RCG. RCG does not separate the
		conditions, and the text qualifies the \textbf{early} rates.}
		\label{tab:rcg}
	\end{table}

	Retrospective coherence gain, which is defined as follows:
	$$\mathrm{RCG} = \tau_{\mathrm{rec}}(T) - \tau_{\mathrm{rec}}(T_{<m})$$ for a
	telling and its prefix ending just before the designated reveal $m$, asks
	whether later material makes earlier material more recoverable. The floor of
	\S\ref{sec:floor} caps it near \RcgCeiling{} by construction, as it caps any
	difference of reconstruction scores (Corollary~\ref{cor:pinned}(b)).

	Table~\ref{tab:rcg} gives every condition. The reconstruction solver throughout
	is Grok~4.5, held out from every family that generated stories here, and
	prefix segmentation uses the aligner validated in \S\ref{sec:alignvalidation}.
	The two terms are not scored on the same events.
	$\tau_{\mathrm{rec}}(T_{<m})$ ranges over the events the prefix narrates
	and $\tau_{\mathrm{rec}}(T)$ over all $n$, so a positive RCG mixes two
	things, a better ordering of the events already seen and the arrival of
	events the prefix did not contain. Restricting both terms to the prefix's
	events would separate them. We have not done so, which gives a second
	reason, besides the ceiling, to read Table~\ref{tab:rcg} as exploratory.
	RCG is paired with a five-option guess task at the cut point, with distractors
	drawn only from events the prefix has not yet narrated, so none can be
	eliminated by memory. Neither number means anything alone. Only the pair
	separates a deliberate, successful withholding from a merely confusing one.

	The definition is sound where it has room to move. \textsc{retrograde}, which
	inverts rather than withholds, has $\mathrm{RCG} = \RetroRcg$ exactly as theory
	predicts while placing its reveal latest of any category (\RetroRevealPos{} of
	the way through the telling). It is the ceiling, not the definition, that makes
	the measure inert on the other six.

	Two secondary observations survive the flat result in \S\ref{sec:worth}.
	The graph is
	the only condition that makes the withheld cause harder to anticipate (guess
	accuracy \GuessGZeroPct\% to \GuessGOnePct\%) without converting that concealment
	into payoff. And supplying the told order gives the most predictable
	condition of all
	(\GuessGTwoPct\%). A model handed the human told order narrates to that order
	rather than withholding along it.

	Our planned fix for the ceiling, de-dating the solver's event list, was
	canceled by \S\ref{sec:floor}, since the ceiling comes from semantic
	entailment rather than from dating. The replacement is to score only the event
	pairs a no-story solver gets wrong across replicates. That is an
	entailment-free pair set, identifiable automatically from runs we already have
	and needing no new annotation.

	\paragraph{Premature disclosure, checked.}
	Paired by generator, story and replicate, \GraphCondition{0} and \GraphCondition{2}
	disagree on \EarlyPairedDisc{} stories, all early under \GraphCondition{0} (exact
	McNemar $p = \EarlyPairedP$). \GraphCondition{1} against \GraphCondition{2} gives the
	same result, and \GraphCondition{0} against \GraphCondition{1} splits
	\EarlyGZeroGOneDisc{} to \EarlyGZeroGOneDisc{}. Two cases inflate the
	rates. S45, a chronicle, has no designated reveal, and the scorer records
	it as early in all four conditions (\EarlyNoRevealN{} records, including
	all of \GraphCondition{2}'s early stories and the human reference's one). In S27,
	the other chronicle, the designated explanation is the third event in
	chronology, so any chronological telling, which is what the
	\textsc{chronicle} instruction asks for, places it after two events. Every
	\GraphCondition{0} and \GraphCondition{1} story of S27 is therefore scored early, and
	neither \GraphCondition{2} nor the human telling is. Without the two works the
	rates are \EarlyCleanGZeroK{} of \EarlyCleanGZeroN{} (\GraphCondition{0}),
	\EarlyCleanGOneK{} of \EarlyCleanGOneN{} (\GraphCondition{1}), 0 of
	\EarlyCleanGTwoN{} (\GraphCondition{2}) and 0 of \EarlyCleanHumanN{} (human), and
	the paired test gives $p = \EarlyCleanP$. On the stories with a designated
	reveal, the gap between \GraphCondition{0} and \GraphCondition{2} holds at every
	threshold from two to six events before the reveal (\GraphCondition{0} from
	\EarlyCurveGZeroLo\% to \EarlyCurveGZeroHi\%, \GraphCondition{2} at most
	\EarlyCurveGTwoHi\%), and the remaining early disclosures spread over five
	shapes, too few for a per-shape reading. The \GraphCondition{2} repair is partly
	by construction, since the designated reveal is the one the human telling
	places latest and \GraphCondition{2} is handed that telling.

	\section{Structural classification}
	\label{sec:confusion}

	\begin{figure}[h]
		\centering
		\includegraphics[width=\columnwidth]{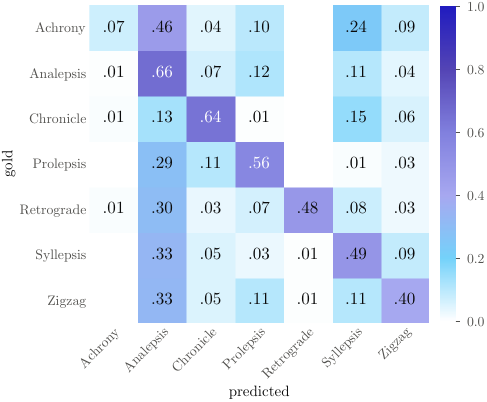}
		\caption{\Task{4} confusion, row-normalized and pooled over
			\NModelsOrdering{} models. \SinkN{} of the six other classes send
			\SinkLo--\SinkHi\% of their mass to \textsc{analepsis}, and
			\textsc{achrony} is recognized \ConfAchronySelfPct\% of the time.}
		\label{fig:confusion}
	\end{figure}

	The classification errors reported in \S\ref{sec:negatives} are given in full
	in Figure~\ref{fig:confusion}. Across models, naming a shape correlates with \Task{1} band hit at $\rho =
	\RhoClassBand$ and with \Task{3} $\tau_{\mathrm{rec}}$ at $\rho = \RhoClassRec$,
	and \Task{1} band hit and \Task{3} $\tau_{\mathrm{rec}}$ correlate at $\rho =
	\RhoBandRec$, so the tasks rank models alike. Because
	$\tau_{\mathrm{rec}}$ is dominated by the floor (\S\ref{sec:floor}),
	however, it is a weak stand-in for reading a telling. \textsc{achrony} is
	recognized \ConfAchronySelfPct\% of the time, but the definition the
	models saw describes unrecoverable chronology
	(\S\ref{sec:prompts}), which the corpus excludes. A model that
	follows that definition should rarely answer \textsc{achrony} on these
	event lists, and part of the \textsc{analepsis} sink may have the same
	cause. Two of the seven labels are recovered at better
	than 0.6. They are \textsc{analepsis}, which is also the sink, and
	\textsc{chronicle}, which is the one class identifiable without reading for
	structure at all.

	\section{Task and judge prompts}
\label{sec:prompts}

This appendix presents the task templates and operational category
definitions used by the evaluation harness. Line breaks and typography
are adjusted for presentation. An editorial expansion of the
\textsc{prolepsis} definition is identified below.

The per-work substitutions are the numbered event list, the category
name and, where noted, the graph or order payload. In the templates,
\texttt{<events>} stands for the work's events, numbered $1 \ldots N$
in chronological order, one per line; \Task{3} and \Task{4} instead receive these
events in telling order. \texttt{N} denotes the event count.

Where a template shows \texttt{1 2 ... N}, the harness writes the whole
index list out, so the model is shown $1$ through $N$ in full.
Ordering tasks decode at temperature 0 and low reasoning effort under
an 8{,}000-token cap. Prose tasks decode at temperature 0.7 under a
4{,}000-token cap.

\paragraph{The seven structural definitions.}
The definitions below reproduce the experimental wording, except for
the expanded opening of \textsc{prolepsis}. \Task{1} and \ProseTask{a} receive only the
requested category's definition, whereas \Task{4} receives all seven.
In the completed runs, the \textsc{prolepsis} definition began with a
cross-reference to \textsc{analepsis}, leaving that reference unresolved
in \Task{1} and \ProseTask{a}. We expand it here to make the intended definition
self-contained. This expansion was not supplied in the completed runs;
the experimental code retains the original wording.

\begin{description}\itemsep2pt
    \item[\textsc{chronicle}] ``The chronological ranks run
    1, 2, 3 \ldots{} n in told order. Simultaneous events appearing
    in either order still qualify.''

    \item[\textsc{analepsis}] ``One timeline dominates (it holds more
    than about three quarters of the scenes) and the remaining scenes
    are insertions from earlier than that timeline.''

    \item[\textsc{prolepsis}] \emph{With the opening cross-reference
    expanded:} ``One timeline dominates (it holds more than about
    three quarters of the scenes) and the remaining scenes are
    insertions from later than the dominant timeline. The inserted
    scene must be a real scene the reader can later place, not a
    vague premonition.''

    \item[\textsc{zigzag}] ``The scenes split into two time periods
    that alternate on a regular beat; each period progresses in its
    own consistent direction; and the paired scenes are semantically
    related (echo, explanation, or contrast). One of the two strands
    may run backward.''

    \item[\textsc{syllepsis}] ``There is no dominant timeline, but the
    scenes form blocks grouped by a nameable non-temporal principle:
    place, character, theme, or object. You must be able to state
    the principle.''

    \item[\textsc{retrograde}] ``The ranks run
    n \ldots{} 3, 2, 1 throughout, and there is no present-day scene
    the narrative keeps returning to.''

    \item[\textsc{achrony}] ``You cannot confidently assign
    chronological ranks because the text gives no dates, no ageing,
    no causal signposts, and no other ordering cues. Stop here only
    if the order is genuinely unrecoverable, not merely withheld
    until a late reveal.''
\end{description}

	\paragraph{System prompts.}
	\begin{description}\itemsep2pt
		\item[\Task{1}] ``You are a narrative-ordering engine. Reorder the supplied numbered
		events so their told order realizes the requested nonlinear category. Use
		every event index exactly once. Return only one line of space-separated
		integers. Do not return prose, JSON, labels, punctuation, analysis,
		markdown, or an explanation.''
		\item[\Task{2}] ``You are a storyteller planning how to tell a story. You choose the
		order in which the events will be narrated. Use every event index exactly
		once. Return only one line of space-separated integers. Do not return prose,
		JSON, labels, punctuation, analysis, markdown, or an explanation.''
		\item[\Task{3}] ``You are a narrative analyst. You reconstruct the true
		chronological order of a story's events from a telling that may present
		them out of order. Use every event index exactly once. Return only one line
		of space-separated integers. Do not return prose, JSON, labels, punctuation,
		analysis, markdown, or an explanation.''
		\item[\Task{4}] ``You are a narratologist. You name the way a telling departs from
		chronology. Answer with exactly one label from the supplied list and
		nothing else.''
		\item[prose] ``You are an expert fiction writer. Use only the supplied events.
		Return story prose only: no title, analysis, category name, bullets, or
		markdown.'' (All prose conditions, the graph conditions
		\GraphCondition{0}--\GraphCondition{2} and ablations
		\Ablation{1}--\Ablation{5} included.)
	\end{description}

	\paragraph{\Task{1}, cued transposition.}
	{\footnotesize\begin{verbatim}
			TARGET NONLINEAR CATEGORY
			<category>

			CATEGORY DESCRIPTION
			<its definition above>

			FABULA EVENTS IN CHRONOLOGICAL ORDER
			<events>

			Choose the told order. Output exactly N
			space-separated
			indices: a permutation of
			1 2 ... N
			No index may be missing or repeated. Output
			the indices
			only.
	\end{verbatim}}

	\paragraph{\Task{2}, free transposition.}
	{\footnotesize\begin{verbatim}
			STORY EVENTS IN CHRONOLOGICAL ORDER
			<events>

			You are going to tell this story. Decide the
			order in
			which you would narrate these events to a
			reader. Output
			exactly N space-separated indices: a
			permutation of
			1 2 ... N
			No index may be missing or repeated. Output
			the indices
			only.
	\end{verbatim}}

	\paragraph{\Task{3}, reconstruction.}
	{\footnotesize\begin{verbatim}
			SCENES IN THE ORDER THE STORY TELLS THEM
			<events>

			These scenes may be told out of chronological
			order.
			Work out the order in which the events
			actually happened
			in the story world, earliest first.
			Output exactly N space-separated indices: a
			permutation
			of
			1 2 ... N
			Each index refers to a scene above. No index
			may be
			missing or repeated.
			Output the indices only.
	\end{verbatim}}

	\paragraph{\Task{4}, classification.}
	{\footnotesize\begin{verbatim}
			SCENES IN THE ORDER THE STORY TELLS THEM
			<events>

			CATEGORY DEFINITIONS
			<the seven definitions, "LABEL: text", one per
			line>

			Which single category best describes how this
			telling
			departs from chronology?
			Answer with exactly one of: ACHRONY,
			CHRONICLE,
			RETROGRADE, ZIGZAG, ANALEPSIS, PROLEPSIS,
			SYLLEPSIS
			Output the label only.
	\end{verbatim}}

	\paragraph{\ProseTask{a}, cued prose (= \GraphCondition{0}).}
	{\footnotesize\begin{verbatim}
			TARGET NONLINEAR CATEGORY
			<category>

			CATEGORY DESCRIPTION
			<its definition above>

			STORY EVENTS IN CHRONOLOGICAL ORDER
			<events>

			Write this as a short story of 550-750 words
			whose told
			order realizes the target category. Cover
			every event
			above. Do not number the scenes and do not
			mention the
			category. Output the story prose only.
	\end{verbatim}}

	\GraphCondition{0} is this prompt verbatim. \ProseTask{a} and \ProseTask{b} ran on the
	\ProseNModels{} prose models and \GraphCondition{0} on the \GraphNModels{}
	graph-study models, which share only GLM-5.2, so \ProseTask{a} and \GraphCondition{0} outputs
	differ by generator, not by prompt.

	\paragraph{Condition names.}
	\emph{Free} (\Task{2}, \ProseTask{b}) also appears as \emph{left alone} in the tables, and
	\emph{cued} (\Task{1}, \ProseTask{a}) as \emph{instructed}. \GraphCondition{0}, the \ProseTask{a} prompt,
	appears as \emph{no graph} or \emph{events + shape}, never as events only.

	\paragraph{\ProseTask{b}, free prose.}
	{\footnotesize\begin{verbatim}
			STORY EVENTS IN CHRONOLOGICAL ORDER
			<events>

			Write this as a short story of 550-750 words.
			Cover
			every event above.
			Do not number the scenes. Output the story
			prose only.
	\end{verbatim}}

	\paragraph{\GraphCondition{1}, full graph.}
	The \ProseTask{a} prompt with one block inserted between the events and the brief, and
	the brief extended by ``Use the graph to keep the world state and the
	characters' knowledge consistent, and to decide what to withhold and when to
	pay it off'' and ``do not quote the graph''. The payload is the gold graph
	with both position fields, chronological and told, withheld.
	{\footnotesize\begin{verbatim}
			NARRATIVE GRAPH
			The graph below annotates the same events.
			`core`/`content` describe each event,
			`participants`/`location` ground it,
			`pre`/`post` are
			world-state facts that must stay consistent,
			`learns`
			records what each character and the reader
			comes to
			know, `questions`/`raises`/`resolves` mark
			what the
			telling withholds and where it pays off, and
			the typed
			edges give causes, plants, withheld
			explanations and
			shared material.
			<graph JSON: characters, events, edges>
	\end{verbatim}}
	Each ablation deletes one layer from this JSON \emph{and} the clause that
	describes it from the gloss paragraph, so an ablated prompt never describes a
	layer it does not carry. In the order of Table~\ref{tab:grapharm} and
	Figure~\ref{fig:grapharm}, \Ablation{1} drops the \texttt{CAUSES} and
	\texttt{FORESHADOWS} edges, \Ablation{2} the \texttt{pre}/\texttt{post}
	world state, \Ablation{3} the \texttt{learns} layer, \Ablation{4} the
	\texttt{questions}/\texttt{raises}/\texttt{resolves} fields with the
	\texttt{REVEALS} edges, and \Ablation{5} \texttt{participants} and
	\texttt{location}. \Ablation{1} is the one case that rewrites rather than
	deletes: the typed-edge list survives with \texttt{CAUSES} and
	\texttt{FORESHADOWS} removed, so its clause becomes ``the typed edges give
	withheld explanations and shared material'' instead of disappearing.

	\paragraph{\GraphCondition{2}, told order supplied.}
	The \ProseTask{a} prompt with this block inserted instead, and a brief asking for a story
	``that narrates the events in the told order given above''.
	{\footnotesize\begin{verbatim}
			TOLD ORDER
			Narrate the events in exactly this order:
			<permutation>
	\end{verbatim}}

	\paragraph{The judge-side aligner.}
	The extractor whose gates are reported in \S\ref{sec:alignvalidation}, applied to every
	generated story.
	The instruction ``include each number exactly once'' makes every alignment
	a permutation, so a story that omits an event, returns to one, or
	discloses it in stages still yields one position per event, taken at first
	mention. The two checks in \S\ref{sec:alignvalidation} bound disagreement between
	extractors and recovery of a supplied order. Neither bounds the error this
	projection introduces, and no output has been aligned by hand.
	{\footnotesize\begin{verbatim}
			STORY EVENTS (chronological, numbered)
			<events>

			STORY
			<the generated story>

			The story narrates these events in some order,
			possibly
			not the chronological one.
			List the event numbers in the order the STORY
			NARRATES
			them. Include each number exactly once; if the
			story
			handles an event only in passing, still place
			it where
			it is first narrated.
			Output exactly N space-separated indices and
			nothing
			else.
	\end{verbatim}}

	\paragraph{The judge, preference.}
	Every automated pairwise verdict in \S\ref{sec:graph}, \S\ref{sec:worth} and
	Table~\ref{tab:judgehuman} comes from this prompt. Each pair is judged twice
	with the positions swapped, and disagreement between the two passes is recorded
	as a tie, so a win has to survive being shown in the other order. System
	prompt: ``You compare two stories built from the same events. Reply with a
	single JSON object and nothing else:
	\verb+{"winner":"A"|"B"|"tie"}+.''
	{\footnotesize\begin{verbatim}
			Which is the better story overall -- more
			engaging,
			better told, more satisfying to read?

			STORY A
			<the first story>

			STORY B
			<the second story>
	\end{verbatim}}
	The reader instrument asks the same question in the same words, so the human and
	judge rates in Table~\ref{tab:judgehuman} answer one question rather than two.
	Readers are additionally asked which was harder to follow and whose ending tied
	things together; the judge is not, and those two rows have no automated
	counterpart.

	\paragraph{The judge, Psychological Depth Scale.}
	System prompt: ``You rate stories on the Psychological Depth Scale. Reply with
	a single JSON object and nothing else:
	\verb|{"authenticity":n,"empathy":n,|\allowbreak\verb|"engagement":n,"emotion_provocation":n,|\allowbreak\verb|"narrative_complexity":n}|
	with each $n$ an integer 1--5.'' The user message is \verb|STORY:| followed by
	the story. Readers rated the same five components on the same 1--5 scale with
	the item wording of \citet{harelcanada2024measuring}, which is what makes the
	correlation in \S\ref{sec:humanjudge} a like-for-like comparison. Judge calls
	decode at temperature 0 and low reasoning effort, under a 1{,}200-token cap
	for the depth scale and an 800-token cap for a pairwise verdict.

	\section{The reader instrument}
	\label{sec:instrument}

	\paragraph{One sitting.}
	Consent; two background items (reading habit, writing experience); a tutorial on
	reading a story curve; two blocks of three shape-comprehension items; four blind
	comparisons, of which one is an unannounced repeat of an earlier pair with the
	sides swapped; one close reading. About 45 minutes, of which roughly 24
	are spent reading, at about 6{,}000 words. Reported rates include both
	presentations of the repeated pair.

	\paragraph{The three pairwise questions, as asked.}
	\begin{description}\itemsep2pt
		\item[better] ``Which is the better story overall: more engaging, better told,
		more satisfying to read?'' Answer: Story A / Story B / about the same.
		\item[confusing] ``Which was harder to follow?'' Same three options.
		\item[ending] ``Whose ending tied things together better?'' Same three options.
	\end{description}
	Each pair also takes a confidence rating (1--5) and an optional free-text
	reason. The two story curves are revealed only after all three verdicts are
	locked, so a reader cannot read the shape off a figure and then judge the prose.

	\paragraph{The close reading.}
	One story, read in full, then: the five-item Psychological Depth Scale; a
	chronology reconstruction, in which the reader drags the story's shuffled event
	sentences into the order they believe things happened, scored against the corpus
	chronology; for each reader question the gold graph raises, whether it was never
	raised, raised and answered, or raised and left open; a six-item consistency
	checklist (unanchored scene, prematurely disclosed reveal, unmotivated jump,
	impossible knowledge, world-state error, uncaused event); and an authorship
	guess (human / machine / unsure).

	\paragraph{The shape-comprehension items.}
	Three questions answerable from the drawn permutation alone, adapted from the
	readability study of \citet{kim2018storycurves}: whether the telling opens at or
	near the earliest event (two options), how many times it jumps backwards (four
	options), and whether the last scene shown is also the last thing to happen (two
	options). They are scored separately from every preference statistic, so a
	reader's skill at reading a figure is never confused with their judgment of a
	story.

	\paragraph{Blinding and assignment.}
	The browser receives a pair as a left text and a right text with the condition
	stripped; the mapping is held server-side and rejoined only at export.
	A participant receives about 120\,KB of a 9.8\,MB pool, so the payload itself
	carries no information about the design. Assignment is recorded before any
	answer arrives, so a reader who drops out leaves a visible hole rather than
	silently rebalancing the design, and reopening the page resumes the same block.
	No reader saw the same item twice apart from the deliberate repeat.

	\paragraph{What never reached a reader.}
	A pair was dropped at build time if either text was under 250 words or
	did not end in terminal punctuation. This removed 39 pairs in the three main
	comparisons and 14 in the layer-removal set. Two of those were the reply ``I'm sorry, but I cannot assist
	with that request.'' Showing a refusal or a truncation beside a complete story
	would measure which side was broken and would give the condition away. The
	refusal rate is a finding in its own right and is reported from the model side
	in \S\ref{sec:negatives} rather than from here.

	\section{Reproduction}
	\label{sec:appendix}

	\begin{table*}
		[t]
		\centering\small
		\fitfull{\begin{tabular}{@{}lllll@{}}
\toprule
\textbf{Model}     &     \textbf{Lab}     &     \textbf{OpenRouter identifier}     &     \textbf{Access}     &     \textbf{Tier} \\
\midrule
\micon{claude}Claude Opus 5 & Anthropic & \texttt{anthropic/claude-opus-5} & closed & frontier \\
\micon{claude}Claude Sonnet 5 & Anthropic & \texttt{anthropic/claude-sonnet-5} & closed & mid \\
\micon{claude}Claude Haiku 4.5 & Anthropic & \texttt{anthropic/claude-haiku-4.5} & closed & small \\
\micon{gpt}GPT-5.6 Sol & OpenAI & \texttt{openai/gpt-5.6-sol} & closed & frontier \\
\micon{gpt}GPT-5.6 Terra & OpenAI & \texttt{openai/gpt-5.6-terra} & closed & mid \\
\micon{gpt}GPT-5.6 Luna & OpenAI & \texttt{openai/gpt-5.6-luna} & closed & small \\
\micon{gemini}Gemini 3.1 Pro & Google & \texttt{google/gemini-3.1-pro-preview} & closed & frontier \\
\micon{gemini}Gemini 3.6 Flash & Google & \texttt{google/gemini-3.6-flash} & closed & mid \\
\micon{grok}Grok 4.5 & xAI & \texttt{x-ai/grok-4.5} & closed & frontier \\
\micon{deepseek}DeepSeek V4 Pro & DeepSeek & \texttt{deepseek/deepseek-v4-pro} & open-weights & frontier \\
\micon{deepseek}DeepSeek V4 Flash & DeepSeek & \texttt{deepseek/deepseek-v4-flash} & open-weights & small \\
\micon{zhipu}GLM-5.2 & Z.ai & \texttt{z-ai/glm-5.2} & open-weights & frontier \\
\micon{nvidia}Nemotron 3 Ultra & NVIDIA & \texttt{nvidia/nemotron-3-ultra-550b-a55b} & open-weights & frontier \\
\micon{kimi}Kimi K2 0905 & Moonshot & \texttt{moonshotai/kimi-k2-0905} & open-weights & mid \\
\micon{qwen}Qwen3-235B-A22B & Alibaba & \texttt{qwen/qwen3-235b-a22b-2507} & open-weights & mid \\
\micon{gpt}gpt-oss-120b & OpenAI & \texttt{openai/gpt-oss-120b} & open-weights & mid \\
\micon{meta}Llama 4 Maverick & Meta & \texttt{meta-llama/llama-4-maverick} & open-weights & mid \\
\micon{meta}Llama 3.3 70B & Meta & \texttt{meta-llama/llama-3.3-70b-instruct} & open-weights & small \\
\micon{mistral}Mistral Small 3.2 & Mistral & \texttt{mistralai/mistral-small-3.2-24b-instruct} & open-weights & small \\
\bottomrule
\end{tabular}
}
		\caption{The \NModelsOrdering{} evaluated models with the exact identifier
			each was queried under. Access and tier are the roster's own labels, and the
			grouping is by lab. All generation ran between 6 and 10 August 2026, and the
			reader study between 8 and 17 September 2026.}
		\label{tab:roster}
	\end{table*}

	Table~\ref{tab:roster} lists every evaluated model with the identifier it was
	queried under. Hosted models change under a fixed name, so the query window
	matters. Every reply log was written between 6 and 10 August 2026, with the
	ordering grid on 6--7 August and the prose and scaffolding runs on 6--8
	August.
	The stored replies record latency, token counts and cost per row but not a
	wall-clock timestamp, so those dates come from the run files themselves and
	bound the window rather than dating a single call.

	Every table, figure and inline number in this paper is generated from the
	stored results rather than typed. The model-side tables and numbers, the
	reader-side tables and numbers, and all seven data figures are each written by
	a script that reads the recorded replies and responses directly, and every
	quantity in the prose is a macro expanded at compile time from the same
	source, so a table, a figure and the sentence beside them cannot disagree. The
	exceptions are the handful of fixed quantities that are definitions rather
	than measurements, such as chance baselines, the word-count brief and the
	event-count range, and a set of macros added in revision, recomputed
	offline from the same stored replies and defined in the preamble with the
	source of each noted beside it.

	The reader-side statistics need one further note, because two of them are not
	closed-form. The cluster bootstrap intervals and the bootstrap intervals on
	Krippendorff's $\alpha$ resample with a fixed seed, so a rebuild reproduces the
	published intervals exactly; a different seed moves the third decimal place and
	nothing that any claim in \S\ref{sec:human} turns on. Preference $p$-values are
	exact binomials and do not depend on the seed. The generalized estimating
	equation fits behind the length-adjusted rates use an exchangeable working
	correlation with readers as clusters, and refitting with stories as clusters
	changes no conclusion; both intervals are recorded alongside the rate they
	adjust. A mixed-effects model with crossed random intercepts for readers
	and stories is not yet fit.

	All generation ran through a single OpenRouter harness with the fixed decoding
	of \S\ref{sec:tasks}. Providers were not pinned and the serving provider
	was not logged, so the quantization behind the open-weights models is not
	controlled, and \texttt{gemini-3.1-pro-preview} is a preview endpoint. An output that fails the strict permutation parse is a
	failed row, never repaired, and refusals are recorded separately from parse
	failures rather than pooled with them. The \NModelsOrdering{}-model ordering
	grid cost \$\GridCost{} of tracked API spend and about five hours of wall-clock
	at six-way concurrency, so the benchmark is cheap to re-run. Everything
	downstream of the stored replies, from scoring and tables to figures and this
	PDF, rebuilds offline and deterministically. Verbatim task prompts are in
	\S\ref{sec:prompts}. The project repository is
	\url{https://github.com/Starscream-11813/LLM-nonlinear-storytelling},
	and the project page is \url{https://nonlinear-narrative.github.io/}.
	The release, distributed under the MIT License, includes event lists and
	scene mappings, labels and their rationales, the graph schema and a worked
	example, anonymized twins, selection criteria and the archetype
	distribution, stored model replies with prompts and run metadata, token
	counts and per-model costs, the reasoning-effort ablation, common-subset
	comparisons, reader responses, and the scripts that produce every table,
	figure and macro in this paper.

	\section{Proofs of \S\ref{sec:formal}}
	\label{sec:proofs}

	Throughout, $\sigma$ is a permutation of $\{1,\dots,n\}$, a pair $i<j$ of told
	positions is \emph{concordant} if $\sigma(i)<\sigma(j)$ and \emph{discordant}
	otherwise, and $\tau=(C-D)/\binom{n}{2}$ with $C$ and $D$ the numbers of
	concordant and discordant pairs. A permutation has no ties, so
	$C+D=\binom{n}{2}$ and $\tau=1-2D/\binom{n}{2}$.

	\begin{proof}[Proof of Proposition~\ref{prop:hops}]
		(i) If $h(\sigma)=0$ then $\sigma(i+1)=\sigma(i)+1$ for every $i$, so
		$\sigma(i)=\sigma(1)+i-1$, and $\sigma(n)\le n$ forces $\sigma(1)=1$. The
		converse is immediate.

		(ii) A run is a set of consecutive told positions carrying consecutive
		chronological ranks, hence an increasing subsequence, so
		$\mathrm{LIS}\ge\max_j b_j/n$. The $h+1$ runs partition $n$ positions, so the
		longest holds at least $n/(h+1)$ of them.

		(iii) A pair inside one run is concordant, so
		$D\le\binom{n}{2}-\sum_j\binom{b_j}{2}$, and substituting into
		$\tau=1-2D/\binom{n}{2}$ gives the bound. Equality holds if and only if every
		pair drawn from two different runs is discordant. Each run occupies an
		interval of chronological ranks, so this says that a run told earlier lies
		entirely later in chronology than any run told after it, which is the
		decreasing order in the statement. In that case an increasing subsequence
		cannot leave the run it starts in, so $\mathrm{LIS}=\max_j b_j/n$.

		(iv) A telling with exactly $h$ hops is determined by its runs, which are
		$h+1$ chronological intervals partitioning $\{1,\dots,n\}$, together with the
		order in which they are told. There are $\binom{n-1}{h}$ ways to cut
		$\{1,\dots,n\}$ into $h+1$ intervals and $(h+1)!$ orders to tell them in, and
		$\binom{n-1}{h}(h+1)!\le(n-1)^h(h+1)$. The map from tellings to such pairs is
		injective, and not every pair yields exactly $h$ hops, since two intervals told
		consecutively in chronological order merge into one run. The count is
		therefore an upper bound.
	\end{proof}

	\begin{proof}[Proof of Proposition~\ref{prop:rule}]
		(i) $\tau=-1$ means $D=\binom{n}{2}$, so every pair is discordant, which is
		the reversal. Its runs all have length one, so $h=n-1$, and its longest
		increasing subsequence has one element.

		(ii) The scene-level retrograde tells the scenes in decreasing chronological
		order with each scene forward, so its runs are the scenes and they meet the
		equality condition of Proposition~\ref{prop:hops}(iii). Substituting $b_j=s_j$
		gives the formula and $\mathrm{LIS}=\max_j s_j/n$, and there are $k-1$ hops
		because hops occur only at scene boundaries. The sum $\sum_j\binom{s_j}{2}$ is
		positive as soon as some $s_j\ge2$.

		(iii) $\tau=+1$ means $D=0$, which is the identity. Swapping the events at
		chronological ranks $a<b$ makes $2(b-a)-1$ pairs discordant, at least one for
		any swap, so every other telling has $\tau<1$.
	\end{proof}

	\begin{proof}[Proof of Lemma~\ref{lem:pinned}]
		For a pair $p$ let $X_p=+1$ if the solver places it as the target does and
		$X_p=-1$ otherwise, so that $\tau=\binom{n}{2}^{-1}\sum_p X_p$ and
		$\mathbb{E}[X_p]=2\Pr[p\text{ correct}]-1$. Each pinned pair contributes $+1$.
		Averaging $\mathbb{E}[X_p]$ over the unpinned pairs gives $2\alpha-1$ by the
		definition of $\alpha$. Summing,
		$\mathbb{E}[\tau]=\lambda+(1-\lambda)(2\alpha-1)$.

		For the second claim let $A$ be the set of $w$ events. Exact match requires
		the solver's order restricted to $A$ to be the target's. If the solver's output
		is exchangeable over $A$, each of the $w!$ orders of $A$ has the same
		probability, so the correct one has probability $1/w!$.
	\end{proof}

	\begin{proof}[Proof of Corollary~\ref{cor:pinned}]
		(a) Take the target to be chronology and the pinned pairs to be those the
		event sentences entail, so $\lambda=\lambda_E$. Uninformed on the rest means
		$\alpha=\tfrac12$, and the lemma gives
		$\mathbb{E}[\tau_{\mathrm{rec}}]=\lambda_E$. Every solver has
		$\tau_{\mathrm{rec}}\le1$, so reading a telling adds at most $1-\lambda_E$.

		(b) $\mathrm{RCG}=\tau_{\mathrm{rec}}(T)-\tau_{\mathrm{rec}}(T_{<m})\le
		1-\tau_{\mathrm{rec}}(T_{<m})$. Write $\lambda_E^{<m}$ for the entailed share
		among the pairs the prefix is scored on, which is the quantity the lemma
		applies to here. A prefix solver that keeps those pairs and has
		$\alpha\ge\tfrac12$ on the rest has
		$\mathbb{E}[\tau_{\mathrm{rec}}(T_{<m})]\ge\lambda_E^{<m}$, so
		$\mathbb{E}[\mathrm{RCG}]\le1-\lambda_E^{<m}$; the statement reads
		$\lambda_E^{<m}$ off the prefix rather than the whole work, and the two
		coincide only when entailment is spread evenly through the events. The same
		two lines apply to any difference of two reconstruction scores whose
		subtracted term keeps the entailed pairs.

		(c) Take the target to be the gold telling. Before the graph is supplied the
		generator has $\lambda=0$ and accuracy $\alpha$ on every pair, so
		$\mathbb{E}[\tau]=2\alpha-1$. Afterwards the graph's pairs are pinned,
		$\lambda=\lambda_G$, and the accuracy on the rest is still $\alpha$, so
		$\mathbb{E}[\tau]=\lambda_G+(1-\lambda_G)(2\alpha-1)$. The difference is
		$\lambda_G-\lambda_G(2\alpha-1)=2\lambda_G(1-\alpha)\le2\lambda_G$. Supplying
		the order
		supplies every pair, so $\lambda=1$ and $\mathbb{E}[\tau]=1$.
	\end{proof}	
\end{document}